\documentclass{article}

\usepackage{microtype}
\usepackage{graphicx}
\usepackage{booktabs} 

\usepackage{tikz}
\usepackage{pgfplots}
\usetikzlibrary{calc}
\usepackage{relsize}
\usepackage{mleftright}
\usepackage{makecell}%
\usetikzlibrary{shapes.geometric, arrows}
\usetikzlibrary{shapes.misc, automata, positioning}

\RequirePackage{natbib}

\usepackage[verbose=true,letterpaper]{geometry}
\AtBeginDocument{
  \newgeometry{
    textheight=9in,
    textwidth=6.5in,
    top=1in,
    headheight=12pt,
    headsep=25pt,
    footskip=30pt
  }
}

\usepackage{hyperref}
\usepackage{multirow}

\usepackage{amsmath,amsfonts,bm}
\usepackage{xcolor}
\usepackage[normalem]{ulem}

\usepackage{graphicx}
\usepackage{caption}
\usepackage{subcaption}

\usepackage{algorithm,algorithmic}

\usepackage{amsthm}

\usepackage{footmisc}

\theoremstyle{plain}
\newtheorem{theorem}{Theorem}[section]
\newtheorem{proposition}[theorem]{Proposition}
\newtheorem{lemma}[theorem]{Lemma}
\newtheorem{corollary}[theorem]{Corollary}
\theoremstyle{definition}

\theoremstyle{remark}

\def\eqref#1{(\ref{#1})}

\def\1{\bm{1}}

\def\vtheta{{\bm{\theta}}}

\def\vh{{\bm{h}}}

\def\vp{{\bm{p}}}

\def\vu{{\bm{u}}}

\def\vw{{\bm{w}}}
\def\vx{{\bm{x}}}
\def\vy{{\bm{y}}}
\def\vz{{\bm{z}}}

\def\mU{{\bm{U}}}

\def\mW{{\bm{W}}}
\def\mX{{\bm{X}}}
\def\mY{{\bm{Y}}}

\DeclareMathAlphabet{\mathsfit}{\encodingdefault}{\sfdefault}{m}{sl}
\SetMathAlphabet{\mathsfit}{bold}{\encodingdefault}{\sfdefault}{bx}{n}

\def\gG{{\mathcal{G}}}
\def\gH{{\mathcal{H}}}

\def\sE{{\mathbb{E}}}

\def\sG{{\mathbb{G}}}

\def\sP{{\mathbb{P}}}

\def\sR{{\mathbb{R}}}

\def\sW{{\mathbb{W}}}

\newcommand{\R}{\mathbb{R}}

\newcommand{\yg}[1]{\textcolor{purple}{#1}}

\usepackage{amsmath, amssymb}
\usepackage{mathtools}
\usepackage{amsthm}
\usepackage{dsfont}

\usepackage[utf8]{inputenc} 
\usepackage[T1]{fontenc}    
\usepackage{hyperref}       
\usepackage{url}            
\usepackage{booktabs}       
\usepackage{amsfonts}       
\usepackage{nicefrac}       
\usepackage{microtype}      
\usepackage{xcolor}         

\usepackage[textsize=tiny]{todonotes}

\title{An Unpooling Layer for Graph Generation}

\author{%
  Yinglong Guo \\
  School of Mathematics\\
  University of Minnesota\\
  Minneapolis, MN 55455 \\
  \texttt{guo00413@umn.edu} 
   \and
   Dongmian Zou\\
   Division of Natural and Applied Sciences\\
   Duke Kunshan University\\
   Jiangsu, China\\
   \texttt{dongmian.zou@duke.edu} 
   \and
   Gilad Lerman\\
  School of Mathematics\\
  University of Minnesota\\
  Minneapolis, MN 55455 \\
  \texttt{lerman@umn.edu} 
}

\begin{document}
\maketitle

\tikzstyle{shaded_node} = [circle, shading = axis, left color=#1, right color=#1!40!white,shading angle=135, draw=#1!40!black]

\usetikzlibrary{decorations.markings}

\tikzset{shadedarrow/.style n args={3}{ 
    postaction={
        decorate,
        decoration={
            markings,
            mark=at position \pgfdecoratedpathlength-0.5pt with {\arrow[#2,line width=#1] {>}; },
            mark=between positions 0 and \pgfdecoratedpathlength-12pt step 0.5pt with {
                \pgfmathsetmacro\myval{multiply(divide(
                    \pgfkeysvalueof{/pgf/decoration/mark info/distance from start}, \pgfdecoratedpathlength),100)};
                \pgfsetfillcolor{#2!\myval!#3};
                \pgfpathcircle{\pgfpointorigin}{#1};
                \pgfusepath{fill};}, 
            mark=between positions \pgfdecoratedpathlength-11.5pt and \pgfdecoratedpathlength-6pt step 0.5pt with {
                \pgfmathsetmacro\myval{multiply(divide(
                    \pgfkeysvalueof{/pgf/decoration/mark info/distance from start}, \pgfdecoratedpathlength),100)};
                \pgfsetfillcolor{#2};
                \pgfpathcircle{\pgfpointorigin}{#1};
                \pgfusepath{fill};}
}}}}

\definecolor{c1}{rgb}{0.6,0.6,1.4}
\definecolor{c2}{rgb}{0.8,0.8,2.0}


\begin{abstract}
We propose a novel and trainable graph unpooling layer for effective graph generation. The unpooling layer receives an input graph with features and outputs an enlarged graph with desired structure and features. 
We prove that the output graph of the unpooling layer remains connected and for any connected graph there exists a series of unpooling layers that can produce it from a 3-node graph. 
We apply the unpooling layer within the generator of a generative adversarial network as well as the decoder of a variational autoencoder.
We give extensive experimental evidence demonstrating the competitive performance of our proposed method on synthetic and real data.

\end{abstract}

\section{INTRODUCTION}
\label{sec:inroduction}

Graph data appear in many application areas, such as chemistry \citep{duvenaud2015convolutional}, biology \citep{maere2005bingo} and social recommendation \citep{fan2019graphNetworkRec}. 
Common tasks that arise with graph data include regression and classification of either graph nodes or whole graphs and graph generation, which is useful for molecule generation and drug discovery.
Graph neural networks (GNNs) have successfully generalized standard methods and architectures of neural networks to graph data and have achieved great success in many common tasks.

The task of graph generation is challenging due to its vast search space and the complexity of the graph structure. Furthermore, as we clarify next, it is hard to generalize basic procedures of deep generative networks in 
image generation to graph generation. 
For image data, a generative neural network, which may take the form of the decoder of a variational auto-encoder (VAE) or the generator of a generative adversarial network (GAN), usually first converts the input to a small intermediate image and then applies convolutional-transpose layers \citep{zeiler2010deconvolutional, radford2015unsupervised} or unpooling layers \citep{pu2016variational} 
to upsample and refine the image. 
In the graph domain, it is hard to form a similar convolution-transpose or unpooling layer in order to upsample graphs. Indeed, convolution and message passing do not change the structure of the underlying graph and there is no natural way of building structure for an unpooled graph. We are unaware of any graph generation work that follows the same idea of image generation and produces intermediate graphs and upsamples them to obtain the desired graph.

Inspired by image generation, we propose a novel unpooling layer for graph data that is similar to the unpooling operator for images. By incorporating this layer, one can build a deep graph generative network that utilizes intermediate graph structures.

\subsection{Related Work}\label{sec:related_work}

\textbf{Graph neural networks. }
There is already a vast amount of recent work on GNNs. Many of those works focus on regression or classification tasks of nodes or whole graphs. 
A common building block of a GNN is the message-passing neural network (MPNN) layer, which was first proposed for predicting molecular properties \citep{scarselli2008graph, duvenaud2015convolutional} and was immediately extended to other applications. 
More recently, 
different variants and extensions of the MPNN layer have been proposed within GNNs. For example, graph convolutional networks (GCNs) \citep{kipf2016semi} use a first-order approximation of the spectral convolution to derive a simple propagation rule for node classification, graph attention networks \citep{velivckovic2017graph} use self-attention to assign different weights to different nodes in a neighborhood, and graph isomorphism networks \citep{xu2018powerful} adopt multilayer perceptrons (MLPs) after message passing to enhance expressivity. These networks are simple to implement and also follow a message-passing scheme.

\textbf{Graph pooling and unpooling. }
The common idea of pooling layers in 2D convolutional neural networks has been generalized to graph-based data in order to produce smaller graphs. A graph pooling layer was first proposed by \citet{bruna2014graphpooling1} 
and later extended in various works  \citep{michael2016graphcoarsen,ying2018diffpooling,ma2019eigenpooling,lee2019graphattention}.
Pooling layers are widely used in classification and regression on graphs since they downsample the graph while summarizing the aggregated presence of encoded features. 

Unlike the convolution-transpose or unpooling operation for images, there is no obvious way to define a trainable unpooling procedure for graphs. Some works \citep{jin2018junctionTree, jin2019learning, bongini2021molecular} sequentially expand graphs by adding  one node at a time. While their operations can be considered as unpooling, they cannot be regarded as the inverse operation of common pooling, since  graph pooling is not generally done by removing one node at a time. 
Among all works related to graph pooling, \citet{gao2019graph} is the only one that seeks to define graph unpooling. They proposed Graph-UNet to combine pooling and unpooling processes in an encoder-decoder architecture. A Graph-UNet first pools an input graph into a smaller graph, encodes its global features and then applies the exact inverse process to perform the unpooling procedure. 
Since this operation deterministically depends on the pooling layers in the encoder, it is not suitable to be used for graph generation. 

{\bf Graph generation.}
\citet{dgg_survey2021} nicely
survey graph generation models and categorized them as follows: auto-regressive (AR) \citep{you2018graphrnn, bongini2021molecular, ahn2021spanningTree}, VAE \citep{simonovsky2018graphvae, jin2018junctionTree, samanta2020nevae, guo2021deep}, GAN \citep{de2018molgan}, reinforcement learning (RL) \citep{you2018gcpn}, and normalized flow  \citep{madhawa2019graphnvp, shi2020graphaf, zang2020moflow, luo2021graphdf}. 
Another recent category is diffusion models \citep{Jo2022ScorebasedGM}. 
There are two types of graph generation strategies: one-shot and sequential, where the former generates the output graph at once and the latter generates it in a node-by-node and edge-by-edge fashion. 
Most existing methods that leverage the one-shot strategy, such as  \citet{de2018molgan, simonovsky2018graphvae, zang2020moflow}, produce a vectorized adjacency matrix, 
which does not utilize any graph structure during generation. 
On the other hand, methods  that sequentially generate graphs typically use the graph structure. For example, \citet{you2018graphrnn, shi2020graphaf, luo2021graphdf} predict the next node or edge based on the features extracted from the existing graph.

Molecule generation is the most common application in this area as molecules can be naturally represented as graphs with features. However, molecules need to be chemically valid and this validity issue does not arise in general graph generation.
Many methods \citep{de2018molgan, simonovsky2018graphvae, samanta2020nevae, zang2020moflow} generate molecules at one-shot by producing an adjacency matrix that captures the molecular graph structure.  \citet{mahmood2021masked} propose the masked graph model (MGM) that generates graphs at one-shot by sampling masked sub-graphs of the respective complete graph. This method suggests a new graph generation category. 
On the other hand, 
it is also possible to sequentially generate molecules.
For example, \citet{jin2018junctionTree, jin2019learning} 
first generate a junction-tree as the scaffolding and then complete the graph. The junction-tree is generated recursively from the root, one node at a time. 
Some other recent models \citep{shi2020graphaf, ahn2021spanningTree, luo2021graphdf} apply node-by-node and edge-by-edge sequential generation.
\citet{bongini2021molecular} break the graph generation into three subproblems: node classification (which leads to node expansion), edge classification, and edge addition, where three separate GNNs are trained for each subproblem.

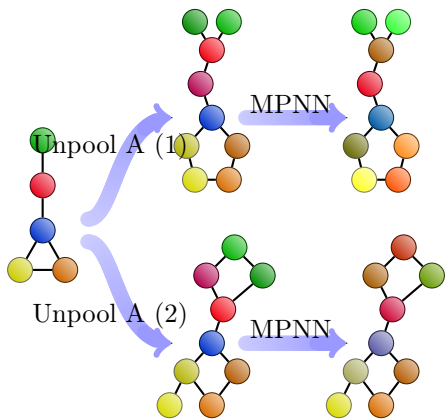
\begin{figure}[h]
  \centering
\begin{tikzpicture}[scale=0.75]
\begin{scope}
\node[shaded_node=blue!80!green] (x1) at (0, 0) {};
\node[shaded_node=red!90!blue] (x2) at (0, 0.8) {};
\node[shaded_node=green!65!black] (x3) at (0, 1.6) {};
\node[shaded_node=yellow!80!black] (x4) at (-0.4, -0.7) {};
\node[shaded_node=orange!80!black] (x5) at (0.4, -0.7) {};
\draw [thick, draw=black](x1) -- (x2); 
\draw [thick, draw=black](x2) -- (x3); 
\draw [thick, draw=black](x1) -- (x4); 
\draw [thick, draw=black](x1) -- (x5); 
\draw [thick, draw=black](x4) -- (x5); 
\end{scope}

\begin{scope}[xshift=3cm,yshift=2cm]
\node[shaded_node=blue!80!green] (x1) at (0, 0) {};
\node[shaded_node=red!70!blue] (x21) at (-0.2, 0.6) {};
\node[shaded_node=red!95!blue] (x22) at (0, 1.2) {};
\node[shaded_node=green!55!black] (x31) at (-0.3, 1.7) {};
\node[shaded_node=green!70!black] (x32) at (0.3, 1.7) {};
\node[shaded_node=yellow!70!black] (x41) at (-0.45, -0.5) {};
\node[shaded_node=yellow!85!black] (x42) at (-0.3, -1.1) {};
\node[shaded_node=orange!70!black] (x51) at (0.45, -0.5) {};
\node[shaded_node=orange!85!black] (x52) at (0.3, -1.1) {};
\draw [thick, draw=black](x1) -- (x21); 
\draw [thick, draw=black](x21) -- (x22); 
\draw [thick, draw=black](x22) -- (x31); 
\draw [thick, draw=black](x22) -- (x32); 
\draw [thick, draw=black](x1) -- (x41); 
\draw [thick, draw=black](x1) -- (x51); 
\draw [thick, draw=black](x41) -- (x42); 
\draw [thick, draw=black](x51) -- (x52); 
\draw [thick, draw=black](x52) -- (x42); 
\end{scope}

\begin{scope}[xshift=6cm,yshift=2cm]
\node[shaded_node=blue!65!green] (x1) at (0, 0) {};
\node[shaded_node=red!90!blue] (x21) at (-0.2, 0.6) {};
\node[shaded_node=red!65!green] (x22) at (0, 1.2) {};
\node[shaded_node=green!80!black] (x31) at (-0.3, 1.7) {};
\node[shaded_node=green!80!white] (x32) at (0.3, 1.7) {};
\node[shaded_node=yellow!40!black] (x41) at (-0.45, -0.5) {};
\node[shaded_node=yellow!85!white] (x42) at (-0.3, -1.1) {};
\node[shaded_node=orange!90!white] (x51) at (0.45, -0.5) {};
\node[shaded_node=orange!65!red] (x52) at (0.3, -1.1) {};
\draw [thick, draw=black](x1) -- (x21); 
\draw [thick, draw=black](x21) -- (x22); 
\draw [thick, draw=black](x22) -- (x31); 
\draw [thick, draw=black](x22) -- (x32); 
\draw [thick, draw=black](x1) -- (x41); 
\draw [thick, draw=black](x1) -- (x51); 
\draw [thick, draw=black](x41) -- (x42); 
\draw [thick, draw=black](x51) -- (x52); 
\draw [thick, draw=black](x52) -- (x42); 
\end{scope}

\begin{scope}[xshift=3cm,yshift=-2cm]
\node[shaded_node=blue!80!green] (x1) at (0, 0) {};
\node[shaded_node=red!95!blue] (x21) at (0.2, 0.6) {};
\node[shaded_node=red!70!blue] (x22) at (-0.1, 1.2) {};
\node[shaded_node=green!55!black] (x31) at (0.9, 1.2) {};
\node[shaded_node=green!70!black] (x32) at (0.4, 1.7) {};
\node[shaded_node=yellow!70!black] (x41) at (-0.45, -0.5) {};
\node[shaded_node=yellow!85!black] (x42) at (-0.75, -1.1) {};
\node[shaded_node=orange!70!black] (x51) at (0.45, -0.5) {};
\node[shaded_node=orange!85!black] (x52) at (0, -1.1) {};
\draw [thick, draw=black](x1) -- (x21); 
\draw [thick, draw=black](x21) -- (x22); 
\draw [thick, draw=black](x21) -- (x31); 
\draw [thick, draw=black](x22) -- (x32); 
\draw [thick, draw=black](x31) -- (x32); 
\draw [thick, draw=black](x1) -- (x41); 
\draw [thick, draw=black](x1) -- (x51); 
\draw [thick, draw=black](x41) -- (x42); 
\draw [thick, draw=black](x51) -- (x52); 
\draw [thick, draw=black](x41) -- (x52); 
\end{scope}

\begin{scope}[xshift=6cm,yshift=-2cm]
\node[shaded_node=blue!65!yellow] (x1) at (0, 0) {};
\node[shaded_node=red!80!blue] (x21) at (0.2, 0.6) {};
\node[shaded_node=red!65!green] (x22) at (-0.1, 1.2) {};
\node[shaded_node=green!60!red] (x31) at (0.9, 1.2) {};
\node[shaded_node=green!60!black!25!red] (x32) at (0.4, 1.7) {};
\node[shaded_node=yellow!65!blue] (x41) at (-0.45, -0.5) {};
\node[shaded_node=yellow!85!black] (x42) at (-0.75, -1.1) {};
\node[shaded_node=orange!70!black] (x51) at (0.45, -0.5) {};
\node[shaded_node=orange!85!black] (x52) at (0, -1.1) {};
\draw [thick, draw=black](x1) -- (x21); 
\draw [thick, draw=black](x21) -- (x22); 
\draw [thick, draw=black](x21) -- (x31); 
\draw [thick, draw=black](x22) -- (x32); 
\draw [thick, draw=black](x31) -- (x32); 
\draw [thick, draw=black](x1) -- (x41); 
\draw [thick, draw=black](x1) -- (x51); 
\draw [thick, draw=black](x41) -- (x42); 
\draw [thick, draw=black](x51) -- (x52); 
\draw [thick, draw=black](x41) -- (x52); 
\end{scope}

\draw [shadedarrow={0.1cm}{c1}{c2}{0.4}] (0.8,0.2) .. controls (1.3, 0.2) and (1.3, 2.0) ..  (2.3,2.0);
\draw [shadedarrow={0.1cm}{c1}{c2}{0.4}] (0.8,-0.2) .. controls (1.3, -0.2) and (1.3, -2.0) ..  (2.3,-2.0);
\node  at (1.2, 1.5) {Unpool A (1)};
\node  at (1.2, -1.5) {Unpool A (2)};
\draw [shadedarrow={0.1cm}{c1}{c2}{0.4}]   (3.6,2.0) -- (5.4,2.0);
\draw [shadedarrow={0.1cm}{c1}{c2}{0.4}]   (3.6,-2.0) -- (5.4,-2.0);
\node  at (4.4, 2.25) {MPNN};
\node  at (4.4, -1.75) {MPNN};


\end{tikzpicture}

\caption{Demonstration of possible outputs of the proposed unpooling layer. Left: input graph, middle: two potential  outputs of the unpooling layer, right: further application of a message-passing neural network (MPNN) layer. The colors of nodes represent their features.}\label{fig:overall_graph}
\end{figure}

\subsection{This Work}
\label{sec:contribution}

We propose a novel unpooling layer that 
effectively leverages the features and structure of a given graph to form an enlarged graph with learned features and structure. One may apply additional layers, such as MPNN layers, to further refine the features of nodes and edges in the graph. Figure~\ref{fig:overall_graph} demonstrates two possible outputs of the unpooling layer with a followup MPNN layer for a given input graph.
By incorporating unpooling layers in deep graph generative networks, we can generate a graph at one shot. 
In the experiments, we incorporate unpooling layers within both the generator of GAN and the decoder of VAE. We demonstrate in a synthetic setting how the  unpooling layers reveal useful intermediate graph structures. We believe that the incorporation of such structures results in the competitive performance which is evident in all numerical experiments.

Our proposed unpooling layer (UL) leads to one-shot generation that utilizes the graph structure during generation. Among all existing methods, only the masked graph model (MGM) of \citet{mahmood2021masked}  applies one-shot generation that utilizes the graph structure during generation. Nevertheless, the implementation of our method is very different from \citet{mahmood2021masked}. In particular, the proposed unpooling layer enlarges the graph at each intermediate step, 
whereas \citet{mahmood2021masked} sample the graph in masked subgraphs. As mentioned earlier, the graph generation category of MGM is rather different from the common categories. 
Furthermore, MGM was implemented and applied for molecule generation and not general graph generation. 
Therefore, in terms of methodology our proposed strategy is distinguished from the many previous graph generation methods. 

\input{tikz_plot_algo_unpool}

We emphasize the following contributions of our work:
\begin{itemize}
    \item We propose a novel unpooling layer that  produces an enlarged graph with learnable structure.
    It can be inserted into GANs and VAEs.  The resulting generation framework is distinguished in its ability to both generate graphs at one shot and utilize the graph structures for generation. 
    \item We show that the unpooling layer is valid and expressive. That is, the unpooled graph remains connected and any connected graph can be generated by a series of unpooling layers from a 3-nodes graph.
    \item We test the unpooling layer within both GANs and VAEs on a random graph dataset, a protein dataset and two molecule datasets and demonstrate competitive performance.
    \end{itemize}

\subsection{Structure of the Rest of the Paper}
Section \ref{sec:methodology} details our proposed methodology; \S\ref{sec:TheoreticalResult} provides theoretical guarantees of connectivity and expressivity;  \S\ref{sec:experiment} reports numerical results on synthetic and real data of protein and molecule generation; and \S\ref{sec:conclusion} concludes this work and discusses its limitations.

\section{METHODOLOGY}\label{sec:methodology}

Section \ref{subsec:unpooling} clarifies the  construction of the unpooling layer and  \S\ref{subsec:update_parameters} explains how to use such layers for graph generation and how we update the parameters for the unpooling layers. 

\subsection{Unpooling Layer}\label{subsec:unpooling}

Given an input graph (with features), the unpooling layer first unpools some of its nodes by replacing them with two ``children'' nodes and then learns a graph structure for the new set of nodes (in the output graph) and further learns new features.   

\textbf{Notation.} 
We use the following notation for the input graph and its features. Its node set is $V = [N] := \{1, \cdots, N\}$; its edge set is $E$, where its members are of the form $\{i,j\}$ for some $i$, $j \in V$; its number of edges is $M = |E|$; its node feature matrix is $\mX \in \mathbb{R}^{N\times d}$, where its $i$-th row, $\vx_i$, is the $d$-dimensional feature of node $i$; and its edge feature matrix is $\mW \in \mathbb{R}^{M\times e}$, where for edge $\{i, j\} \in E$, its corresponding row of $\mW$ is the $e$-dimensional feature of that edge, which we denote by $\vw_{i, j}$.
 Similarly, we use the following notation for the output graph and its features: $V^o$ is its node set, $E^o$ is its edge set and $\mY$ and $\mU$ are the feature matrices for the output nodes and edges, respectively. 
We remark that the size of $V^o$ lies in $[|V|+1, 2|V|]$ and depends on the hyperparameters that determine which nodes should be unpooled.

The input and output graphs of the unpooling layers with their features are respectively denoted by 
$$\gG = (V, E, \mX, \mW) \ \text{ and } \ \gG^o = (V^o, E^o, \mY, \mU).$$ 
We will refer to them as featured graphs and to $(V,E)$ and $(V^o,E^o)$ as graphs.

\textbf{An overview of the unpooling layer.}
The unpooling layer determines the output graph in a stochastic manner. Ideally, it should produce probabilities of every possible output graph based on the input graph by using trainable parameters, $\vtheta_P$. That is, it would output a probability mass function, $\vp$, on the sample space $\sG^o = \{(V^o, E^o): \text{unpooled from } \gG\}$, where $\vp(\gG; \vtheta_P) \in [0, 1]^{|\sG^o|}$. 
In this ideal case, the unpooling layer then randomly draws an output graph, $(V^o, E^o)$, according to this probability mass function. The drawn probability of this sample point, which we denote by
\begin{equation}
\sP(V^o, E^o | \gG; \vtheta_P)\label{eqn:unpool_probs},
\end{equation}
can then be used to update the parameters $\vtheta_P$ during training. Therefore, the unpooling layer can ideally refine the probability distribution $\vp(\gG;\vtheta_P)$ in order to obtain graphs that minimize the training loss function.

Since the sample space $\sG^o$ is huge, we cannot explicitly produce the probabilities of all possible output graphs. 
In practice, we use several multi-layered perceptrons (MLPs) to produce probabilities to determine if nodes should be unpooled and if edges between unpooling children nodes should be included in the output graph. 
The product of all these probabilities for nodes and edges gives the probability of the entire output graph in \eqref{eqn:unpool_probs} (assuming these events are independent), which is used during training.

After forming the output graph, the unpooling layer also produces the node and edge features using two MLPs that exploit the input featured graph and possibly the output graph, that is,
$$
\mY = \operatorname{MLP-}Y(\gG; \theta_Y), \mU = \operatorname{MLP-}W(V^o, E^o, \gG; \theta_W).
$$

\input{tikz_plot_model_flow_noMolecular} 

\textbf{Detailed mechanism of the unpooling layer.} The unpooling layer contains seven MLPs that serve different purposes, which we introduce below and in the supplementary materials \S\ref{sec:UL_method}. 
For the formation of the unpooling layer We use the following  three node sets that partition the input node set: (1) the set of nodes $I_s'$ that are unchanged in the unpooling layer; (2) the set of nodes  $I_u'$ that are determined to be unpooled in the unpooling layer; and (3) the set of nodes $I_r'$ that requires a probabilistic decision whether to unpool or not.

With the specified node sets, we describe the procedure of generating the output graph in the unpooling layer according to 3 steps, which we demonstrate in Figure~\ref{fig:unpool}.

\textbf{Step 1. Generating the output nodes and node features ((a) in Figure~\ref{fig:unpool})} We first probabilistically determine which nodes in $I_r'$ will be unpooled. For each node $i\in I_r'$, we determine the probability of unpooling it, $p_r(\vx_i)$,  by an MLP as
follows:
$p_r(\vx_i) = \textrm{MLP-}R(\vx_i)
$. Then we draw uniform random variables $U_i\sim U[0,1]$ and form the following sets $I_u$ and $I_s$ of  unpooled nodes and unchanged (or stable) nodes: $I_u := I_u'\cup \{i\in I_r' \ : \ U_i < p_r(\vx_i)\}$ and $I_s := I_s'\cup\{i\in I_r' \ : \ U_i \geq p_r(\vx_i) \}$.

We remark that if we want to generate the output graph with a fixed number of nodes, we could choose $I_r' = \emptyset$ for the unpooling layer.

The set of nodes of the output graph $V^o$ is the union of the nodes in $I_s$ (with different indices) and a set of nodes of size $2 |I_u|$ representing the unpooled nodes from $I_u$. 
For each node $i\in I_s$, we denote by $f(i)$ the index of this node in the output graph. For each node $i\in I_u$, we denote by $f_1(i)$ and $f_2(i)$ the indices of the two unpooled nodes in the output graph. The output node features are generated by an MLP and two projection operators $P_{S_1}$ and $P_{S_2}$, which are defined in detail in \S\ref{subsec:unpoolinglayers}
\begin{align*}
    \vy_{f(i)} & = \textrm{MLP-}y(P_{S_1}\vx_{i}), \ \text{ for } i\in I_s.\\
    \vy_{f_j(i)} & = \textrm{MLP-}y(P_{S_j}\vx_{i}), \ \text{ for } i\in I_u, \ j=1, 2.
\end{align*}

\textbf{Step 2. Building output edges.} We  sequentially generate the set $E^o$ of edges in the output graph following  the next two substeps. We initiate this set by $E^o := \emptyset$.

\textbf{Step 2.1. Building intra-links ((b) in Figure~\ref{fig:unpool}).} For each node $i\in I_u$, we determine whether to generate an edge that connects the children nodes $f_1(i)$ and $f_2(i)$ based on a probability, which we denote by $p_c(\vx_i)$. This probability is produced using an MLP as follows: $p_c(\vx_i) \equiv \textrm{MLP-IA}(\vx_i)$. Then we draw uniform random variable $U_i\sim U[0, 1]$ and if $U_i < p_c(\vx_i)$, we add this edge to the output graph, that is, $E^o = E^o\cup \big\{\{f_1(i), f_2(i)\}\big\}$.

\textbf{Step 2.2. Building inter-links ((c) in Figure~\ref{fig:unpool}).} For each edge $\{i, j\}\in E$ in the input graph, we determine the edges for the corresponding nodes in the output graph according to the following three different cases: (1) $i, j\in I_s$: we include the edge $\{f(i), f(j)\}$ in the output graph; (2) $i\in I_s$ and $j\in I_u$: we probabilistically determine what are the edges between $f(i)$ and $\{f_1(j), f_2(j)\}$; and (3) $i, j\in I_u$: we probabilistically determine what are the edges between $\{f_1(i), f_2(i)\}$ and $\{f_1(j), f_2(j)\}$.

For each edge $\{i, j\}\in E$, we introduce node sets $N_{\{i, j\}, i}$ and $N_{\{i, j\}, j}$ in order to uniformly handle the above cases. For $\{i, j\}\in E$ and $i\in V$, we form $N_{\{i, j\}, i}$ as follows: If $i\in I_s$, then $N_{\{i, j\}, i}=\{f(i)\}$ and if $i\in I_u$, then $N_{\{i, j\}, i}$ is a nonempty subset of the children nodes of $i$, which we probabilistically determine as follows. 
We use an MLP to calculate two probabilities: $(p_1, p_2) = \textrm{MLP-IE}(\vy_{f_1(i)}, \vy_{f_2(i)}, \vw_{i, j}, \vx_{j})$. We then draw a uniform random variable $U\sim U[0, 1]$ and determine $N_{\{i, j\}, i}$ as follows:
$$
    N_{\{i, j\}, i} =\left\{ \begin{array}{ll}
\{f_1(i)\}, & \ \text{ if } \ U < p_1;\\
    \{f_1(i), f_2(i)\}, & \ \text{ if } \ U \geq p_1+p_2;\\
    \{f_2(i)\}, & \ \text{ otherwise}.
    \end{array}\right.
$$
For $\{i, j\} \in E$ and $j \in V$ we similarly define $N_{\{i, j\}, j}$, while swapping $i$ and $j$. 
The edges in the output graph are updated as follows: \begin{equation}E^o = E^o\cup \{\{k, l\}: \ k\in N_{\{i, j\}, i}, \ l\in N_{\{i, j\}, j}\}.\label{eqn:generate_inter_Eo}\end{equation}

We need to take extra care to ensure connectivity and expressivity of the output graph. We first form two additional MLPs, $\textrm{MLP-C}$ and $\textrm{MLP-IE-A}$ to calculate probabilities. We then probabilistically form some additional edges. 
These details are a bit technical and can be fully understood after getting familiar with our theory for connectivity and expressivity (see \S\ref{sec:TheoreticalResult} and the proofs in the supplementary materials \S\ref{sec:theorem}). Therefore we leave these details to \S\ref{sec:UL_method} of the supplementary materials (see Step 2b and Step 2d in \S\ref{sec:UL_method}).

\textbf{Step 3. Constructing the edge features ((d) in Figure~\ref{fig:unpool}).} For each edge $\{k, l\}\in E^o$, we construct the edge feature $\vu_{k, l}$ by an MLP as follows:
$$\vu_{k, l} = \textrm{MLP-}u\big(\textrm{LeakyReLU}(\vy_k + \vy_l)\big).$$

\textbf{Summary.} We described a probabilistic construction of $\gG^o = (V^o, E^o, \mY, \mU)$. It contains seven MLPs to produce various probabilities and features for the nodes and edges of $\gG^o$. The parameters in those seven MLPs form the training parameters of the unpooling layer.
The overall probability $\sP(V^o, E^o|\gG; \vtheta_P)$ is the product of all the probabilities in the first two steps and is used to update the training parameters in the unpooling layer, while using the REINFORCE algorithm introduced below.

\subsection{Graph Generation and Training}\label{subsec:update_parameters}

We use the unpooling layer within a generative GNN, which can be either a generator of a GAN or a decoder of a VAE.
We describe here its basic mechanism and demonstrate it in Figure~\ref{fig:gen_model}. Complete details of implementation are in both \S\ref{sec:experiment} and the supplementary materials. 
The generative GNN first maps a given latent vector into  a featured 3-nodes graph, whose structure is probabilistically determined by an MLP and its edge features are determined by another MLP (details are in the supplementary materials). Next, it applies a GCN (such as MPNN) to update the node features for this initial featured graph. It then sequentially applies unpooling layers, where each of them is followed by a GCN (such as MPNN), which further updates the node features. The final output is the generated graph. 

The parameters used for generating features in the unpooling layer can be updated during training following the common framework of GAN or VAE. The major challenge in using the unpooling layers for graph generation is that the graph generation process is not differentiable. To overcome this, we follow REINFORCE with baseline \citep{weaver2001optimal, sutton2018reinforcement} to update all the graph generation parameters in the unpooling layer. 

In order to explain this procedure with more details, we need the following notation. We denote by $G$ a generative GNN (a generator of a graph GAN or the decoder of a graph VAE) with several unpooling layers that takes a latent vector and produces a generated graph. We denote by $m$ the number of unpooling layers of $G$, by $U_1$, $U_2$, $\cdots$, $U_m$ the unpooling layers
and by $(V^o_1, E^o_1)$, $\cdots$, $(V^o_m, E^o_m)$ the generated intermediate graphs.
Let $\vtheta$ denote all the parameters of $G$,
which include the parameters of the $\textrm{MLP}$s in the unpooling layers. In view of \eqref{eqn:unpool_probs}, the total log probability of the unpooling layer $U_k$  is  $\log \sP(V^o_k, E^o_k)$. We define the total log probability of the generator $G$ as $\log \sP:=\sum_k \log \sP(V^o_k, E^o_k)$ and note that $\log\sP$ depends on $\vtheta$. We  denote the learning rate by $\alpha$ and the reward for the generated graph by $r$. Note that this reward depends on the specific generation task, e.g., it can be the likelihood predicted by the discriminator or the chemical property which one aims to optimize. 

We update $\vtheta$ as follows
\begin{equation}
    \vtheta_{k+1} = \vtheta_k + \alpha \mleft(\nabla_\vtheta \log\sP|_{\vtheta_k}\mright) (r - \sE r),\label{eqn:update}
\end{equation}
where we approximate $\sE r$ by the sample mean.
In our experiments we incorporate the unpooling layers within a GAN and a VAE. For a GAN,  we set the reward $r$ to be $D(G(z; \vtheta))$, where $D$ is the GAN's discriminator, in order to compete with the discriminator. For a VAE, we choose the reward $r$ to be the negative of the reconstruction error in order to minimize the reconstruction error.

\section{THEORETICAL GUARANTEES}\label{sec:TheoreticalResult}
We establish the connectivity and expressivity of the unpooling layer. All proofs are in the supplementary materials.

\subsection{Guarantee of Connectivity of the Output Graph}

The following proposition implies that if the input graph is connected, then the unpooling layer will produce a connected graph. This is an important property in molecule generation since otherwise the output molecule will be invalid. 
Adjacency matrix-based generators cannot ensure connectivity. 
\begin{proposition}\label{thm:connecivity}
Given an unpooling layer  and any connected input graph $\gG$, the output graph, $\gG^o$, of this unpooling layer is connected.
\end{proposition}

\subsection{Guarantee of Expressivity for the Unpooling Layer}

It is important to know whether a series of unpooling layers can produce any connected graph. For instance, in molecule generation, a good generative model should contain all valid molecules in the set of possible output. Some previous work (e.g., \citet{jin2018junctionTree}) cannot produce some valid molecular structures and is thus not fully expressive. Fortunately, we are able to produce any connected graph by applying certain unpooling layers to an input graph with three nodes
(our implementation of the graph generative model starts with a 3-nodes graph).
We first formulate a theorem on the expressivity of a single unpooling layer and then formulate the desired corollary when starting with a 3-nodes graph and using a series of unpooling layers. 

\begin{theorem}\label{thm:expressivity}
Given a connected graph $\gG^o$ with $N$ nodes and an integer $K \in [\lceil N/2\rceil, N-1]$, 
there exist an unpooling layer
and an input graph $\gG$ with $K$ nodes so that $\gG^o$  is the corresponding output.
\end{theorem}

\begin{corollary}\label{thm:main_result}
Given a connected graph $\gG^o$ with $N$ nodes, there exist a 3-nodes graph $\gG$ and $\mleft\lceil \log_2( N/3 ) \mright\rceil$ unpooling layers, so that $\gG^o$ is the output of this series of unpooling layers acting on $\gG$.
\end{corollary}

We remark that the proof of Theorem \ref{thm:expressivity} naturally provides a ``pooling'' procedure on the graph structure which can be regarded as the inverse operation of our unpooling layer. This validates the name ``unpooling''.

\section{EXPERIMENTS}\label{sec:experiment}
We demonstrate the effectiveness of the unpooling layer for molecule generation. 
We describe the two datasets in \S\ref{subsec:datasets}, the 
evaluation metrics in \S\ref{subsec:evaluation_metrics} and the details of the implemented methods in \S\ref{subsec:model_detail}. We then report the results in \S\ref{subsec:result}, while comparing with benchmark methods.
All implemented codes are provided in the supplementary materials.

\subsection{Datasets}\label{subsec:datasets}

\textbf{Waxman random graph dataset.} 
The dataset contains randomly generated Waxman graphs \citep{waxman1988}. More precisely, we first created 20,000 graphs with 12 nodes and node features uniformly drawn from $[0,1]^2$. For each graph, we connected nodes $i$ and $j \in [12]$ with probability $qe^{-s d_{ij}}$, where $q=0.65$, $s=0.3$ and $d_{ij}$ is the Euclidean distance between their features.  We do not assign edge features. Next, for each graph we compute the largest connected subgraph as long as it has at least 5 nodes. The final set contains these subgraphs with at least 5 nodes (the ones with less nodes are discarded). Thus the number of nodes ranges from 5 to 12 and the node features are the $x$ and $y$ coordinates. On average, each graph contains 9.2 nodes and 10.3 edges. There are 18,910 graphs in this dataset.

\textbf{Protein dataset.} We use the protein dataset introduced in \citet{guo2021deep}, which is a benchmark for graph generation \citep{du2021graphgt}. All the graphs contain 8 nodes and their node features are their 3D coordinate vectors.  There are no edge features. On average, each graph contains 8 nodes and 19.3 edges. There are 76,000 graphs in the datasets, where we use 38,000 for training and 38,000 for testing.

\textbf{Molecule datasets.} We use two common datasets for molecule generation: QM9 \citep{ramakrishnan2014quantum}  and ZINC \citep{sterling2015zinc}. 
QM9 contains 130k molecules and each molecule consists of up to 9 heavy atoms among carbon (C), oxygen (O), nitrogen (N) and fluorine (F). 
In \S\ref{subsec:model_detail} we explain how we choose the hyperparameters of the unpooling layers so that the generator will generate molecules 
with the number of atoms ranging between 6 and 9. On average, each graph contains 8.8 nodes and 9.4 edges.

ZINC contains about 250k molecules, where each molecule consists of 9 to 38 heavy atoms among carbon (C), oxygen (O), nitrogen (N), sulfur (S), fluorine (F), chlorine (Cl), bromine (Br), iodine (I) and phosphorus (P). For simplicity, we only take molecules with 11 - 36 heavy atoms (99.8\% of ZINC). On average, each graph contains 23.2 nodes and 24.9 edges.

For both QM9 and ZINC, we include the following node features: atom type, chiral specification of an atom (unspecified, clockwise or counter-clockwise) 
and the formal charge of an atom (0, +1 or -1).  
We use bond type (single, double or triple) as edge features. The node and edge features are represented as one-hot vectors.

\subsection{Evaluation Metrics}
\label{subsec:evaluation_metrics}

In the numerical experiments of the Waxman random graph and protein datasets, we evaluate the similarity of the generated data and source data by comparing the distributions of some graph properties in the generated graphs and source data. We use the following graph properties for comparison: average node connectivity, average clustering coefficient, edge density and node features. For both datasets, we report the Kullback–Leibler divergence and Wasserstein distance between the distributions of the source and generated data. For the Waxman random graph dataset we further demonstrate the distributions of the four properties for the source and generated data, while considering several methods for graph generation.

In the numerical experiments of molecule generation, we compare the different generators by generating 10,000 molecules
and applying the following metrics: 
validity (the ratio between the number of generated valid molecules and all generated graphs); uniqueness (the ratio between the number of unique valid molecules and generated valid molecules); and novelty (the ratio between 
the number of unique valid molecules which are different from all molecules in the dataset and the total number of generated unique valid molecules). We also report the geometric mean of the above three metrics (G-mean).

\subsection{Implementation Details }\label{subsec:model_detail}

We implemented a GAN with the unpooling layers (UL GAN), a GAN with an adjacency matrix-based generator (Adj GAN) and a VAE with the unpooling layers (UL VAE). For simplicity, we just describe here the implementation for molecule generation using the QM9 dataset. Indeed, the implementations of Adj GAN, UL GAN and UL VAE for the other applications are similar. Additional details are in the supplementary materials.

{\bf Discriminator for Adj GAN and UL GAN.} 
It takes an input graph and uses two MPNN layers with $128$ units to generate a graph $\gG=(V, E, \mX, \mW)$. It then aggregates the node features (the rows $\{\vx_j\}_{j \in V}$ of $\mX$) to produce the following single feature vector for the graph:
$$
\vh(\gG) = \sum_{j\in V}\sigma (\operatorname{lin}_1(\vx_j))  \odot
\tanh (\operatorname{lin}_2 (\vx_j));
$$
where $\sigma$ is logistic sigmoid, $\odot$ is element-wise multiplication and $\operatorname{lin}_1$ and $\operatorname{lin}_2$  are two layers with $128$ units.
Then it applies a layer $\operatorname{lin}_{3}$ with 256 units. A final  layer with a single unit then produces the output, where its activation function is $\tanh$.
A batch normalization and leaky ReLU activation function are used after the two MPNNs and $\operatorname{lin}_1, \operatorname{lin}_2, \operatorname{lin}_3$ layers. 

{\bf Encoder for UL VAE.} It has the same architecture as the above discriminator, except that the final layer maps into a 256-dimensional vector with a linear activation function. This vector further splits to two 128-dimensional vectors: $\vz_\mu$ and $\vz_\sigma$. It then  generates the following latent vector: $\vz = \vz_\mu +  \exp(\frac{1}{2}\vz_\sigma)\odot \vx$, where $\vx\sim N(0, 1)$.

{\bf Generator for Adj GAN.} 
It contains four linear layers with $128, 256, 256, 512$ units with batch normalization and leaky ReLU activation function. The last layer generates a $9 \times 11$ tensor (matrix) for the node features and $9\times 9 \times 4$ tensor for the edge features. 
We use a hard Gumbel softmax to produce the one-hot feature vectors for the nodes and edges. 

{\bf Generator for UL GAN and decoder for UL VAE.} 
It takes a 128-dimensional vector and outputs a graph.
The generator contains the following layers: an initial MLP layer that takes the random noise vector and creates a 3-nodes graph with 256-dimensional node features; an MPNN layer with 128 units; an unpooling layer that maps the 3-nodes graph to a 4-or-5-nodes graph; an MPNN layer with 128 units; an unpooling layer that maps the 4-or-5-nodes graph to a 6-to-9-nodes graph; an MPNN layer with 64 units; a linear layer with 64 hidden units and two final layers that produces node and edge features in $\R^{10}$ and $\R^3$, respectively. 
A skip connection, which takes the input vector, is added to the node features after each unpooling layer. Finally, a hard Gumbel softmax generates the desired one-hot features. The dimension of the edge features in all intermediate graphs is $32$. 
The log probabilities from this sampling process are added to obtain the total log probability when updating the parameters according to \eqref{eqn:update}.

{\bf Training process.} For UL GAN and Adj GAN, the loss function corresponds to Wasserstein GAN with gradient penalty \cite{gulrajani2017Penalized_WGAN}.  For UL VAE, the loss function is the sum of the reconstruction errors of node features and edge features and the Kullback-Leibler divergence between the latent vector and a standard Gaussian. We use the Adam optimizer with a learning rate $2\times 10^{-4}$ for the generator and a learning rate $10^{-4}$ for the discriminator with a training batch size of 64. During the training process, we evaluate the model
 every 500 iterations and we report the result with optimal validity before a mode collapse occurs. In each training step,  
we alternatively minimize the loss function and update the parameters according to  the policy gradient procedure in \eqref{eqn:update}.

\subsection{Results}\label{subsec:result}

\begin{figure}[t]
\begin{center}
\includegraphics[width=7cm]{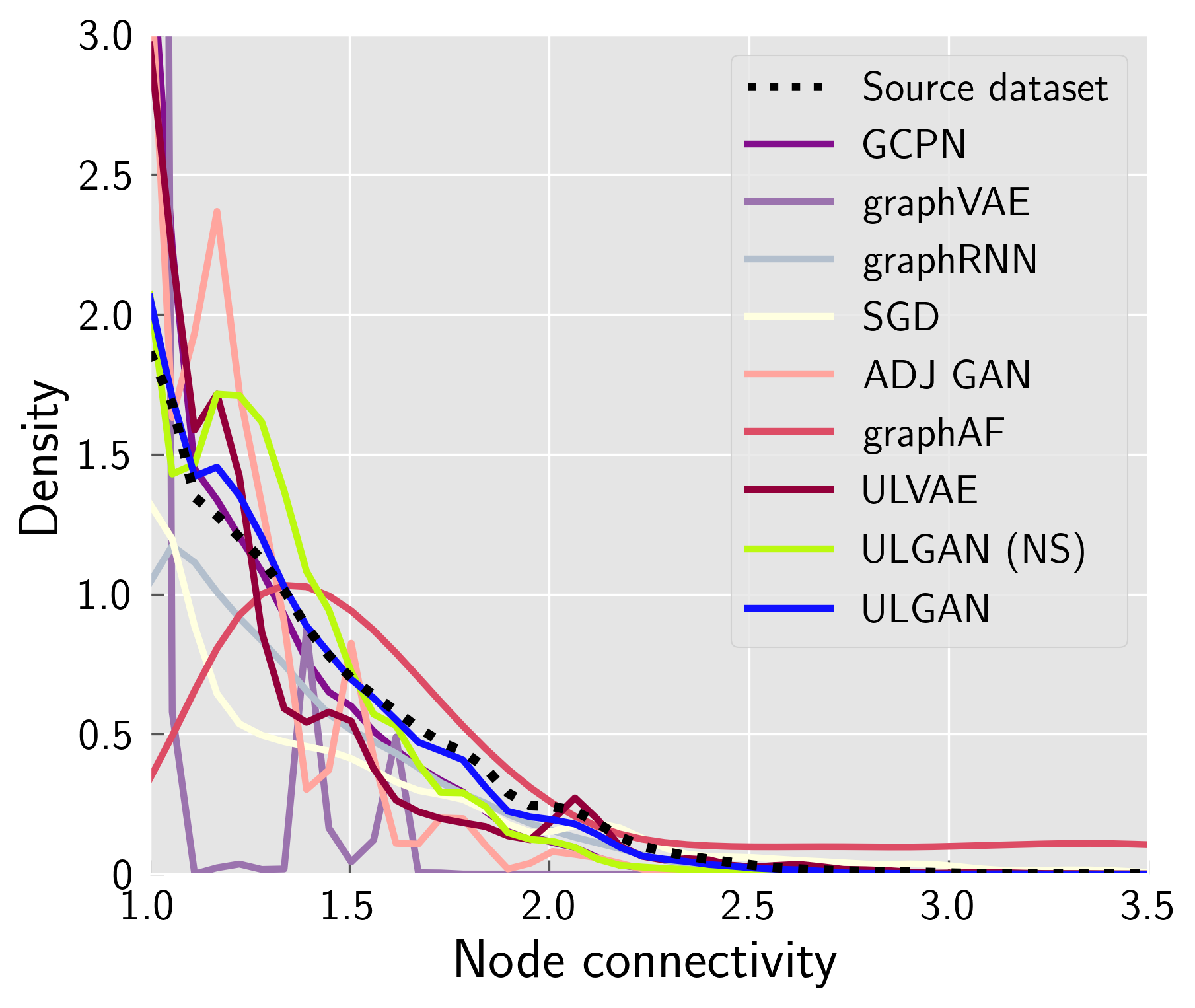}
\includegraphics[width=7cm]{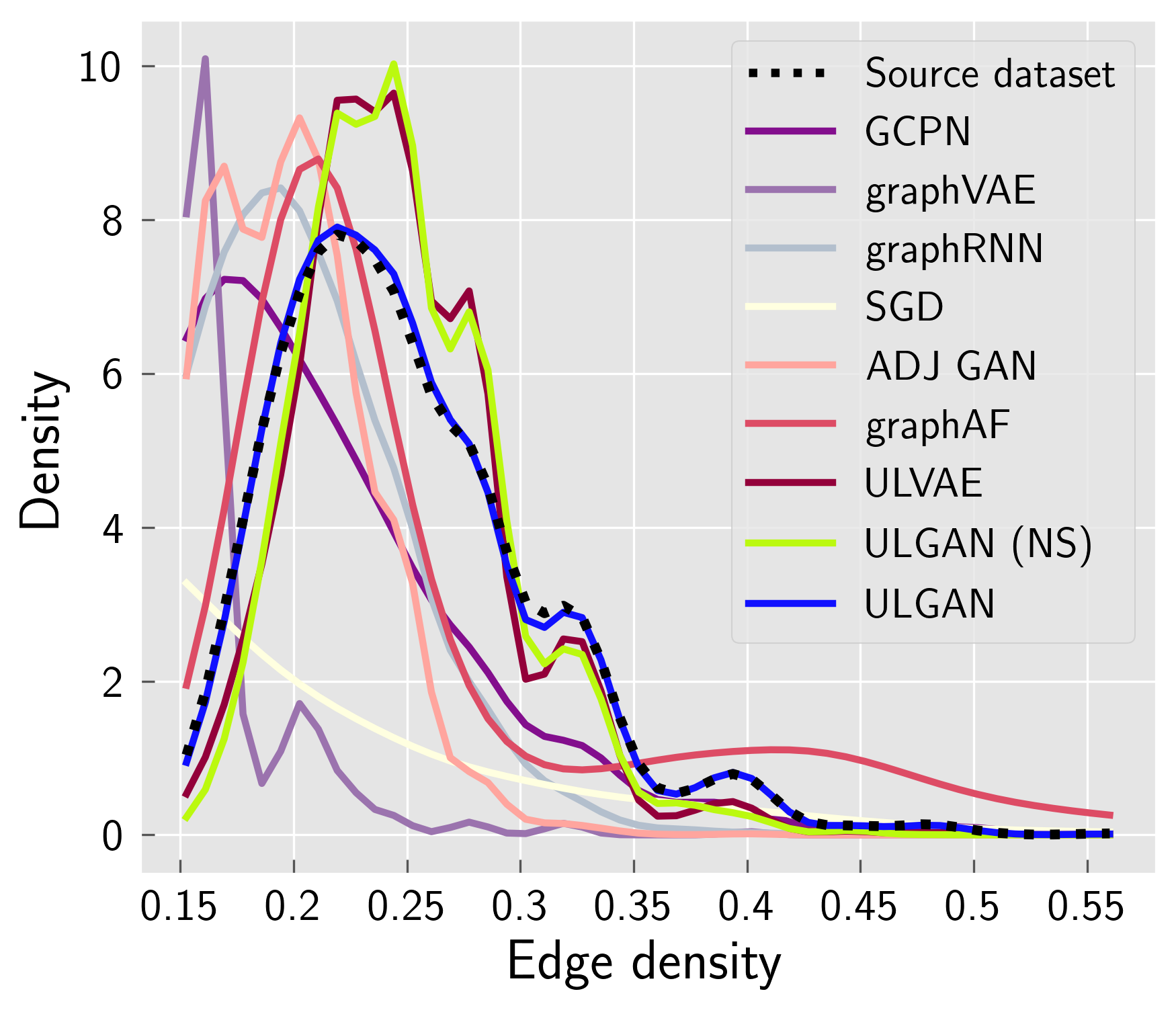}
\includegraphics[width=7cm]{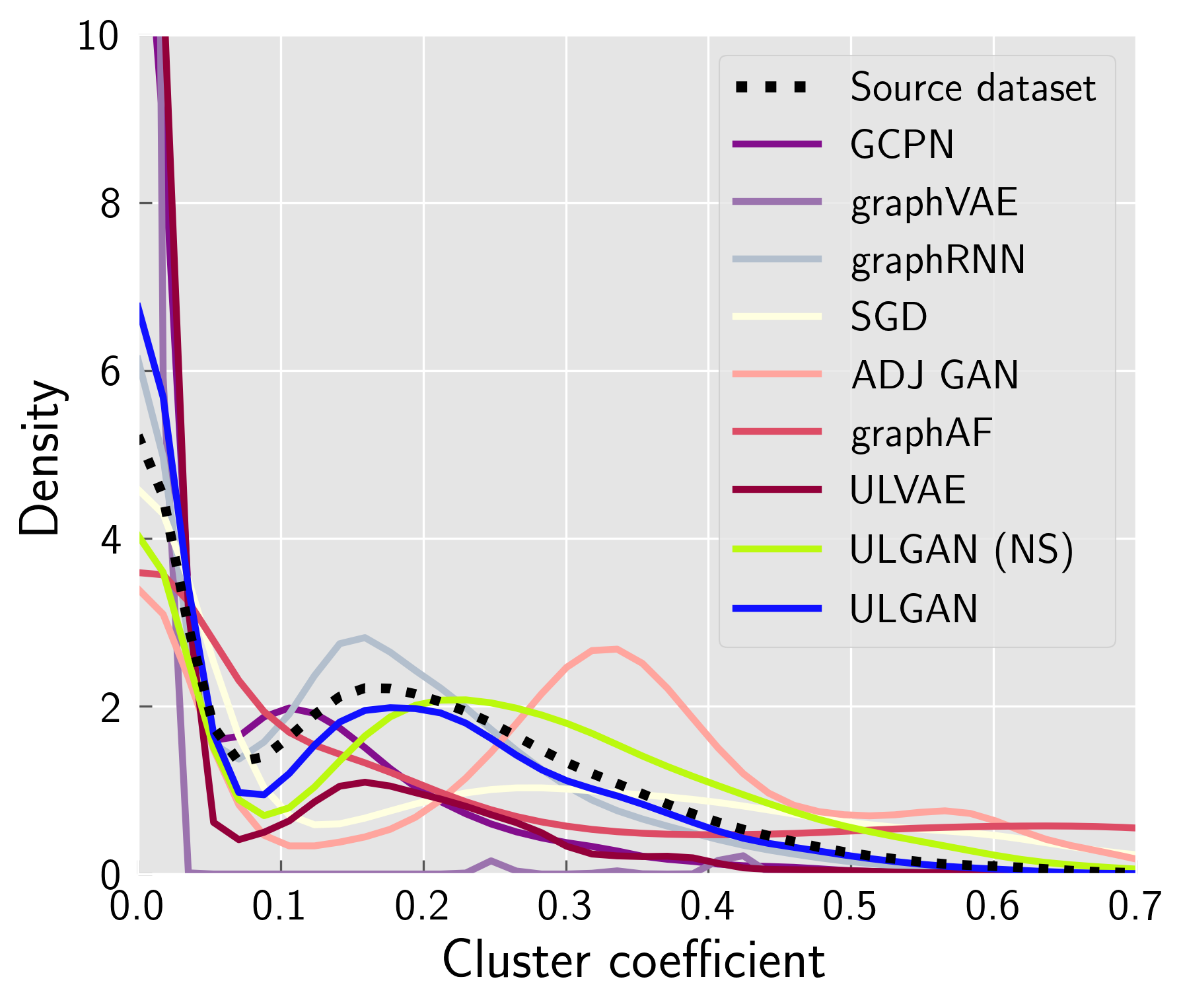}
\includegraphics[width=7cm]{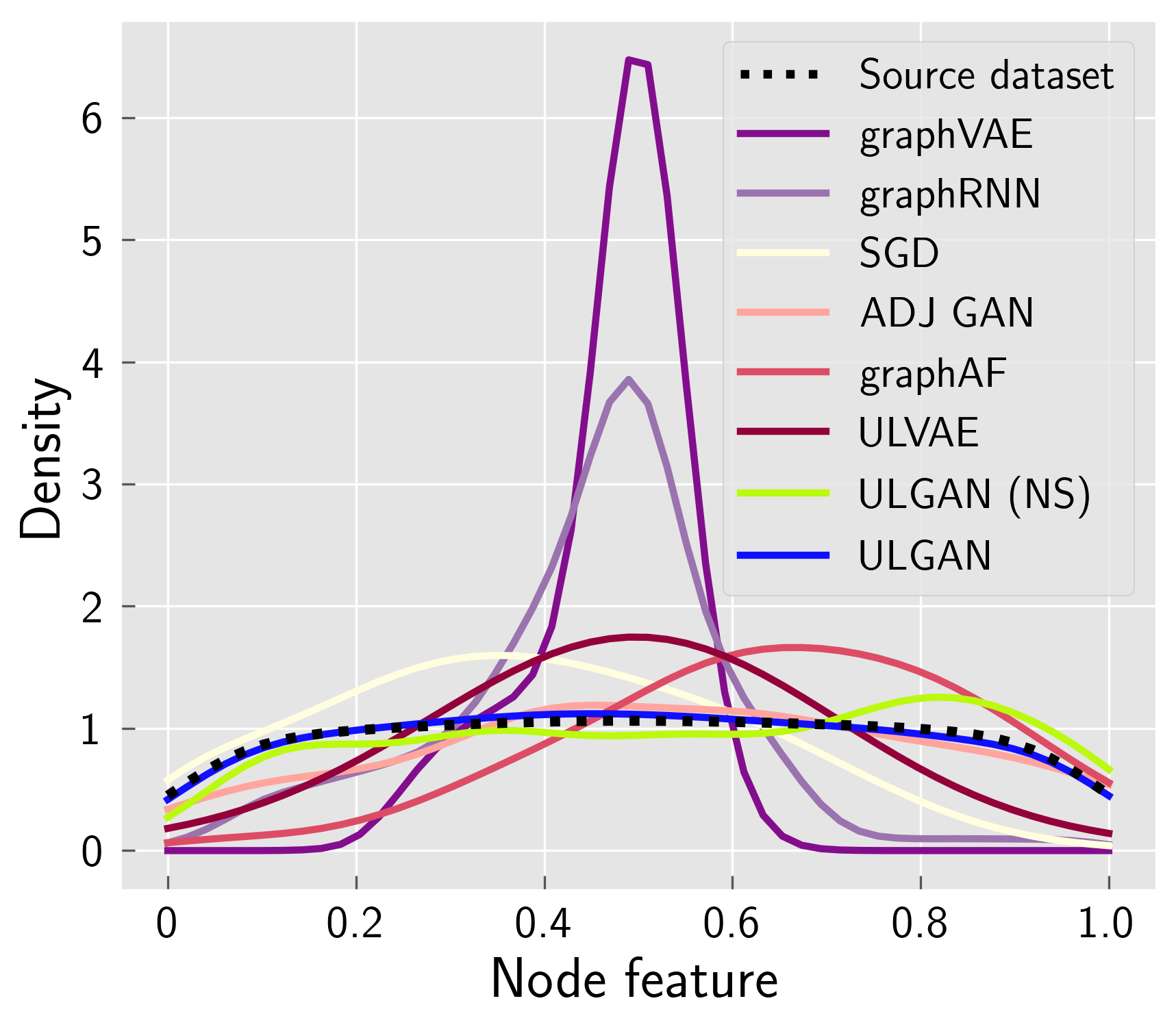}
\end{center}
\vskip -0.2cm
\caption{Distributions of the four graph properties for the generated and source data using the Waxman random graph dataset. The graph properties include average node connectivity (top left), average clustering coefficient (bottom left), edge density (top right), and node features  (bottom right).}
\label{fig:hist_waxman}
\end{figure}

For all datasets, we generate 10,000 samples for evaluation.
For the Waxman random graph and protein datasets, we compared Adj GAN, UL GAN and UL VAE with our implementations of the following baseline methods: GraphVAE \citep{simonovsky2018graphvae}, GraphRNN \cite{you2018graphrnn}, GCPN \citep{you2018gcpn}, GraphAF \citep{shi2020graphaf} and SGD-VAE \citep{guo2021deep}. 
Since GraphAF, GraphVAE and GCPN were designed for molecule generation, we adapted them to general graph generation. 
We remark that GCPN is not able to produce numeric node features and we thus do not report the node feature metrics for it. In order to test the effect of skips connections, we implemented for the random graph dataset a version of UL GAN that has no skip connections from the input vector and we refer to it as UL GAN (NS).

\begin{table*}[ht]
\caption{Results for the Waxman random graph dataset. We report the KL divergence and Wasserstein distance with respect to the following quantities: edge density (kl edge dense and wd edge dense), average node connectivity (kl conn and wd conn), average clustering coefficient (kl clust and wd clust) and node features (kl node feat and wd node feat). } 
\label{tab:waxman}
\begin{center}
\begin{small}
\begin{tabular}{lcccccccc}
\toprule
Methods & kl edge dense & kl clust & kl conn & kl node feat & wd edge dense & wd clust & wd conn & wd node feat\\\midrule
GraphVAE & 4.009 & 6.127 & 2.363 & 5.088 & 0.118 & 0.140 & 0.315 & 0.183 \\\hline
GraphRNN & 0.592 & 0.034 & 0.373 & 0.717 & 0.052 & \textbf{0.019} & 0.217 & 0.135 \\\hline
GCPN & 0.402 & 0.292 & 0.096 & N/A & 0.043 & 0.086 & 0.116 & N/A \\\hline
GraphAF & 0.507 & 0.249 & 0.460 & 0.382 & 0.032 & 0.079 & 0.503 & 0.163 \\\hline
SGD-VAE & 1.493 & 0.321 & 0.448 & 0.378 & 0.122 & 0.070 & 0.294 & 0.131 \\\hline
ADJ GAN & 1.128 & 1.081 & 0.247 & 0.196 & 0.065 & 0.097 & 0.178 & 0.088 \\\hline
UL VAE & 0.092 & 0.408 & 0.108 & 0.635 & 0.010 & 0.099 & 0.105 & 0.132 \\\hline
UL GAN (NS) & 0.100 & 0.098 & 0.088 & 0.147 & 0.011 & 0.076 & 0.046 & 0.056 \\\hline
UL GAN & \textbf{0.001} & \textbf{0.023} & \textbf{0.034} & \textbf{0.145} & \textbf{0.004} & 0.027 & \textbf{0.011} & \textbf{0.042} \\
\bottomrule
\end{tabular}
\end{small}
\end{center}
\end{table*}

\begin{table*}[ht]
\caption{Results for the protein dataset. We report the same quantities summarized in the caption of Table \ref{tab:waxman}.} 
\label{tab:protein}
\vskip 0.1in
\begin{center}
\begin{small}
\begin{tabular}{lcccccccc}
\toprule
Methods & kl edge dense & kl clust & kl conn & kl node feat & wd edge dense & wd clust & wd conn & wd node feat\\
\midrule
GraphVAE & 2.497 & 0.343 & 3.228 & 9.777 & 0.065 & 0.020 & 0.758 & 10.590 \\\hline
GraphRNN & 0.082 & 0.337 & 0.163 & 2.818 & 0.013 & 0.071 & 0.165 & 4.473 \\\hline
GCPN & 1.567 & 4.943 & 7.358 & N/A & 0.483 & 0.508 & 2.292 & N/A \\\hline
GraphAF & 1.942 & 1.659 & 1.987 & 2.603 & 0.043 & 0.163 & 1.543 & 17.217 \\\hline
SGD-VAE &  1.035& 1.228 & 0.975 & 10.195 & 0.169 & 0.311 & 1.549 & 10.698 \\\hline
ADJ GAN & 6.176 & 4.600 & 7.145 & 0.730 & 0.086 & 0.050 & 0.815 & 6.705 \\\hline
UL VAE & 0.492 & 0.889 & 0.791 & 0.181 & 0.041 & 0.072 & 0.373 & 4.344 \\\hline
UL GAN & \textbf{0.074} & \textbf{0.224} & \textbf{0.101} & \textbf{0.095} & \textbf{0.011} & \textbf{0.009} & \textbf{0.127} & \textbf{3.142} \\
\bottomrule
\end{tabular}
\end{small}
\end{center}
\vskip -0.1in
\end{table*}

Figure \ref{fig:hist_waxman} demonstrates the distributions of the four graph properties of both the source data and the different generating methods for the Waxman random graph dataset. For all four properties, the distribution obtained by UL GAN seems to be the closest to the source data.  
Table \ref{tab:waxman} reports the 8 evaluation metrics for the Waxman random graph dataset. We note that UL GAN outperforms the other methods in most of the metrics, except for Wasserstein distance between average clustering coefficients, where GraphRNN achieves the smallest metric. Nevertheless, ULGAN achieves the smallest KL divergence and its distribution seems to be closer to the source data according to Figure \ref{fig:hist_waxman}.
The better performance of UL GAN over Adj GAN and of UL VAE over GraphVAE indicates the clear advantage of using the unpooling layer over a standard adjacency-based method. 
We further note that the improvement of UL GAN over UL GAN (NS) is not significant. This indicates that even without the skip connection our model performs really well and that its main advantage is due to the unpooling layer and not the skip connection.

Table \ref{tab:protein} reports the evaluation metrics for the protein dataset. We note that UL GAN outperforms the other baseline methods in all the metrics. Also, the adjacency matrix-based methods (GraphVAE and Adj GAN) perform poorly in this dataset while their unpooling-layer-based counterparts (UL VAE and UL GAN) perform much better.

For QM9, we compare UL VAE, UL GAN, Adj GAN, MolGAN \citep{de2018molgan}, CharacterVAE \citep{gomez2018characterVAE}, GrammarVAE \citep{kusner2017grammar}, Graph VAE \citep{simonovsky2018graphvae}, GraphAF \citep{shi2020graphaf}, GraphDF \citep{luo2021graphdf}, MoFlow \citep{zang2020moflow}, Spanning tree \citep{ahn2021spanningTree}, and MGM \citep{mahmood2021masked}. 
For ZINC, we compare UL GAN, Adj GAN, CharacterRNN \citep{segler2018characterRNN}, LatentGAN \citep{prykhodko2019latentGAN}, junction tree VAE (JT VAE) \citep{jin2018junctionTree}, GraphAF \citep{shi2020graphaf}, GraphDF \citep{luo2021graphdf}, MoFlow \citep{zang2020moflow}, and Spanning tree \citep{ahn2021spanningTree}. We do not report results of UL VAE for ZINC because its training was slow and we do not have results for its counterpart, GraphVAE. For Adj GAN, UL GAN and UL VAE, we report the means based on 100 runs of generating 10k samples. The results of the other baseline methods are copied from their original papers. Standard deviations for our implementations are included in the supplementary materials.

Table~\ref{tab:qm9} and Table~\ref{tab:zinc} report validity, uniqueness, novelty and their geometric mean for QM9 and ZINC, respectively. 
For QM9, UL GAN 
improves significantly from Adj GAN. Its performance is overall competitive when compared to other state-of-the-art approaches. In particular, UL GAN achieves the third-highest geometric mean.  
UL VAE significantly outperforms its adjacency-matrix-based counterpart, GraphVAE.
For ZINC, UL GAN achieves perfect uniqueness and novelty scores. In terms of validity, it outperforms Adj GAN, whose validity and uniqueness scores are poor. We thus note that our unpooling layer is able to generate graphs of moderate sizes, while adjacency-matrix generators are only suitable for small graphs. Although some other methods achieve better validity, the overall performance of UL GAN is comparable with state-of-the-art methods.

\begin{table}[ht]
\caption{Validity, uniqueness, novelty and their geometric mean (G-mean) for molecule generation using QM9.  Scores for the competing methods (listed above the middle line) were copied from their original papers. } 
\label{tab:qm9}
\vskip 0.15in
\begin{center}
\begin{small}
\begin{tabular}{lccccr}
\toprule
Method & Valid & Unique & Novel & G-mean \\
\midrule
CharacterVAE & 0.103 & 0.675 & 0.900 & 0.397 \\
GrammarVAE   & 0.602 & 0.093 & 0.809 & 0.356\\
GraphVAE     & 0.557 & 0.760 & 0.616 & 0.639\\
MolGAN       & 0.981 & 0.104 & 0.942 & 0.458\\
GraphAF & 0.670 & 0.945 & 0.888 & 0.825\\ 
GraphDF & 0.827 & 0.976 & 0.981 & 0.925\\
MoFlow & 0.962 & \textbf{0.992} & \textbf{0.980} & \textbf{0.978}\\
Spanning tree & \textbf{1.00} & 0.968 & 0.727 & 0.889\\ 
MGM & 0.886	& 0.978	 & 	0.518 & 0.766\\
\hline
UL VAE & 0.735  & 0.940 & 0.949 & 0.869 \\
Adj GAN      & 0.941 & 0.139 & 0.886  & 0.488\\
UL GAN & 0.907 & 0.826 & 0.949 & 0.893\\
\bottomrule
\end{tabular}
\end{small}
\end{center}
\vskip -0.1in
\end{table}

\begin{table}[ht]
\caption{
Validity, uniqueness, novelty and their geometric mean  (G-mean) for molecule generation using ZINC.  Scores for the first three methods were copied from \citet{polykovskiy2020molecular} and scores for other competing methods (listed above the middle line) were copied from their original papers.
}
\label{tab:zinc}
\vskip 0.15in
\begin{center}
\begin{small}
\begin{tabular}{lccccr}
\toprule
Method & Valid & Unique & Novel & G-mean \\
\midrule
CharRNN	 & 0.975 & \textbf{1.00} & 0.842 & 0.936\\
LatentGAN & 0.897 & 0.997 & 0.950 & 0.947\\
JT VAE & \textbf{1.00} & \textbf{1.00} & 0.914 & 0.970\\
GraphAF & 0.680 & 0.991 & \textbf{1.00} & 0.877 \\ 
GraphDF & 0.890 & 0.992 & \textbf{1.00} & 0.959\\
MoFlow & 0.818 & \textbf{1.00} & \textbf{1.00}& 0.935\\
Spanning tree & 0.995 & \textbf{1.00} & 0.999&\textbf{0.998} \\ 
\hline
Adj GAN   & 0.109 & 0.196  & \textbf{1.00} & 0.277\\
UL GAN & 0.871 & \textbf{1.00} & \textbf{1.00} & 0.955\\
\bottomrule
\end{tabular}
\end{small}
\end{center}
\vskip -0.1in
\end{table}

\begin{figure*}[ht]
    \begin{center}
\includegraphics[scale=0.6]{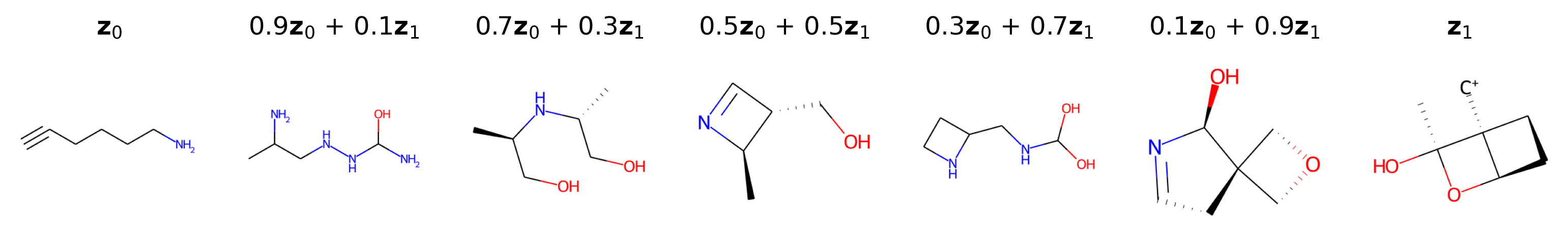}
  \vskip -0.43cm
\includegraphics[scale=0.6]{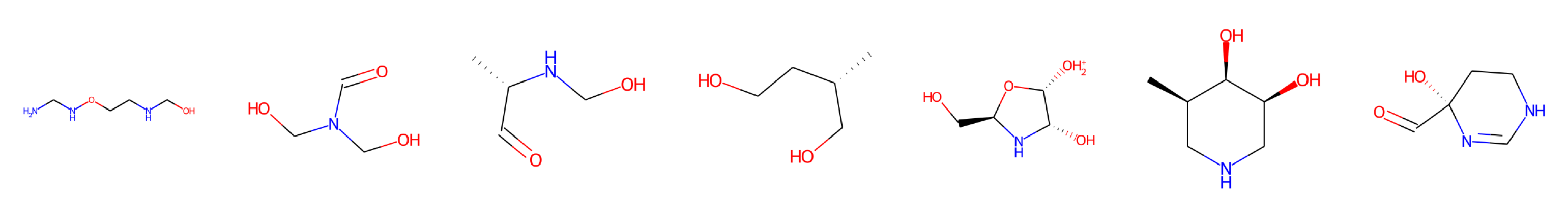}
  \vskip -0.43cm
\includegraphics[scale=0.6]{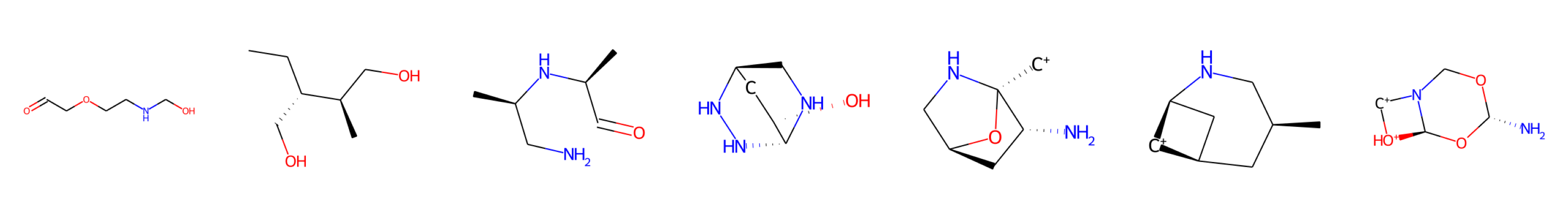}
  \vskip -0.43cm
\end{center}
\caption{\label{fig:latent_GAN}
Demonstration of gradual change between two different molecular properties when applying  UL GAN for QM9. Left: 3 string-like molecules (with latent vector $\vz_0$), Right: 3 molecules with rings (with latent vector $\vz_1$). We gradually change the latent vector (on top) and notice the gradual change of the 3 molecules from having string-like structure to having ring-like structure.}
\end{figure*}

Figure \ref{fig:latent_GAN}
studies the latent space structure of UL GAN for QM9. It picks two latent vectors $\vz_0$ and $\vz_1$ corresponding to a string-like molecule and a molecule with a ring, respectively. Then, it forms a series of latent vectors $\vz_t = t \vz_1 + (1-t) \vz_0$ and shows their corresponding molecules. We note that as $t$ increases the molecular structures are changing from string-like to ring-like ones.

\section{CONCLUSION AND LIMITATIONS}\label{sec:conclusion}

We introduced a novel unpooling layer that can enlarge a given graph. 
We have proved that this unpooling layer generates connected graphs and its range covers all possible connected graphs. We utilized such layers within GANs and VAEs and tested their performance for graph, protein and molecule generation. Our unpooling-based generation outperforms other methods for graph and protein generation and is competitive 
for molecule generation.  In particular, a significant improvement was noticed in comparison to other methods that are based on adjacency matrices, such as Adj GAN, MolGAN and Graph VAE.

The unpooling layer can be used for other purposes. In future work, we plan to apply it to generic graph reconstruction, while considering the particular applications of recommender systems and graph anomaly detection. We also plan to apply it to conditional graph generation, that is,  generating graphs with some desired properties.

The unpooling layer has some limitations. First, its training requires relatively large computational resources, since it relies on various probabilities in order to determine the graph structure. Second, it becomes difficult to optimize several unpooling layers within the generative model because the log probabilities are added and backpropagated in all the unpooling layers. 

Our work has practical applications of societal impact, such as drug discovery. 
However, our main focus has been on general graph generation and in order to have a stronger impact on drug discovery we need to carefully specialize the method for this purpose. For example, we need to improve the druglikeness, or other desired chemical properties, of the generated molecules.

\subsubsection*{Acknowledgements}
This work was partially supported by NSF award DMS 2124913 and the Kunshan Municipal Government
research funding.

\bibliography{ref}
\bibliographystyle{refstyle}

\newpage
\appendix

\begin{center}
\textbf{\large SUPPLEMENTARY MATERIALS}
\end{center}
We supplement the main text as follows: \S\ref{sec:UL_method} includes more details on the unpooling layer; \S\ref{sec:code} discusses the tested codes; \S\ref{sec:details} describes implementation details of both the unpooling layer and the generative neural network; \S\ref{sec:theorem} carefully validates the theory stated in \S\ref{sec:TheoreticalResult}; \S\ref{sec:opt_chem_prop} reports molecule generation results in a particular setting that aims to optimize specific chemical properties;
and \S\ref{sec:additional_result} reports some additional numerical results that supplement the ones in \S\ref{sec:experiment}.

\section{ADDITIONAL DETAILS ABOUT THE UNPOOLING LAYER}\label{sec:UL_method}

We describe in more details the unpooling layer mechanism. For completeness, we repeat details explained in \S\ref{subsec:unpooling}.

{\bf Notations}. For an unpooling layer $U$, the input featured graph is denoted by $\gG = (V, E, \mX, \mW)$, which is an undirected graph $(V,E)$ with feature matrices$\mX$ and $\mW$ for nodes and edges respectively. More specifically, $V = [N] := \{1, 2, \cdots N\}$ are node indices, $E$ is the edge set whose members are of the form $\{i, j\}$ where $i, j \in V$, and $M = |E|$ is the number of edges in the graph. $\mX\in \sR^{N\times d}$ is the node features matrix whose $i$-th row $\vx_i$ is a $d$-dimensional vector for node $i$. $\mW\in \sR^{M\times e}$ is the edge features matrix whose rows are $e-$dimensional vectors for edge features and we refer to the feature vector on edge $\{i, j\}\in E$ as $\vw_{i, j}$.

{\bf Detailed Steps of the Unpooling Layer.} 
We provide more details of the unpooling layer according to the following steps. Recall that the unpooling layer enlarges an input featured graph $\gG = (V, E, \mX, \mW)$ to a graph $\gG^o=(V^o, E^o, \mY, \mU)$, while using the set of unchanged (or stable) nodes, $I_s'$, the set of unpooled nodes, $I_u'$, and the set of nodes that may be unpooled, $I_r'$.

{\bf Step 1. Generating nodes and node features.} We first probabilistically determine the eventual set of nodes to be unpooled and consequently form the set of nodes in the output graph. We finally construct node features.

{\bf Step 1a. Determining nodes to be unpooled.} We determine a set of nodes to be unpooled and a set of nodes to stay unchanged based on $I_s', I_r'$ and $I_u'$. For each node $i\in I_r'\subset V$, we determine the probability of being unpooled by
$$
p_{r}(\vx_i) = \operatorname{MLP-}R(\vx_i).
$$
A node $i\in I_r'$ is unpooled if a randomly drawn uniform random variable $U_i\sim U[0, 1]$ is smaller than $p_{r}(\vx_i)$. We form a set of nodes to be unpooled as 
$$
I_u = I_u'\cup \{i: i\in I_r' , U_i<p_r(\vx_i)\}.
$$
The rest nodes stay unchanged and form a set $I_s:=V\setminus I_u$. 
The total log probability in this step is 
\begin{equation}
    \textrm{logP-R} = \sum_{i\in I_r'} \log\big(p_r(\vx_i)\mathds{1}_{U_i < p_r(\vx_i)}  + (1-p_r(\vx_i))\mathds{1}_{U_i \geq p_r(\vx_i)}\big)\label{eqn:logPR}
\end{equation}

We remark that it is possible to have $I_r' = \emptyset$ (e.g. the unpooling layers used for the protein dataset), thus all the unpooling nodes are pre-determined. In this case, we obtain $I_u:=I_u'$ and $I_s=I_s'$ without using $\operatorname{MLP-}R$ and we do not establish the log probability $\textrm{logP-R}$. 

{\bf Step 1b. Unpooling nodes and constructing node features}. 
From the previous step, we obtain $I_s$, a set of nodes to stay unchanged, and $I_u$, a set of nodes to be unpooled. We thus clarified the set $V^o$, where we note that $|V^o|=|I_s| + 2|I_u|$. Given two children nodes in $V^o$, we refer to the ``subdivided'' node in $V$ as their parent node. We define a map $f$ that assigns to each node in $i \in I_s$ 
the corresponding index, $f(i) \in V^o$, and to each node in $i \in I_u$ the set of two indices of its children nodes in $V^o$; in the latter case we denote $f(i)  =  \{f_1(i), f_2(i)\}$.
The feature vectors of the nodes in $V^o$  are obtained by an MLP for vertices, $\operatorname{MLP-}y$. To produce different features for the children nodes $f_1(i)$ and $f_2(i)$, we apply $\textrm{MLP-}y$ to two different pieces of the node feature vector $\vx_i$. More specifically, we use the orthogonal projectors, $P_{S_1}$ and $P_{S_2}$, onto two fixed subspaces, $S_1$ and $S_2$ in $\mathbb{R}^d$ (which are later specified in \S\ref{sec:details}), as follows:
\begin{align*}
    \vy_{f(i)} & = \textrm{MLP-}y(P_{S_1} \vx_{i}), &\text{ for }&~ i\in I_s;\\
    \vy_{f_{j}(i)} & = \textrm{MLP-}y(P_{S_j}\vx_{i}), &\text{ for }&~ j=1, 2,~ i\in I_u.
\end{align*}

One can also use three different MLP-y's for the static nodes and the two different children nodes.

{\bf Step 2. Building edges.} Next, the unpooling layer takes the generated children nodes in $V^o$ and their node features  and builds edges for them. It first builds intra-links, which are edges within the pairs of children nodes; it then builds inter-links, which are edges between pairs of children nodes and images (according to the function $f$) of neighbors of their parent node (in $I_u$). At last, a single step ensures the connectivity of the output graph. 

We also aggregate all log-probabilities in all three steps and use that in the policy gradient algorithm to train the layer and tune those probabilities for edge construction.

We start the edge generation process with $E^o = \emptyset$. 

{\bf Step 2a. Building intra-links.} We  
aim to probabilistically sample a set  $V_c$ of nodes in $I_u \subset V$ whose two children (in $V^0$) will be connected to each other by intra-links. We first initialize  $V_c = \emptyset$. For each node $j\in I_u$ with feature vector $\vx_j$ we add node $j$ into $V_c$ with a probability outputted by an MLP for intra-links,  $p_c(j) =\textrm{MLP-IA}(\vx_j)$. 
That is,  
we draw a uniform random variable, $U\sim U[0, 1]$, and if $U<p_c(j)$, then $V_c = V_c\cup\{j\}$ and $\{f_1(j), f_2(j)\}$ is an intra-link, which is added to $E^o$; otherwise, $V_c$ is unchanged and no intra-link is added.

We add all the intra-link to edge set in output graph
$$
E^o = E^o\cup \big\{\{f_1(j), f_2(j)\}, \forall j \in V_c\big\}.
$$

We track the log probabilities for later use
as follows:
\begin{equation}
    \textrm{logP-IA} := \sum_{j\in V_c} \ln (p_c(j))  + \sum_{j\notin V_c} \ln (1-p_c(j)).\label{eqn:logPIA}
\end{equation}

{\bf Step 2b. Finding a shared node for disconnected children.} 
For disconnected children pairs, we designate a node that will connect to these children. In the next step these nodes and possibly additional ones will be connected to a subset of the children pairs. 
For each node $j\in I_u \setminus V_c$, 
we denote the set of all edges in $E$ that are connected to $j$ by $E_j = \mleft\{ \{i_1, j\}, ...\{i_{m_j}, j\}\mright\}$, 
and we calculate probabilities of selecting those edges as 
\begin{align*}
    &\mleft(p_b(j, i_1), p_b(j, i_2), ... p_b(j, i_{m_j})\mright) \propto 
    \big(\textrm{MLP-C}(\vy_{f_1(j)}, \vy_{f_2(j)}, \vw_{\{j, i_1\}}, \vx_{i_1}), \\
     &\quad\textrm{MLP-C}(\vy_{f_1(j)}, \vy_{f_2(j)}, \vw_{\{j, i_2\}}, \vx_{i_2}), 
     ..., ~\textrm{MLP-C}(\vy_{f_1(j)}, \vy_{f_2(j)}, \vw_{\{j, i_{m_j}\}}, \vx_{i_{m_j}})\big).
\end{align*}

We draw from uniform distribution $U\sim U[0, 1]$ and let
$$
   b_j = \left\{\begin{array}{ccl}
i_1, &&\text{if } U< p_b(j, i_1); \\
 i_k, &&\text{if } U\in \mleft[\sum_{l=1}^{k-1} p_b(j, i_l), \sum_{l=1}^{k} p_b(j, i_l)\mright); \\
  i_{m_j}, &&\text{otherwise.}    
  \end{array}\right.
$$
Then, we 
define $N_{\{b_j, j\}}(j) := \{f_1(j), f_2(j)\}$.

In the next step, each edge $\{i, j\}$ in $E$ gives rise to adding edges between $N_{\{i, j\}}(i)$ and $N_{\{i, j\}}(j)$ in $\gG^o$. Therefore, our approach ensures that $f_1(j)$ and $f_2(j)$ both connect to $N_{\{b_j, j\}}(b_j)$ so $\gG^o$ is connected.

We track the log-probabilities from this step that ensures connectivity as follows:
\begin{equation}
\textrm{logP-C} := \sum_{j\in I_u, j\notin V_c} \ln ( p_b(j, b_j))
    \label{eqn:logPC}
\end{equation}

{\bf Step 2c. Building inter-links.} For each edge $\{i, j\} \in E$ and node $j$, we first generate a set of nodes $N_{\{i, j\}}(j)\subset V^o$ as follows.
For $j\in I_u\setminus V_c$ and $i=b_j$, we have already defined $N_{\{i, j\}}(j)$ in the above step.
If $j\in I_s$, we let $N_{\{i, j\}}(j) = \{f(j)\}$. If $j\in I_u$, we first
calculate probabilities from an MLP of intra-links as follows: $\mleft(p_1(i, j), p_2(i, j)\mright) = \textrm{MLP-IE}(\vy_{f_1(j)}, \vy_{f_2(j)}, \vw_{\{i,j\}}, \vx_i)$. Then, we draw a random variable $U\sim U(0, 1)$ and let
$$
N_{\{i, j\}}(j)  =\left\{ 
\begin{array}{lcl}
\{f_1(j)\}, &&\text{if } U< p_1(i, j); \\
    \{f_1(j), f_2(j)\}, &&\text{if } U\geq p_1(i, j)+p_2(i, j); \\
    \{f_2(j)\}, &&\text{otherwise.} 
\end{array}
    \right.
$$

We similarly form $N_{\{i, j\}}(i)$ by swapping the $j$ and $i$ nodes.

The log probability for each edge-node pair $(\{i, j\}, j)$ is
\begin{align}
&\textrm{logP-IE}(\{i, j\}, j):=     \Big(
\mathds{1}_{N_{\{i, j\}}(j) = \{f_1(j)\}} \, \ln (p_1(i, j)) + \mathds{1}_{N_{\{i, j\}}(j) = \{f_2(j)\}} \, \ln (p_2(i, j)) \label{eqn:logIEij}\\ 
\notag
& \quad + 
\mathds{1}_{N_{\{i, j\}}(j) = \{f_1(j), f_2(j)\}} \, \ln (1-p_1(i, j)-p_2(i, j))
\Big).
\end{align}

Lastly, we add to $E^o$ all inter-links, i.e., all possible edges that connect nodes  in $N_{\{i, j\}}(i)$ and $N_{\{i, j\}}(j)$:
\begin{equation}
\label{eq:add_inter_link_with_N}
E^o = E^o \cup \mleft\{\{k, l\}, \ \forall k\in N_{\{i, j\}}(i), \ l\in N_{\{i, j\}}(j)\mright\}.
\end{equation}

Let $A_j:= \{j\in V_c\}$ and $A_{i, j}:= \{j\in I_u,~j\notin V_c,~i\neq b_j\}$. The cumulative log-probability for inter-links outputted by $\textrm{MLP-IE}$ is
\begin{equation}
    \textrm{logP-IE} = \sum_{\{i, j\}\in E} \mathds{1}_{A_j \cup A_{i, j}}\textrm{logP-IE}(\{i, j\}, j)  +\mathds{1}_{A_i \cup A_{j, i}}\textrm{logP-IE}(\{i, j\}, i).\label{eqn:logPIE}
\end{equation}

{\bf Step 2d. Building additional edges between children node pairs.} 
We insert additional edges between children nodes in order to assure the expressivity of the output graph (this will be clarified in the proof of Theorem~2).
Let $E_u = \{\{i, j\}\in E: i\in I_u, j\in I_u\}\subset E$ denote the collection of edges whose two ends are both unpooling nodes. We initialize $E_a:=\emptyset$. For each edge $\{i, j\}\in E_u$, we generate a probability of adding one more additional edge using an MLP as follows: $p_{a}(\{i, j\}) = \textrm{MLP-IE-A}(\vx_i, \vx_j, \vw_{\{i, j\}})$. We draw a random variable $U_1\sim U[0, 1]$ and if $U_1<p_a(\{i, j\})$, we let $E_a: = E_a\cup \{\{i, j\}\}$ and $E^o : = E^o \cup E^o_a$, where $E^o_a$ defined in the following three cases: (1) if $|N_{\{i, j\}}(i)| = |N_{\{i, j\}}(j)| = 1$, then $E^o_a(\{i, j\}):=\{\{k, l\}, k\in f(i), l\in f(j), \text{ and } k\notin N_{\{i, j\}}(i), l\notin N_{\{i, j\}}(j)\}$; (2) if $|N_{\{i, j\}}(i)| +|N_{\{i, j\}}(j)| = 3$, without loss of generality, assume $|N_{\{i, j\}}(i)|=1$ and $|N_{\{i, j\}}(j)|=2$, then draw $U_2\sim U[0, 1]$. If $U_2<\frac{p_1(i, j)}{p_1(i, j) + p_2(i, j)}$ ($p_1(i,j)$ and $p_2(i,j)$ were obtained in Step 3c), we set $r_{ij}(j):= 1$, otherwise $r_{ij}(j):=2$. We then define $E^o_a(\{i, j\}):=\{\{k, f_{r_{ij}(j)}(j)\}$, $k\in f(i)$,   and  $k\notin N_{\{i, j\}}(i)\}$; (3) if $|N_{\{i, j\}}(i)| +|N_{\{i, j\}}(j)| = 4$, then $E^o_a(\{i, j\}) := \emptyset$. We thus updated the edge set of the output graph as follows
$$
E^o := E^o \bigcup \bigcup_{\{i, j\}\in E_a} E^o_a(\{i, j\}).
$$
We also record the total log probability of this step:
\begin{align}
    \textrm{logP-A} & = \sum_{\{i, j\}\in E_a}\ln p_a(\{i, j\}) + \sum_{\{i, j\}\notin E_a, \{i, j\}\in E_c}\ln (1- p_a(\{i, j\})) \notag \\
    & + \sum_{\substack{\{i, j\}\in E_a, \\|N_{\{i, j\}}(i)|=1, \\|N_{\{i, j\}}(j)|=2}} \ln  \frac{p_{r_{ij}(j)}(i, j)}{p_1(i, j) + p_2(i, j)}.
    \label{eqn:logIA}
\end{align}

{\bf Summary of step 2.} The edges of the output graph include all intra-links and inter-links as follows:
\begin{align*}
    E^o = & \Big\{
            \{f_1(j), f_2(j)\}: j\in I_u, j\in V_c\Big\}  \\
            &\bigcup \Big\{\{k, l\}: \{i, j\}\in E,  k\in N_{\{i,j\}}(i), l\in N_{\{i, j\}}(j)\Big\}\\
             &\bigcup \bigcup_{\{i, j\}\in E_a} E^o_a(\{i, j\}).
\end{align*}

{\bf Step 3. Constructing edge features.} Using the node features and edges from obtained from the previous steps, for each edge $\{k, l\}\in E^o$, we build the corresponding edge feature $\vu_{k, l} = \textrm{MLP-}u(\vy_k, \vy_l)$.

{\bf The overall probability for updating the training parameters:}
The final probability, $\textrm{P}$, is the product of all probabilities from step 1a and step 2a-2d. We obtain its logarithm, $\textrm{logP}$,  by combining  \eqref{eqn:logPR} - \eqref{eqn:logIEij} and \eqref{eqn:logPIE} as follows:
\begin{equation*}
    \textrm{logP} = \textrm{logP-R} + \textrm{logP-IA}  + \textrm{logP-C} + \textrm{logP-IE} + \textrm{logP-A}.
\end{equation*}
This probability is used to update the training parameters in the unpooling layer while using the REINFORCE algorithm.

\section{COMMENTS ON THE CODES} 
\label{sec:code}

We included a zipped folder of codes that implement the proposed method and some baseline methods. It also contains notebooks that implement numerical experiments for the Waxman random graph, protein, QM9, and ZINC datasets. All the implemented codes can be found in \url{https://github.com/guo00413/graph_unpooling}.

We implemented five baseline methods in the numerical experiments for the Waxman random graph and the protein dataset: GraphVAE \citep{simonovsky2018graphvae}, GraphRNN \citep{you2018graphrnn}, GCPN \citep{you2018gcpn}, GraphAF \citep{shi2020graphaf} and SGD-VAE \citep{guo2021deep}. We implement GraphVAE and GraphRNN based on codes from \url{https://github.com/JiaxuanYou/graph-generation} (licensed by MIT), we implement GCPN based on codes from \url{https://github.com/bowenliu16/rl_graph_generation} (licensed by BSD 3-Clause). We implement SGD-VAE and GraphAF from scratch, because we do not find codes with the appropriate license to use and modify.
We select model hyperparameters based on their original papers.

We remark that GCPN can not produce numeric node features due to its generation strategy. As a result, we do not report the metric related to node features for GCPN. 
To clarify, GCPN \citep{you2018gcpn} enlarges a single node at each time based on the existing graph as follows: it first adds 9 new nodes, corresponding to 9 types of heavy atoms in ZINC dataset; it constructs one edge connecting one of the new nodes to one of the nodes in the existing graph; it further sequentially determines to stop this step or to construct more edges. This approach can naturally handle graphs with categorical node features, but it fails to generate numeric node features. Therefore in the numerical experiments, we only generate a non-featured graph from GCPN and do not report metrics related to node features.

\section{IMPLEMENTATION DETAILS}\label{sec:details}

Section \ref{subsec:prelim_n_layer_other} introduces layers used in the experiments other than the unpooling layer; \S\ref{subsec:unpoolinglayers} provides details of the components in the unpooling layer; \S\ref{subsec:discriminators} describes the architectures of the discriminators in UL GAN and of the encoders in UL VAE; \S\ref{subsec:generators} details the architectures of the generative networks, including generators in UL GAN and decoders in UL VAE; and \S\ref{subsec:reconstruction_VAE} explains how we calculate the reconstruction error for UL VAE.

We remark that all numerical experiments are run in a machine with Intel(R) Xeon(R) Gold 6230 CPU @ 2.10GHz and NVIDIA TITAN RTX (576 tensor cores, and 24 GB GDDR6 memory).

\subsection{Details of Other Neural Layers}\label{subsec:prelim_n_layer_other}

{\bf Aggregation.} We introduce the following aggregation function $\operatorname{agg} (\vx_1, \vx_2) = \textrm{LeakyReLU}(\vx_1+\vx_2).$ It is used when the inputs are order invariant, for example, to produce edge feature based on features of two end nodes.

{\bf MLP.} We denote a multilayer perceptron by $\textrm{MLP}[k_0, k_1, k_2, ..k_m]$.
It contains $m-1$ hidden blocks, where the $i$-th block, $i=1,\dots,m-1$, cascades a linear layer with input dimension $k_{i-1}$ and output dimension $k_i$, a batch normalization and a LeakyReLU activation with negative slope $0.05$ (the same slope is used for all LeakyReLU activations in our experiments); and an output block which is a linear layer with output dimension $k_m$.

{\bf Initial layer}. An initial layer of a generator takes a latent vector $\vz$ and produces a 3-nodes graph. 
To define an initial layer, we need to specify the dimension of the input vector, $d_{in}$, the dimension of the output node feature, $d_x$, the dimension of the output edge feature, $d_w$.

The initial layer generates a 3-nodes graph as follows: First, a multilayer perceptron $\textrm{MLP-INI-V} = \textrm{MLP}[d_{in}, 6d_x, 3d_x]$ takes a $d_{in}$-dimensional input $\vz$  and generates the initial node features for the 3-nodes graph. It further reshapes it to a matrix in $\R^{3\times  d_x}$, where $d_x$ is the dimension of the output node features. Second, it calculates the probabilities from $(p_1, p_2, p_3, p_4) = \textrm{MLP-INI-E}(\vx_1, \vx_2, \vx_3) = \textrm{softmax}(\textrm{MLP}[3d, 16, 4])(\vx_1, \vx_2, \vx_3)$, draws a uniform random variable $U\sim U[0, 1]$, and determines the initial edge set to be
$$
E= \left\{ \begin{array}{rcl}
     \mleft\{\{1,2\}, \{1, 3\}\mright\},&&  \text{ if } U<p_1,\\
     \mleft\{\{1,2\}, \{2, 3\}\mright\},&&  \text{ if } p_1\leq U<p_1+p_2\\
     \mleft\{\{1,3\}, \{2, 3\}\mright\},  &&  \text{ if } p_1+p_2\leq U<p_1+p_2+p_3 \\
     \mleft\{\{1, 2\}, \{2, 3\}, \{1, 3\}\mright\}, &&\text{ otherwise } 
\end{array}\right.$$
The log probability from the initial layer is $\ln(p_s)$, where $s\in\{1, 2 ,3 ,4\}$ corresponding to the determined initial edge set.
Third, it constructs edge feature vectors using a multilayer perceptron as follows: $\textrm{MLP-INI-W}(\vx_i, \vx_j)=
\textrm{LeakyReLU}(\textrm{BN}(\textrm{MLP}[d, d_w, d_w](\textrm{agg}(\vx_i, \vx_j))))$, for each $\{i, j\} \in E$, where $\textrm{BN}$ is a batch normalization.

{\bf Skip connection}. The generator of UL GAN adopts a skip connection procedure \citep{Brock2019LargeGANSkipZ} 
to avoid mode collapse.
A skip connection can be specified by the input dimension, $d_z$, the node feature dimension, $d_y$,  the multiplier for hidden dimension, $N_z$, and the maximal number of nodes in the output graph, $N$. 
The skip connection generates additional node features by 
$$
\textrm{LeakyReLU}\mleft(\textrm{BN}( \textrm{MLP}[d_z, N_zd_y, Nd_y](\vz))\mright).
$$

{\bf MPNN}. Our implemented models use an edge-conditional MPNN \citep{gilmer2017edge_mpnn, simonovsky2017edge_mpnn}, which can be characterized by input node feature dimension $d_x$, input edge feature dimension $d_w$ and output node feature dimension $d_y$. For a featured graph $(V, E, \mX, \mW)$ with $d_x$-dimensional node features and $d_w$-dimensional edge features, the rows of the output features matrix $\mY$ are formed by
$$
\vy_j = \textrm{LeakyReLU}\mleft(\textrm{BN}\left(\vx_{j} \bm{\Theta}_s + 
\sum_{i\in N_j}\vx_{i} H_n( \vw_{i, j})\right)\mright),
$$
where $N_j$ is a set of neighbors of $j$, $H_n$ is a linear layer, which maps an edge feature vector to a matrix with the same dimensions as $\bm{\Theta}_s$.

{\bf Final layer for edge features.} The final layer outputs feature vectors for each edge in the output graph, while leveraging the node features of the connecting nodes and the edge features from the last intermediate graph. The parameters of this layer include the dimension of input edge features, $d_w$, the dimension of input node features, $d_x$, and  the dimension of the output edge features, $d_u$. It produces the final edge features by
$$
\textrm{MLP-}uf(\vw_{i,j}, \vx_i, \vx_j) = \textrm{MLP}[d_x + 2d_w, d_u]\mleft(\vw_{i,j}, \textrm{MLP}[d_x, d_w, d_w](\textrm{agg}(\vx_i, \vx_j)), \textrm{agg}(\vx_i, \vx_j)\mright).
$$

\subsection{Details of the Unpooling Layer}\label{subsec:unpoolinglayers}
{\bf Hyperparameters}. Recall from \S\ref{sec:UL_method} that an unpooling layer takes a featured input graph $\gG=(V, E, \mX, \mW)$, with $d_x$-dimensional node features and $d_w$-dimensional edge features. This unpooling layer is thus specified using the following hyperparameters:
\begin{itemize}
    \item $I_s'$: a set of nodes fixed as static (i.e., they will not be unpooled).
    \item $I_r$: a set of nodes that are determined to be static based on the probabilities $p_j^s=\textrm{MLP-S}(\vx_j)$ for $j\in I_r$, where  $\textrm{MLP-S}=\textrm{sigmoid}(\textrm{MLP}[d_x, \lfloor d_x/2\rfloor, 1])$. The overall static nodes for the unpooling layer are 
    $$I_s = I_s' \cup \{j: j \in I_r, \, U_j < \textrm{MLP-S}(\vx_j)\},$$ 
    where $U_j$ are i.i.d.~random variables drawn from a uniform distribution on $U[0, 1]$.
    \item $k_{v}, d_y$: hidden dimension and output dimension of node features. We let $d'_x = \lfloor d_x/2\rfloor + \lfloor d_x/4\rfloor$ and $$\textrm{MLP-}y = \textrm{MLP}[d_x', k_v, d_y],$$ so that the child node feature vector is given by $\vy_{f_j(i)} = \textrm{MLP-}y ( P_{S_j}(\vx_i))$.
    The latter projection, $P_{S_j}$, is defined for $j=1$, 2, and each feature vector $\vx=(x_1, ...x_{d})$ of 
    a node in $I_u$ as
    \begin{align*}
        P_{S_1} (\vx) &= (x_1, ..., x_{d_s}, x_{d_{s} + 1}, ... x_{d_{s} + D})^T,\\
        P_{S_2} (\vx) &= (x_1, ..., x_{d_{s}}, x_{d_{s} + D + 1}, ... x_{d_{s} + 2D})^T,
    \end{align*}
    where  $d_s =  \lfloor d_x/2\rfloor$ and $D = \lfloor\frac{d_x}{4}\rfloor$.
    
    \item $k_{ia}$ and $k_{ie}$: hidden dimensions in $\textrm{MLP-IA}$ and $\textrm{MLP-IE}$, respectively. The probabilities used for generating intra- and inter-links are given by
    
    \begin{align}
    \textrm{MLP-IA}(\vx) =& \textrm{sigmoid}\mleft(\textrm{MLP}[d_y, k_{ia}, 1](\vx)\mright), \label{eqn:intra_prob}\\
        \textrm{MLP-IE}\mleft(\vy_1, \vy_2, \vw, \vx\mright) =& \textrm{softmax} \Big(\textrm{MLP-IE-1}(\vy_1, \vw, \vx),  \textrm{MLP-IE-1}(\vy_2, \vw, \vx), \notag\\
        &\textrm{MLP-IE-2}(\vy_1, \vy_2, \vw, \vx)\Big),\label{eqn:inter_prob}
    \end{align}
    where the first two $\textrm{MLP}$s used in $\textrm{MLP-IE}$ are the same networks, defined as $$\textrm{MLP-IE-1}:=\textrm{MLP}[d_y+d_w+d_x, k_{ie}, 1],$$ and the third term is defined as $$\textrm{MLP-IE-2}(\vy_1, \vy_2, \vw, \vx):=\textrm{MLP}[d_y+d_w+d_x, k_{ie}, 1](\textrm{agg}(\vy_1, \vy_2), \vw, \vx).$$ In practice $\textrm{MLP-C}$ in Step 3b in \S\ref{sec:UL_method} is based on $\textrm{MLP-IE}$ as follows. It uses $\textrm{MLP-IE-2}$ and calculates $$h_C(\vy_1, \vy_2, \vw, \vx):=\textrm{MLP-IE-2}(\vy_1, \vy_2, \vw, \vx).$$ The probabilities used in Step 3b in \S\ref{sec:UL_method} are then calculated as 
    
\begin{align*}
    &\mleft(p_b(j, i_1), p_b(j, i_2), ... p_b(j, i_{m_j})\mright) = \textrm{softmax}\big(h_C(\vy_{f_1(j)}, \vy_{f_2(j)}, \vw_{\{j, i_1\}}, \vx_{i_1}), \\
    &\quad \quad h_C(\vy_{f_1(j)}, \vy_{f_2(j)}, \vw_{\{j, i_2\}}, \vx_{i_2}),  ..., ~h_C(\vy_{f_1(j)}, \vy_{f_2(j)}, \vw_{\{j, i_{m_j}\}}, \vx_{i_{m_j}})\big).
\end{align*}
    
    \item $k_w, d_u$: hidden dimension and output dimension of the edge features. The output edge feature vectors  are built as follows: $$\vu_{i, j} = \textrm{MLP-}u(\vy_i, \vy_j) = \textrm{LeakyReLU}(\textrm{BN}(\textrm{MLP}[d_y, k_w, d_u](\textrm{agg}(\vy_i, \vy_j)))).$$ 
\end{itemize}

{\bf Preference score}. We further introduce preference score and modify the direct way of generating the probabilities obtained in \eqref{eqn:inter_prob}, in order to avoid all the edges being inherited by one single node in a pair of children nodes and thus to ensure that the unpooling layer can produce all possible output graphs with probabilities that are not too small. Denote $$h_s(\vy, \vw, \vx):= \textrm{MLP-IE-1}(\vy, \vw, \vx) \ \text{ and } \ h_b(\vy_1, \vy_2, \vw, \vx)):=\textrm{MLP-IE-2}(\vy, \vw, \vx),$$ where the forms of $\textrm{MLP-IE-1}$ and $\textrm{MLP-IE-2}$ are defined below \eqref{eqn:inter_prob}. For each child node $f_1(j)$ with feature $\vy$, which is generated from an input node $j$, we rescale the $h_s$'s to obtain the preference score $h^{(p)}_s$ when $N_{\{k, j\}}(j) = \{f_1(j)\}$ as follows:
\begin{align*}
    &\mleft(\{h_s^{(p)}(\vy, k): k\in N_j\}, h^{(p)}_{s,n}(\vy)\mright):=\textrm{softmax}\mleft(\{h_s(\vy, \vw_{j,k}, \vx_k): k\in N_j\}, h^{(0)}_s(\vy) \mright),
\end{align*}
where $N_j$ is set of neighbors of node $j$, and $h^{(0)}_s(\vy)$ is capturing zero-preference (i.e.,  preference of not connecting any inter-link) and is calculated by a multilayer perceptron $\textrm{MLP}[d_y, 2d_y, 1]$. 

Similarly, we also rescale the $h_b$'s to obtain the preference score when $N_{\{k, j\}}(j) = \{f_1(j), f_2(j)\}$. Let $\vy_1$ and $\vy_2$ be feature vectors of children nodes generated from $j$. Let $\vx$ be the feature vector of node $j$ in the input graph. With an added zero preference $h^{(0)}_b(\vx)=\textrm{MLP}[d_x, 2d_x, 1](\vx)$,  the preference score is calculated as
\begin{align*}
    &\left(\{h^{(p)}_b(\vy_1, \vy_2, k): k\in N_j\}, h^{(p)}_{b, n}(\vy_1, \vy_2)\right) := \textrm{softmax}\left(\{h_b(\vy_1, \vy_2, \vw_{j,k}, \vx_k): k\in N_j\}, h^{(0)}_b(\vx) \right).  
\end{align*}

We use the preference scores $h^{(p)}_s$ and $h^{(p)}_b$ instead of $\textrm{MLP-IE}$ to obtain probabilities for building $N_{\{k, j\}}(j)$ for the inter-link in Step 2c in \S\ref{sec:UL_method} as
$$
(p_1, p_2)=\mleft(\frac{(h^{(p)}_s(\vy_1, k)}{Z}, \frac{h^{(p)}_s(\vy_1, k)}{Z}\mright), \ \text{ where } \  Z = \left(h^{(p)}_s(\vy_1, k)+h^{(p)}_s(\vy_2, k)+h^{(p)}_b (\vy_1, \vy_2, k)\right).
$$

\subsection{Architectures of Discriminators and Encoders}\label{subsec:discriminators}

{\bf Discriminators of UL GAN.} The details of the discriminator in UL GAN and the encoder in UL VAE for QM9 have been described in \S4.3. In this part, we describe the details for other datasets. 

The discriminator in the Waxman random graph data
takes an input graph and uses two MPNN layers with $64$ and $128$ units to generate a graph $\gG=(V, E, \mX, \mW)$. It then aggregates the node features (the rows $\{\vx_j\}_{j \in V}$ of $\mX$) to produce the following single feature vector for the graph:
$$
\vh(\gG) = \sum_{j\in V}\sigma (\operatorname{lin}_1(\vx_j))  \odot
\tanh (\operatorname{lin}_2 (\vx_j));
$$
where $\sigma$ is logistic sigmoid, $\odot$ is element-wise multiplication and $\operatorname{lin}_1$ and $\operatorname{lin}_2$  are two layers with $128$ units.
Then it uses a global sum pooling to extract a signal for each graph. It applies two linear layers with $128$ and $256$ units. A final layer with a single unit then produces the output, where its activation function is $\textrm{sigmoid}$.
The LeakyReLU activation function (with 0.05 negative slope) is used after the two MPNNs and all linear layers. 

The discriminator in the protein dataset is similar to the discriminator in the random graph dataset. Some parameters were changed to simplify the network that fit the protein dataset; we used one MPNN layer with 128 units and one linear layer with 256 units after the global pooling.

The discriminator in ZINC is the same as the discriminator used in QM9 as described in \S4.3.

{\bf Encoder in UL VAE.} It has the same architecture as the corresponding discriminator described above, except that the final layer maps into a 256-dimensional vector with a linear activation function. This vector further splits to two 128-dimensional vectors: $\vz_\mu$ and $\vz_\sigma$. It then  generates the following latent vector: $\vz = \vz_\mu +  \exp(\frac{1}{2}\vz_\sigma)\odot \vx$, where $\vx\sim N(0, 1)$.

{\bf Additional adjustment when the number of nodes of the output graph is not fixed.} For the ZINC dataset, since the number of nodes of the output graph is not fixed by the generator, we added one trivial predictive network to the discriminator. This additional network contains one hidden layer with 256 units and only takes the global sum of node feature vectors and edge feature vectors as input, to encourage a proper distribution of the numbers of atoms and edges. The final output of the discriminator is the sum of the outputs of this network and the original GNN discriminator.

\subsection{Architecture of Generators}
\label{subsec:generators}

We describe details of the architecture of the generator in UL GAN and decoder in UL VAE for the different datasets.

{\bf Architecture of the generator in UL GAN}. 
We show the generators in UL GAN (same as the decoders in UL VAE, if applicable) 
for the Waxman random graph, protein, QM9 and ZINC datasets in Figure~\ref{fig:plot_generator_wax}, Figure~\ref{fig:plot_generator_protein},
Figure~\ref{fig:plot_generator_qm9}, and
Figure \ref{fig:plot_generator_zinc}, respectively. 
For simplicity, we neglect batch sizes. For example, the input should be $\vz\in \sR^{B\times 128}$, where $B$ is the batch size.

{\bf Loss function}. In practice, we alternatively train the generator with a GAN's loss function and with REINFORCE. In each training step, we first update the generator using the loss function $-D(G(z)).$ 
We then update the parameter of the generator, $\theta_G$,   using
$$
    \vtheta^{(t+1)}_G := \vtheta^{(t)}_G + \alpha \mleft(\nabla_\vtheta \textrm{logP}|_{\vtheta_G^{(t)}}\mright) \mleft(D(G(z)) - \mathbb{E}(D(G(z)))\mright),
$$
where the learning rate $\alpha = 5\times 10^{-2}$.

\subsection{Reconstruction Loss for UL VAE}\label{subsec:reconstruction_VAE}
We calculate the reconstruction loss for UL VAE following \citet{simonovsky2018graphvae}. 

Given two featured graphs with the same number of nodes, we calculate their distance by summing the distances between node features, the distances between adjacency matrices and the distances between edge features of these two graphs. We used the same formula described in Section 3.3 in \citet{simonovsky2018graphvae}.

When the number of nodes may vary (e.g., in the QM9 dataset, the corresponding graphs may contain 6--9 nodes) and the input/output graph has less nodes than the maximal number of nodes (e.g., less than 9 nodes in the QM9 experiment), we add  artificial nodes to the graph so that it contains the maximal number of nodes (e.g., 9 nodes for QM9). We add an additional binary feature to the node features: it assigns $0$ to the real nodes and $1$ to the artificial nodes.

For an input graph and an output graph with the same numbers of nodes (if they have different node numbers, we perform the above procedure to make them the same), we permute the node indices of the output graph in order to achieve a minimal distance to the input graph. This minimal distance is regarded as the reconstruction loss between the input and output graph. More details of how to permute the indices appear in Section 3.4 of \citet{simonovsky2018graphvae}.

\begin{figure}
\begin{tikzpicture}[auto]
\begin{scope}
  \coordinate (gap) at (0,-1);
  \coordinate (gap2) at (0,-1.3);
  \coordinate (z1_xsh) at (7, 0);
  \coordinate (z2_xsh) at (7, 0);
    \node (a1) [draw, rounded rectangle]  at (0, 0) {Input $\vz\in \sR^{128}$};
    \node (a2) [draw, rectangle]  at ($(a1) + (gap)$) {\textbf{Initial layer}: $d_{in} = 128, d_x= 128, d_w=8$};
    \node (a3) [draw, rounded rectangle]  at ($(a2) + (gap)$) {Graph $(V, E, \mX, \mW)$, \\$|V| = 3, |E|\in [2, 3], \mX\in \sR^{3\times 128}, \sW\in \sR^{M\times 8}$};
    \node (a4) [draw, rectangle]  at ($(a3) + (gap)$) {\textbf{MPNN layer}: $d_x=128, d_w=8, d_y=64$};
    \node (a5) [draw, rectangle]  at ($(a4) + (gap2)$) {\makecell[l]{\textbf{Unpooling layer}: $I'_s = \emptyset, I_r = \{1, 2\}, d_x=64$, \\$k_v = k_{ia} = k_{ie} = 64, d_w = d_u = 8, d_y = 48$}};
    \node (b5) [draw, rectangle]  at ($(a5) + (z1_xsh)$) {\makecell[l]{\textbf{Skip connection}: \\$d_z=128, N_z=10, d_y=16$}};

    \node (a6) [draw, rounded rectangle]  at ($(a5) + (gap2)$) {Graph $(V, E, \mX, \mW)$, $|V| \in [4, 6], \mX\in \sR^{N\times 48}, \sW\in \sR^{M\times 8}$};
    \node (b6) [draw, rounded rectangle]  at ($(a6) + (z1_xsh)$) {$\mX \in \sR^{6\times 16}$};
    \node (a7) [draw, rounded rectangle]  at ($(a6) + (gap)$) {Graph $(V, E, \mX, \mW)$, $|V| \in [4, 6], \mX\in \sR^{N\times 64}, \sW\in \sR^{M\times 8}$};
    \node (a8) [draw, rectangle]  at ($(a7) + (gap)$) {\textbf{MPNN layer}: $d_x=64, d_w=8, d_y=128$};
    \node (a9) [draw, rectangle]  at ($(a8) + (gap2)$) {\makecell[l]{\textbf{Unpooling layer}: $I'_s = \emptyset, I_r = \{1, 2, 3, 4, 5\}, d_x=128$, \\$k_v = k_{ia} = k_{ie} = 128, d_w = d_u = 8, d_y = 384$}};
    \node (b9) [draw, rectangle]  at ($(a9) + (z2_xsh)$) {\makecell[l]{\textbf{Skip connection}: \\$d_z=128, N_z=15, d_y=128$}};
    \node (a10) [draw, rounded rectangle]  at ($(a9) + (gap2)$) {Graph $(V, E, \mX, \mW)$, $|V| \in [5, 12], \mX\in \sR^{N\times 384}, \sW\in \sR^{M\times 8}$};
    \node (b10) [draw, rounded rectangle]  at ($(a10) + (z2_xsh)$) {$\mX \in \sR^{12\times 128}$};
    \node (a11) [draw, rounded rectangle]  at ($(a10) + (gap)$) {Graph $(V, E, \mX, \mW)$, $|V| \in [5, 12], \mX\in \sR^{N\times 512}, \sW\in \sR^{M\times 8}$};
    \node (a12) [draw, rectangle]  at ($(a11) + (gap)$) {\textbf{MPNN layer}: $d_x=512, d_w=8, d_y=64$};
    \node (a13) [draw, rectangle]  at ($(a12) + (gap)$) {\textbf{MLP}[64, 2]};
    \node (a16) [draw, rounded rectangle]  at ($(a13) + (gap)$) {Output graph $(V, E, \mX)$: $|V|\in [5, 12], \mX\in\sR^{N\times 2}$};
      \draw [->](a1) -- (a2);
      \draw [->](a2) -- (a3);
      \draw [->](a3) -- (a4);
      \draw [->](a4) -- (a5);
      \draw [->](a1.east) to [out=0, in=90] (b5.north);
      \draw [->](a5) -- (a6);
      \draw [->](a6) -- (a7);
      \draw [->](a7) -- (a8);
      \draw [->](a8) -- (a9);
      \draw [->](a9) -- (a10);
      \draw [->](a10) -- (a11);
      \draw [->](a11) -- (a12);
      \draw [->](a12) -- (a13);
      \draw [->](a13) -- (a16);
      \draw [->](b5) -- (b6);
      \draw [->](b6.south) to [out=270, in=0]node{Concatenate} (a7.east);
      \draw [->](b9) -- (b10);
      \draw [->](b10.south) to [out=270, in=0]node{Concatenate} (a11.east);
      \draw [->](a1.east) -- ($(a1) + (5, 0)$) to [out=0, in=0] (b9.east);    
      \end{scope}
\end{tikzpicture}
\caption{Demonstration of the architecture of the generator in UL GAN and the decoder in UL VAE for Waxman random graph dataset. The rectangles present neural network layers and the rounded rectangles clarify the shape of the hidden data. } \label{fig:plot_generator_wax}
\end{figure}
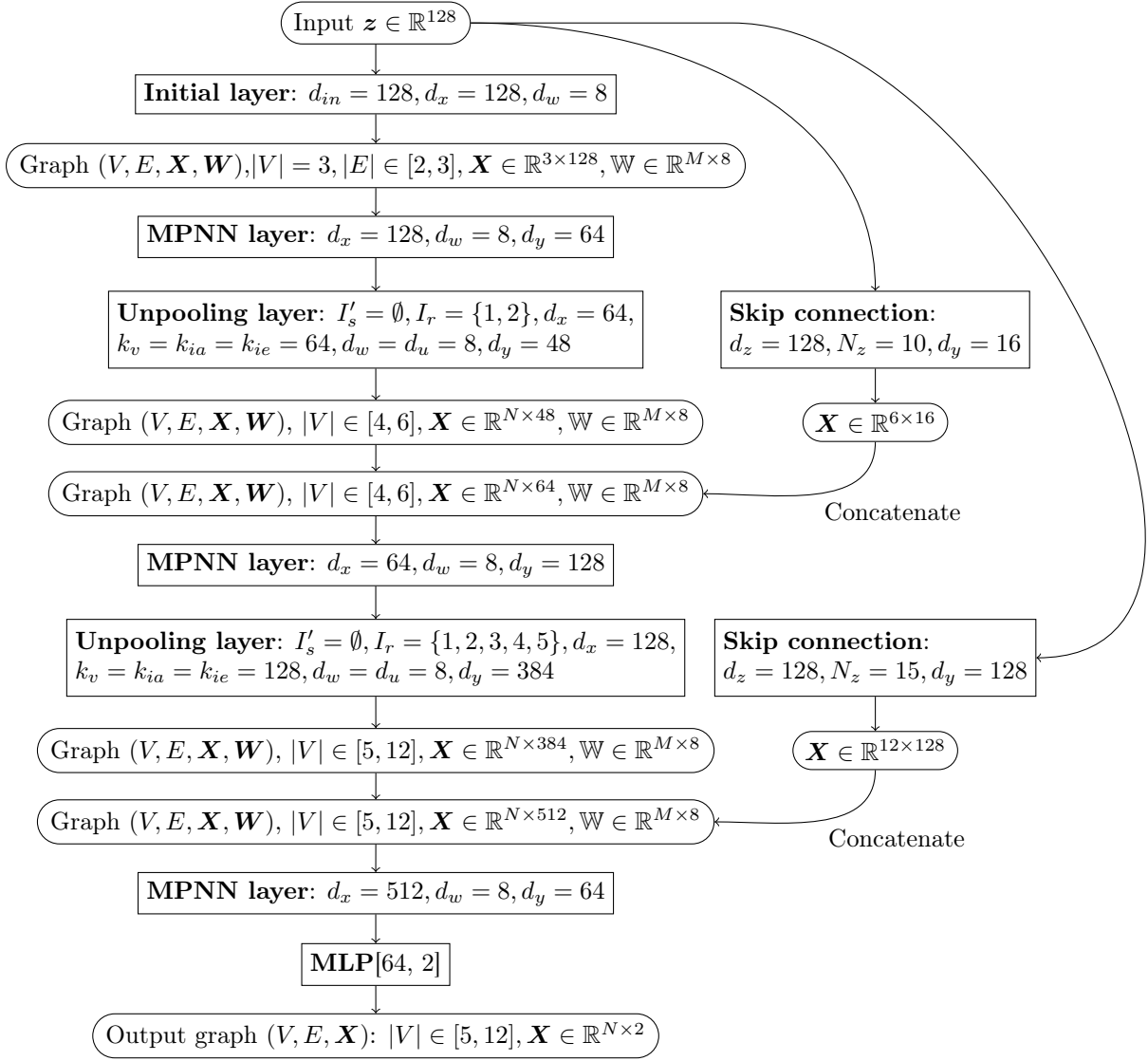

\begin{figure}
\begin{tikzpicture}[auto]
\begin{scope}
  \coordinate (gap) at (0,-1);
  \coordinate (gap2) at (0,-1.3);
  \coordinate (z1_xsh) at (7, 0);
  \coordinate (z2_xsh) at (7, 0);
    \node (a1) [draw, rounded rectangle]  at (0, 0) {Input $\vz\in \sR^{128}$};
    \node (a2) [draw, rectangle]  at ($(a1) + (gap)$) {\textbf{Initial layer}: $d_{in} = 128, d_x= 128, d_w=8$};
    \node (a3) [draw, rounded rectangle]  at ($(a2) + (gap)$) {Graph $(V, E, \mX, \mW)$, \\$|V| = 3, |E|\in [2, 3], \mX\in \sR^{3\times 128}, \sW\in \sR^{M\times 8}$};
    \node (a4) [draw, rectangle]  at ($(a3) + (gap)$) {\textbf{MPNN layer}: $d_x=128, d_w=8, d_y=128$};
    \node (a5) [draw, rectangle]  at ($(a4) + (gap2)$) {\makecell[l]{\textbf{Unpooling layer}: $I'_s = \{1, 2\}, I_r = \emptyset, d_x=128$, \\$k_v = k_{ia} = k_{ie} = 128, d_w = d_u = 8, d_y = 96$}};
    \node (b5) [draw, rectangle]  at ($(a5) + (z1_xsh)$) {\makecell[l]{\textbf{Skip connection}: \\$d_z=128, N_z=10, d_y=32$}};
    \node (a6) [draw, rounded rectangle]  at ($(a5) + (gap2)$) {Graph $(V, E, \mX, \mW)$, $|V| = 4, \mX\in \sR^{4\times 96}, \sW\in \sR^{M\times 8}$};
    \node (b6) [draw, rounded rectangle]  at ($(a6) + (z1_xsh)$) {$\mX \in \sR^{4\times 32}$};
    \node (a7) [draw, rounded rectangle]  at ($(a6) + (gap)$) {Graph $(V, E, \mX, \mW)$, $|V| = 4, \mX\in \sR^{N\times 128}, \sW\in \sR^{M\times 8}$};
    \node (a8) [draw, rectangle]  at ($(a7) + (gap)$) {\textbf{MPNN layer}: $d_x=128, d_w=8, d_y=128$};
    \node (a9) [draw, rectangle]  at ($(a8) + (gap2)$) {\makecell[l]{\textbf{Unpooling layer}: $I'_s = \emptyset, I_r = \emptyset, d_x=128$, \\$k_v = k_{ia} = k_{ie} = 128, d_w = d_u = 8, d_y = 384$}};
    \node (b9) [draw, rectangle]  at ($(a9) + (z2_xsh)$) {\makecell[l]{\textbf{Skip connection}: \\$d_z=128, N_z=15, d_y=128$}};
    \node (a10) [draw, rounded rectangle]  at ($(a9) + (gap2)$) {Graph $(V, E, \mX, \mW)$, $|V| = 8, \mX\in \sR^{8\times 384}, \sW\in \sR^{M\times 8}$};
    \node (b10) [draw, rounded rectangle]  at ($(a10) + (z2_xsh)$) {$\mX \in \sR^{8\times 128}$};
    \node (a11) [draw, rounded rectangle]  at ($(a10) + (gap)$) {Graph $(V, E, \mX, \mW)$, $|V| = 8, \mX\in \sR^{8\times 512}, \sW\in \sR^{M\times 8}$};
    \node (a12) [draw, rectangle]  at ($(a11) + (gap)$) {\textbf{MPNN layer}: $d_x=512, d_w=8, d_y=256$};
    \node (a13) [draw, rectangle]  at ($(a12) + (gap)$) {\textbf{MLP}[256, 3]};
    \node (a16) [draw, rounded rectangle]  at ($(a13) + (gap)$) {Output graph $(V, E, \mX)$: $|V| = 8, \mX\in\sR^{8\times 3}$};
      \draw [->](a1) -- (a2);
      \draw [->](a2) -- (a3);
      \draw [->](a3) -- (a4);
      \draw [->](a4) -- (a5);
      \draw [->](a1.east) to [out=0, in=90] (b5.north);
      \draw [->](a5) -- (a6);
      \draw [->](a6) -- (a7);
      \draw [->](a7) -- (a8);
      \draw [->](a8) -- (a9);
      \draw [->](a9) -- (a10);
      \draw [->](a10) -- (a11);
      \draw [->](a11) -- (a12);
      \draw [->](a12) -- (a13);
      \draw [->](a13) -- (a16);
      \draw [->](b5) -- (b6);
      \draw [->](b6.south) to [out=270, in=0]node{Concatenate} (a7.east);
      \draw [->](b9) -- (b10);
      \draw [->](b10.south) to [out=270, in=0]node{Concatenate} (a11.east);
      \draw [->](a1.east) -- ($(a1) + (5, 0)$) to [out=0, in=0] (b9.east);    
      \end{scope}
\end{tikzpicture}
\caption{Demonstration of the architecture of the generator in UL GAN and the decoder in UL VAE for protein dataset. The rectangles present neural network layers and the rounded rectangles clarify the shape of the hidden data. } \label{fig:plot_generator_protein}
\end{figure}

\begin{figure}
\begin{tikzpicture}[auto]
\begin{scope}
  \coordinate (gap) at (0,-1);
  \coordinate (gap2) at (0,-1.3);
  \coordinate (z1_xsh) at (7, 0);
  \coordinate (z2_xsh) at (7, 0);
    \node (a1) [draw, rounded rectangle]  at (0, 0) {Input $\vz\in \sR^{128}$};
    \node (a2) [draw, rectangle]  at ($(a1) + (gap)$) {\textbf{Initial layer}: $d_{in} = 128, d_x= 256, d_w=32$};
    \node (a3) [draw, rounded rectangle]  at ($(a2) + (gap)$) {Graph $(V, E, \mX, \mW)$, \\$|V| = 3, |E|\in [2, 3], \mX\in \sR^{3\times 256}, \sW\in \sR^{M\times 32}$};
    \node (a4) [draw, rectangle]  at ($(a3) + (gap)$) {\textbf{MPNN layer}: $d_x=256, d_w=32, d_y=128$};
    \node (a5) [draw, rectangle]  at ($(a4) + (gap2)$) {\makecell[l]{\textbf{Unpooling layer}: $I'_s = \{1\}, I_r = \emptyset, d_x=128$, \\$k_v = k_{ia} = k_{ie} = 128, d_w = d_u = 32, d_y = 96$}};
    \node (b5) [draw, rectangle]  at ($(a5) + (z1_xsh)$) {\makecell[l]{\textbf{Skip connection}: \\$d_z=128, N_z=10, d_y=32$}};

    \node (a6) [draw, rounded rectangle]  at ($(a5) + (gap2)$) {Graph $(V, E, \mX, \mW)$, $|V| = 5, \mX\in \sR^{5\times 96}, \sW\in \sR^{M\times 32}$};
    \node (b6) [draw, rounded rectangle]  at ($(a6) + (z1_xsh)$) {$\mX \in \sR^{5\times 32}$};
    \node (a7) [draw, rounded rectangle]  at ($(a6) + (gap)$) {Graph $(V, E, \mX, \mW)$, $|V| = 5, \mX\in \sR^{5\times 128}, \sW\in \sR^{M\times 32}$};
    \node (a8) [draw, rectangle]  at ($(a7) + (gap)$) {\textbf{MPNN layer}: $d_x=128, d_w=32, d_y=128$};
    \node (a9) [draw, rectangle]  at ($(a8) + (gap2)$) {\makecell[l]{\textbf{Unpooling layer}: $I'_s = \{1\}, I_r = \{2, 3, 4\}, d_x=128$, \\$k_v = k_{ia} = k_{ie} = 128, d_w = d_u = 32, d_y = 96$}};
    \node (b9) [draw, rectangle]  at ($(a9) + (z2_xsh)$) {\makecell[l]{\textbf{Skip connection}: \\$d_z=128, N_z=15, d_y=32$}};
    \node (a10) [draw, rounded rectangle]  at ($(a9) + (gap2)$) {Graph $(V, E, \mX, \mW)$, $|V| \in [6, 9], \mX\in \sR^{N\times 96}, \sW\in \sR^{M\times 32}$};
    \node (b10) [draw, rounded rectangle]  at ($(a10) + (z2_xsh)$) {$\mX \in \sR^{9\times 32}$};
    \node (a11) [draw, rounded rectangle]  at ($(a10) + (gap)$) {Graph $(V, E, \mX, \mW)$, $|V| \in [6, 9], \mX\in \sR^{N\times 128}, \sW\in \sR^{M\times 32}$};
    \node (a12) [draw, rectangle]  at ($(a11) + (gap)$) {\textbf{MPNN layer}: $d_x=128, d_w=32, d_y=64$};
    \node (a13) [draw, rectangle]  at ($(a12) + (gap) - (3, 0)$) {\textbf{MLP}[64, 64, 10]};
    \node (a14) [draw, rectangle]  at ($(a13) + (z1_xsh)$) {\textbf{Final edge layer}: $d_x=64, d_w=32, d_u=3$};
    \node (a15) [draw, rectangle]  at ($(a12) + 2*(gap)$) {\textbf{Gumbel softmax layers}};
    \node (a16) [draw, rounded rectangle]  at ($(a15) + (gap)$) {Output graph $(V, E, \mX, \mW)$: $|V|\in [6, 9], \mX\in\sR^{N\times 10}, \mW\in\sR^{M\times 3}$};
    
      \draw [->](a1) -- (a2);
      \draw [->](a2) -- (a3);
      \draw [->](a3) -- (a4);
      \draw [->](a4) -- (a5);
      \draw [->](a1.east) to [out=0, in=90] (b5.north);
      \draw [->](a5) -- (a6);
      \draw [->](a6) -- (a7);
      \draw [->](a7) -- (a8);
      \draw [->](a8) -- (a9);
      \draw [->](a9) -- (a10);
      \draw [->](a10) -- (a11);
      \draw [->](a11) -- (a12);
      \draw [->](a12) -- (a14);
      \draw [->](a12) -- (a13);
      \draw [->](a13) -- (a15);
      \draw [->](a14) -- (a15);
      \draw [->](a15) -- (a16);
      \draw [->](b5) -- (b6);
      \draw [->](b6.south) to [out=270, in=0]node{Concatenate} (a7.east);
      \draw [->](b9) -- (b10);
      \draw [->](b10.south) to [out=270, in=0]node{Concatenate} (a11.east);
      \draw [->](a1.east) -- ($(a1) + (5, 0)$) to [out=0, in=0] (b9.east);    
      \end{scope}
\end{tikzpicture}
\caption{Demonstration of the architecture of the generator in UL GAN and the decoder in UL VAE for QM9. The rectangles present neural network layers and the rounded rectangles clarify the shape of the hidden data. } \label{fig:plot_generator_qm9}
\end{figure}

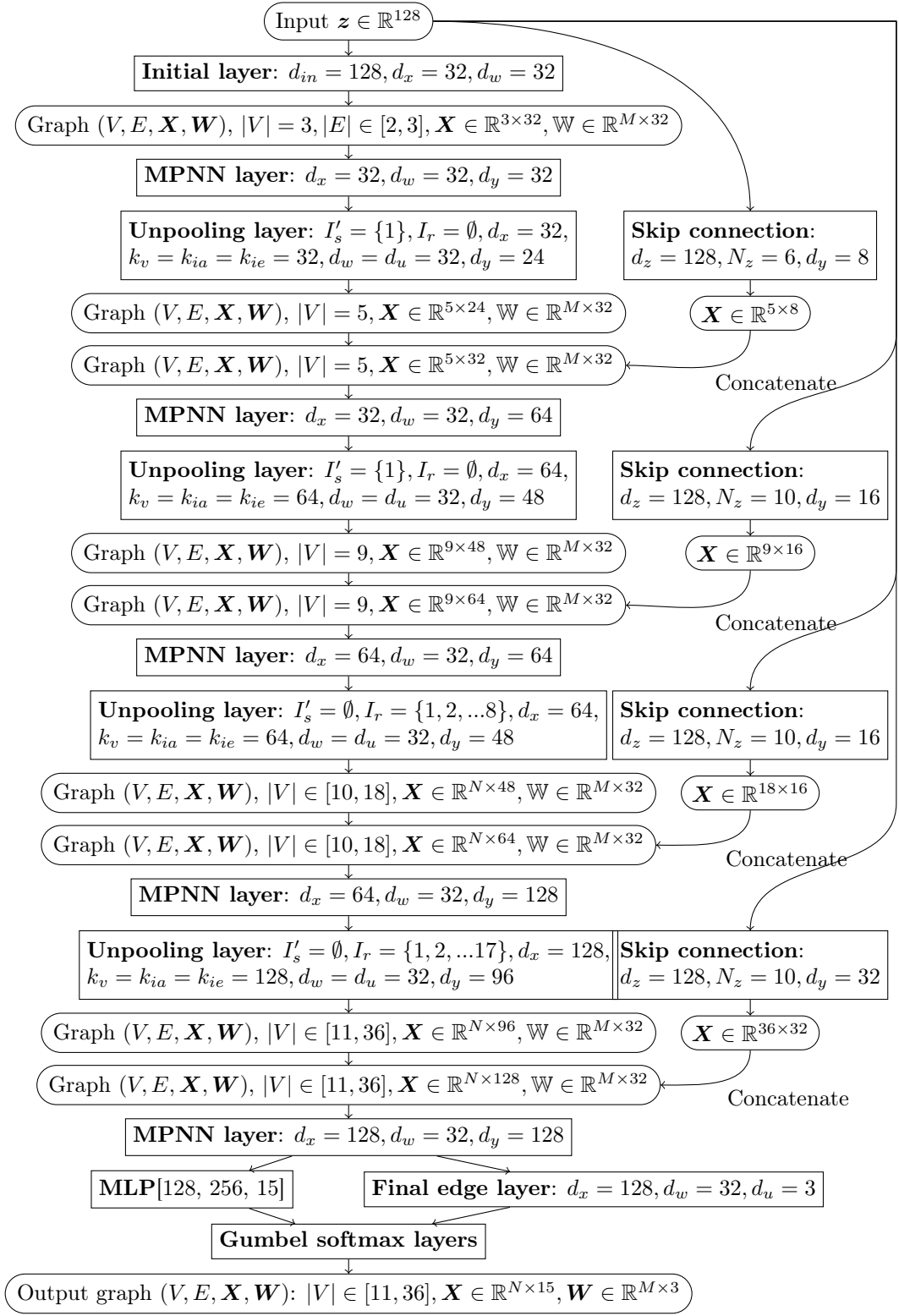
\begin{figure}
\begin{center}
    \begin{tikzpicture}[auto]
\begin{scope}[scale=0.82]
  \coordinate (gap) at (0,-1);
  \coordinate (gap2) at (0,-1.3);
  \coordinate (z1_xsh) at (7.7, 0);
  \coordinate (z2_xsh) at (7.7, 0);
    \node (a1) [draw, rounded rectangle]  at (0, 0) {Input $\vz\in \sR^{128}$};
    \node (a2) [draw, rectangle]  at ($(a1) + (gap)$) {\textbf{Initial layer}: $d_{in} = 128, d_x= 32, d_w=32$};
    \node (a3) [draw, rounded rectangle]  at ($(a2) + (gap)$) {Graph $(V, E, \mX, \mW)$,  $|V| = 3, |E|\in [2, 3], \mX\in \sR^{3\times 32}, \sW\in \sR^{M\times 32}$};
    \node (a4) [draw, rectangle]  at ($(a3) + (gap)$) {\textbf{MPNN layer}: $d_x=32, d_w=32, d_y=32$};
    \node (a5) [draw, rectangle]  at ($(a4) + (gap2)$) {\makecell[l]{\textbf{Unpooling layer}: $I'_s = \{1\}, I_r = \emptyset, d_x=32$, \\$k_v = k_{ia} = k_{ie} = 32, d_w = d_u = 32, d_y = 24$}};
    \node (b5) [draw, rectangle]  at ($(a5) + (z1_xsh)$) {\makecell[l]{\textbf{Skip connection}: \\$d_z=128, N_z=6, d_y=8$}};

    \node (a6) [draw, rounded rectangle]  at ($(a5) + (gap2)$) {Graph $(V, E, \mX, \mW)$, $|V| = 5, \mX\in \sR^{5\times 24}, \sW\in \sR^{M\times 32}$};
    \node (b6) [draw, rounded rectangle]  at ($(a6) + (z1_xsh)$) {$\mX \in \sR^{5\times 8}$};
    \node (a7) [draw, rounded rectangle]  at ($(a6) + (gap)$) {Graph $(V, E, \mX, \mW)$, $|V| = 5, \mX\in \sR^{5\times 32}, \sW\in \sR^{M\times 32}$};
    \node (a8) [draw, rectangle]  at ($(a7) + (gap)$) {\textbf{MPNN layer}: $d_x=32, d_w=32, d_y=64$};
    \node (a9) [draw, rectangle]  at ($(a8) + (gap2)$) {\makecell[l]{\textbf{Unpooling layer}: $I'_s = \{1\}, I_r = \emptyset, d_x=64$, \\$k_v = k_{ia} = k_{ie} = 64, d_w = d_u = 32, d_y = 48$}};
    \node (b9) [draw, rectangle]  at ($(a9) + (z2_xsh)$) {\makecell[l]{\textbf{Skip connection}: \\$d_z=128, N_z=10, d_y=16$}};
    \node (a10) [draw, rounded rectangle]  at ($(a9) + (gap2)$) {Graph $(V, E, \mX, \mW)$, $|V| = 9, \mX\in \sR^{9\times 48}, \sW\in \sR^{M\times 32}$};
    \node (b10) [draw, rounded rectangle]  at ($(a10) + (z2_xsh)$) {$\mX \in \sR^{9\times 16}$};
    \node (a11) [draw, rounded rectangle]  at ($(a10) + (gap)$) {Graph $(V, E, \mX, \mW)$, $|V| = 9, \mX\in \sR^{9\times 64}, \sW\in \sR^{M\times 32}$};
    \node (a12) [draw, rectangle]  at ($(a11) + (gap)$) {\textbf{MPNN layer}: $d_x=64, d_w=32, d_y=64$};
    \node (a13) [draw, rectangle]  at ($(a12) + (gap2)$) {\makecell[l]{\textbf{Unpooling layer}: $I'_s = \emptyset, I_r = \{1, 2, ... 8\}, d_x=64$, \\$k_v = k_{ia} = k_{ie} = 64, d_w = d_u = 32, d_y = 48$}};
    \node (b13) [draw, rectangle]  at ($(a13) + (z2_xsh)$) {\makecell[l]{\textbf{Skip connection}: \\$d_z=128, N_z=10, d_y=16$}};
    \node (a14) [draw, rounded rectangle]  at ($(a13) + (gap2)$) {Graph $(V, E, \mX, \mW)$, $|V| \in [10, 18], \mX\in \sR^{N\times 48}, \sW\in \sR^{M\times 32}$};
    \node (b14) [draw, rounded rectangle]  at ($(a14) + (z2_xsh)$) {$\mX \in \sR^{18\times 16}$};
    \node (a15) [draw, rounded rectangle]  at ($(a14) + (gap)$) {Graph $(V, E, \mX, \mW)$, $|V| \in [10, 18], \mX\in \sR^{N\times 64}, \sW\in \sR^{M\times 32}$};
    \node (a16) [draw, rectangle]  at ($(a15) + (gap)$) {\textbf{MPNN layer}: $d_x=64, d_w=32, d_y=128$};
    \node (a17) [draw, rectangle]  at ($(a16) + (gap2)$) {\makecell[l]{\textbf{Unpooling layer}: $I'_s = \emptyset, I_r = \{1, 2, ... 17\}, d_x=128$, \\$k_v = k_{ia} = k_{ie} = 128, d_w = d_u = 32, d_y = 96$}};
    \node (b17) [draw, rectangle]  at ($(a17) + (z2_xsh)$) {\makecell[l]{\textbf{Skip connection}: \\$d_z=128, N_z=10, d_y=32$}};
    \node (a18) [draw, rounded rectangle]  at ($(a17) + (gap2)$) {Graph $(V, E, \mX, \mW)$, $|V| \in [11, 36], \mX\in \sR^{N\times 96}, \sW\in \sR^{M\times 32}$};
    \node (b18) [draw, rounded rectangle]  at ($(a18) + (z2_xsh)$) {$\mX \in \sR^{36\times 32}$};
    \node (a19) [draw, rounded rectangle]  at ($(a18) + (gap)$) {Graph $(V, E, \mX, \mW)$, $|V| \in [11, 36], \mX\in \sR^{N\times 128}, \sW\in \sR^{M\times 32}$};
    \node (a20) [draw, rectangle]  at ($(a19) + (gap)$) {\textbf{MPNN layer}: $d_x=128, d_w=32, d_y=128$};
    \node (a21) [draw, rectangle]  at ($(a20) + (gap) - (3, 0)$) {\textbf{MLP}[128, 256, 15]};
    \node (a22) [draw, rectangle]  at ($(a21) + (z1_xsh)$) {\textbf{Final edge layer}: $d_x=128, d_w=32, d_u=3$};
    \node (a23) [draw, rectangle]  at ($(a20) + 2*(gap)$) {\textbf{Gumbel softmax layers}};
    \node (a24) [draw, rounded rectangle]  at ($(a23) + (gap)$) {Output graph $(V, E, \mX, \mW)$: $|V|\in [11, 36], \mX\in\sR^{N\times 15}, \mW\in\sR^{M\times 3}$};
    
      \draw [->](a1) -- (a2);
      \draw [->](a2) -- (a3);
      \draw [->](a3) -- (a4);
      \draw [->](a4) -- (a5);
      \draw [->](a1.east) to [out=0, in=90] (b5.north);
      \draw [->](a5) -- (a6);
      \draw [->](a6) -- (a7);
      \draw [->](a7) -- (a8);
      \draw [->](a8) -- (a9);
      \draw [->](a9) -- (a10);
      \draw [->](a10) -- (a11);
      \draw [->](a11) -- (a12);
      \draw [->](a12) -- (a13);
      \draw [->](a13) -- (a14);
      \draw [->](a14) -- (a15);
      \draw [->](a15) -- (a16);
      \draw [->](a16) -- (a17);
      \draw [->](a17) -- (a18);
      \draw [->](a18) -- (a19);
      \draw [->](a19) -- (a20);
      \draw [->](a20) -- (a21);
      \draw [->](a20) -- (a22);
      \draw [->](a21) -- (a23);
      \draw [->](a22) -- (a23);
      \draw [->](a23) -- (a24);
      \draw [->](b5) -- (b6);
      \draw [->](b6.south) to [out=270, in=0]node{Concatenate} (a7.east);

    \draw [->](b9) -- (b10);
      \draw [->](b10.south) to [out=270, in=0]node{Concatenate} (a11.east);
      \draw [->](a1.east) -- ($(a1) + (10.5, 0)$)
      -- ($(a1) + (10.5, -6)$)to [out=270, in=90] (b9.north);   
      
    \draw [->](b13) -- (b14);
      \draw [->](b14.south) to [out=270, in=0]node{Concatenate} (a15.east);
      \draw [->](a1.east) -- ($(a1) + (10.5, 0)$)
      -- ($(a1) + (10.5, -10.5)$)to [out=270, in=90] (b13.north);          

    \draw [->](b17) -- (b18);
      \draw [->](b18.south) to [out=270, in=0]node{Concatenate} (a19.east);
      \draw [->](a1.east) -- ($(a1) + (10.5, 0)$)
      -- ($(a1) + (10.5, -15)$)to [out=270, in=90] (b17.north);          
\end{scope}

\end{tikzpicture}
\end{center}
\caption{Demonstration of the architecture of the generator in UL GAN for ZINC. The rectangles present neural network layers and the rounded rectangles clarify the shape of the hidden data. } \label{fig:plot_generator_zinc}
\end{figure}

\section{PROOF OF THEOREMS}\label{sec:theorem}

In this section, we prove claims on connectivity and expressivity for the output graph of our unpooling layer. For the ease of presenting, we  denote a graph by $\gG=(V,E)$ and omit features in nodes and edges because features are irrelevant with graph connectivity and expressivity.

\subsection{Proof of Proposition~1}

Let  $k^o$, $l^o \in V^o$ be two arbitrary nodes in the output graph $\gG^o = (V^o, E^o)$. In order to prove the connectivity of $\gG^o$ we need to find a path in $\gG^o$ connecting $k^o$ and $l^o$.  Recall that $\gG=(V, E)$ denotes the input graph.
Let $k$, $l \in V$
be the parent nodes of $k^o$ and $l^o$, respectively. 
Since the input graph $\gG$ is connected,  $k$ and $l$ are connected by a path. We denote its length by $n-1$, where $n \geq 2$, and its edges by $\{i_1, i_2\}$, $\{i_2, i_3\}$, ...$\{i_{n-1}, i_n\}$, where $i_1 = k$ and $i_n = l$. Recall that for $r \in [n]$ $N_{\{i_r, i_{r+1}\}}(i_r)$
is a subset of the children of node $i_r$ connected to $N_{\{i_r, i_{r+1}\}}(i_{r+1})$, which is a subset of the children of node $i_{r+1}$. For $r \in [n]$ we arbitrarily choose a vertex in $N_{\{i_r, i_{r+1}\}}(i_r)$ and denote it by $v_{\{i_r, i_{r+1}\}}(i_r)$. We thus note that the following edges exist in the output graph:
$$
\mleft\{v_{\{i_1, i_{2}\}}(i_1), v_{\{i_1, i_{2}\}}(i_{2})\mright\}, \mleft\{v_{\{i_2, i_{3}\}}(i_2), v_{\{i_2, i_{3}\}}(i_{3})\mright\}, \cdots
\mleft\{v_{\{i_{n-1}, i_{n}\}}(i_{n-1}), v_{\{i_{n-1}, i_{n}\}}(i_{n})\mright\}.
$$
In order to prove that $k^o$ and $l^o$ connect by a path we verify the following properties:
\begin{enumerate}
    \item For all $r \in [n-2]$, either $v_{\{i_r, i_{r+1}\}}(i_{r+1}) = v_{\{i_{r+1}, i_{r+2}\}}(i_{r+1})$ or there is a path connecting $v_{\{i_r, i_{r+1}\}}(i_{r+1})$ and $v_{\{i_{r+1}, i_{r+2}\}}(i_{r+1})$ 
    \item Either $k^o = v_{\{i_1, i_{2}\}}(i_1)$ or there is a path connecting $k^o$ and $v_{\{i_1, i_{2}\}}(i_1)$
    \item Either $l^o = v_{\{i_{n-1}, i_{n}\}}(i_{n})$ or there is a path connecting $l^o$ and $v_{\{i_{n-1}, i_{n}\}}(i_{n})$
\end{enumerate}
To prove the first property we note that $v_{\{i_r, i_{r+1}\}}(i_{r+1})\in \{f_1(i_{r+1}), f_2(i_{r+1})\}$ and $v_{\{i_{r+1}, i_{r+2}\}}(i_{r+1})\in \{f_1(i_{r+1}), f_2(i_{r+1})\}$. If $v_{\{i_r, i_{r+1}\}}(i_{r+1}) \neq v_{\{i_{r+2}, i_{r+1}\}}(i_{r+1})$, then 
$\{v_{\{i_r, i_{r+1}\}}(i_{r+1}), v_{\{i_{r+1}, i_{r+2}\}}(i_{r+1})\} = \{f_1(i_{r+1}), f_2(i_{r+1})\}$. 
To show that there is a path connecting $v_{\{i_r, i_{r+1}\}}(i_{r+1})$ and $v_{\{i_{r+1}, i_{r+2}\}}(i_{r+1})$, it suffices to show that there is a path connecting $f_1(i_{r+1})$ and $f_2(i_{r+1})$.
Assume an arbitrary node $j\in V$. If $j\in V_c$ the construction in Step 3a guarantees the existence of an edge connecting $f_1(j)$ and $f_2(j)$. Otherwise, based on Step 3b, there exists a node $b_j$ such that $N_{\{b_j , j\}}(j) = \{f_1(j), f_2(j)\}$, so that $f_1(j)$ and $f_2(j)$ are connected via $f(b_j)$. Letting $j = i_{r+1}$ we conclude this property.

The above argument also applies to the other two properties.
$\qed$

\subsection{Proof of Theorem~2}

We will first define a pooling process and show that there exists an unpooling layer that acts as an inverse of this pooling procedure (see Lemma~\ref{lemma:pool_unpool}). We then show that any graph with $N$ nodes can be pooled to a graph with $\lceil N/2\rceil$ nodes (Lemma~\ref{lemma:even_graph} clarifies the case where $N$ is even and Lemma~\ref{lemma:odd_graph} clarifies the cases where $N$ is odd). Finally we use these observations to conclude the proof of the theorem.

We define the pooling process by using an ``eligible'' set. For a graph $\gG^o = (V^o, E^o)$, a set of pairs of nodes in $V^o$, $S = \{(i_1, j_1), ...(i_n, j_n)\}$, is called {\it eligible} if all nodes $i_1\cdots i_n$ and $j_1\cdots j_n$ are distinct and for any $m\in [n]$, $i_m$ and $j_m$ are connected by a path of length at most 2; that is, there are two cases: either $\{i_m, j_m\}\in E^o$ or there exists $k\in V^o$, such that $\{i_m, k\}\in E^o$ and $\{k, j_m\} \in E^o$. Using an arbitrarily chosen eligible set $S$, we 
describe a specific pooling process with respect to $S$ that produces a graph $\gG$ from $\gG^o$ as follows. We initialize $\gG$ with $V = V^o$ and $E = E^o$. For $m=1, 2, \cdots n$, we follow the next three steps: (1) We remove from $V$ the nodes $i_m$ and $j_m$. We remove from $E$  all edges in $E_m:=\{e\in E: i_m\in e \text{ or } j_m \in e\}$;  (2) We add a new node $i'_m$ to $V$; (3) We add the following set of new edges to $E$: $\mleft\{\{i'_m, k\}: k\in V \text{ and either } \{i_m, k\}\in E_m \text{ or } \{j_m, k\}\in E_m\mright\}$. 
It is clear that the resulting $\gG$ is connected if $\gG^o$ is connected. 

We introduce a lemma showing that for a given pooling process which maps $\gG^o$ to $\gG$, there exists an unpooling layer as the inverse of this pooling process, i.e., it maps $\gG$ to $\gG^o$.

\begin{lemma}\label{lemma:pool_unpool}
For any pooling process that maps $\gG^o = (V^o, E^o)$ to $\gG = (V, E)$, there exists an unpooling layer that maps $\gG$ to $\gG^o$.
\end{lemma}
\begin{proof}
Consider the pooling process with respect to the eligible pairs in $S=\{(i_1, j_1), ...(i_n, j_n)\} \subset V^o$. We use the same notation as above for $i_1', i_2', ...i_n' \subset V$ that were pooled by the respective eligible pairs. 

We construct an unpooling layer whose input is $\gG$ and its output is $\hat{\gG}^o = (\hat{V}^o, \hat{E}^o)$. The unpooling layer unpools the nodes $i_1', i_2', ...i_n'$ and keeps the remaining nodes unchanged. It forms the following children nodes of $i_1', i_2', ...i_n'$ in $\hat{V}^o$: $(f_1(i_1'), f_2(i_1')), ... (f_1(i_n'), f_2(i_n'))$, respectively. For every $r\in[n]$, we identify $(f_1(i_r'), f_2(i_r'))$ with $(i_r, j_r)$ and re-index respectively, so $\hat{V}^o = V^o$. 

It remains to show that we can find an unpooling layer such that $\hat{E}^o = E^o$. We first note that the edges in $E^o$ that do not contain nodes from the eligible set remain unchanged in $E$ and $E^o$ since the pooling process is identical on those edges (for clarity, these edges are the ones in $\{\{i, j\} \in E^o: i, j\notin \{i_1, i_2, ...i_n, j_1, j_2, ...j_n\}\}$). Since we restricted above the unpooling layer to only unpool $i_1', i_2', ...i_n'$ those edges also remain unchanged in $\hat{E}^o$.
We thus only need to show that we can construct the unpooling layer so that the set of edges in $\hat{E}^o$ that contain children nodes is the set of edges that contain nodes from the eligible set in $E^o$. 
Each edge in the latter set falls into one of the following three categories, for which we establish the required equality with the corresponding edges in $\hat{E}^o$:
\begin{enumerate}
    \item Edges whose end nodes form a pair $\{i_r, j_r\} \in S$. For each such edge, we select the intra-link in the unpooling layer (step 3a) so that there is an edge connecting $f_1(i_r')$ and $f_2(i_r')$ in $\hat{E}^o$. 

    \item Edges between an eligible pair $(i_r, j_r)$ and a static node $k$ in $V^o$. That is, for a fixed $r \in [n]$ and a static node $k$, there three possibilities for the set of these edges: $\{\{i_r,k\}\}$ or $\{\{j_r, k\}\}$ or $\{\{i_r, k\}, \{j_r, k\}\}$.
    In view of the pooling process, the edge $\{i'_r, k\}$ is in $E$. In Step 3c, there are three possibilities for determining $N_{\{i_r', k\}}(i_r')$ and we need to select the unpooling layer to match these three possibilities. That is, $N_{\{i_r', k\}}(i_r')=\{f_1(i_r')\}$ in the first case, where the set of the above edges is $\{\{i_r,k\}\}$; $N_{\{i_r', k\}}(i_r')=\{f_2(i_r')\}$ in the second case, where the set of the above edges is $\{\{j_r,k\}\}$ and $N_{\{i_r', k\}}(i_r')=\{f_1(i_r'), f_2(i_r')\}$ in the third case, where the set of the above edges is $\{\{i_r,k\}, \{j_r, k\}\}$.
    In view of the way $N_{\{i_r', k\}}(i_r')$ is used to build inter-links (see  \eqref{eq:add_inter_link_with_N}), the edges in $\hat{E}^o$ between $(f_1(i_r'), f_2(i_r'))$ and $k$  are the same as the ones in $E^o$ between $(i_r, j_r)$ and $k$.
    
    \item Edges between two different eligible pairs, $(i_r, j_r)$ and  $(i_s, j_s)$, that is $\{\{k, l\}\in E^o: k\in \{i_r, j_r\}, l\in \{i_s, j_s\}\}$. For fixed $s, r\in[n]$ this set of edges in $E^o$ is a nonempty subset of the following set of four edges: $\{\{i_r, i_s\}, \{i_r, j_s\}, \{j_r, i_s\}, \{j_r, j_s\}\}$. Therefore, there are $2^4-1=15$ such edge sets. In view of the pooling process, the edge $\{i_r', i_s'\}$ is in $E$. The unpooling layer unpools $i_r'$ to $(f_1(i_r'), f_2(i_r'))$ and unpools $i_s'$ to $(f_1(i_s'), f_2(i_s'))$. We claim that according to Steps 3c and 3d, 
    the unpooling layer can produce all the 15 possible edge sets between the pair $(f_1(i_r'), f_2(i_r'))$ and the pair $(f_1(i_s'), f_2(i_s'))$. 
    To clarify this claim we specify for all 15 possible edges between $(i_r, j_r)$ and $(i_s, j_s)$ the choice of $N_{\{i_r', i_s'\}}(i_r'), N_{\{i_r', i_s'\}}(i_s')$ in Step 3c and the choice of the additional edge in Step 3d:
    \begin{itemize}
        \item If the edge set is $\{\{i_r, i_s\}\}$, we set $N_{\{i_r', i_s'\}}(i_r')=\{f_1(i_r')\}, N_{\{i_r', i_s'\}}(i_s') = \{f_1(i_s')\}$ and we do not insert an edge in Step 3d
        \item If the edge set is $\{\{i_r, j_s\}\}$, we set $N_{\{i_r', i_s'\}}(i_r')=\{f_1(i_r')\}, N_{\{i_r', i_s'\}}(i_s') = \{f_2(i_s')\}$ and we do not insert an edge in Step 3d
        \item If the edge set is $\{\{j_r, i_s\}\}$, we set $N_{\{i_r', i_s'\}}(i_r')=\{f_2(i_r')\}, N_{\{i_r', i_s'\}}(i_s') = \{f_1(i_s')\}$ and we do not insert an edge in Step 3d
        \item If the edge set is $\{\{j_r, j_s\}\}$, we set $N_{\{i_r', i_s'\}}(i_r')=\{f_2(i_r')\}, N_{\{i_r', i_s'\}}(i_s') = \{f_2(i_s')\}$ and we do not insert an edge in Step 3d
        \item If the edge set is $\{\{i_r, i_s\}, \{i_r, j_s\}\}$, we set $N_{\{i_r', i_s'\}}(i_r')=\{f_1(i_r')\}, N_{\{i_r', i_s'\}}(i_s') = \{f_1(i_s'), f_2(i_s')\}$ and we do not insert an edge in Step 3d
        \item If the edge set is $\{\{j_r, i_s\}, \{j_r, j_s\}\}$, we set $N_{\{i_r', i_s'\}}(i_r')=\{f_2(i_r')\}, N_{\{i_r', i_s'\}}(i_s') = \{f_1(i_s'), f_2(i_s')\}$ and we do not insert an edge in Step 3d
        \item If the edge set is $\{\{i_r, i_s\}, \{j_r, i_s\}\}$, we set $N_{\{i_r', i_s'\}}(i_r')=\{f_1(i_r'), f_2(i_r')\}, N_{\{i_r', i_s'\}}(i_s') = \{f_1(i_s')\}$ and we do not insert an edge in Step 3d
        \item If the edge set is $\{\{i_r, j_s\}, \{j_r, j_s\}\}$, we set $N_{\{i_r', i_s'\}}(i_r')=\{f_1(i_r'), f_2(i_r')\}, N_{\{i_r', i_s'\}}(i_s') = \{f_2(i_s')\}$ and we do not insert an edge in Step 3d
        \item If the edge set is $\{\{i_r, i_s\}, \{j_r, j_s\}\}$, we set $N_{\{i_r', i_s'\}}(i_r')=\{f_1(i_r')\}, N_{\{i_r', i_s'\}}(i_s') = \{f_1(i_s')\}$ and in Step 3d we insert the edge $\{f_2(i_r'), f_2(i_s')\}$
        \item If the edge set is $\{\{i_r, j_s\}, \{j_r, i_s\}\}$, we set $N_{\{i_r', i_s'\}}(i_r')=\{f_1(i_r')\}, N_{\{i_r', i_s'\}}(i_s') = \{f_2(i_s')\}$ and in Step 3d we insert the edge $\{f_2(i_r'), f_1(i_s')\}$
        \item If the edge set is $\{\{i_r, i_s\}, \{i_r, j_s\}, \{j_r, i_s\}\}$, we set $N_{\{i_r', i_s'\}}(i_r')=\{f_1(i_r'), f_2(i_r')\}, N_{\{i_r', i_s'\}}(i_s') = \{f_1(i_s')\}$ and in Step 3d we insert the edge $\{f_1(i_r'), f_2(i_s')\}$
        \item If the edge set is $\{\{i_r, i_s\}, \{j_r, i_s\}, \{j_r, j_s\}\}$, we set $N_{\{i_r', i_s'\}}(i_r')=\{f_1(i_r'), f_2(i_r')\}, N_{\{i_r', i_s'\}}(i_s') = \{f_1(i_s')\}$ and in Step 3d we insert the edge $\{f_2(i_r'), f_2(i_s')\}$
        \item If the edge set is $\{\{i_r, i_s\}, \{i_r, j_s\}, \{j_r, j_s\}\}$, we set $N_{\{i_r', i_s'\}}(i_r')=\{f_1(i_r'), f_2(i_r')\}, N_{\{i_r', i_s'\}}(i_s') = \{f_2(i_s')\}$  and in Step 3d we insert the edge $\{f_1(i_r'), f_1(i_s')\}$
        \item If the edge set is $\{\{i_r, j_s\}, \{j_r, i_s\}, \{j_r, j_s\}\}$, we set $N_{\{i_r', i_s'\}}(i_r')=\{f_1(i_r'), f_2(i_r')\}, N_{\{i_r', i_s'\}}(i_s') = \{f_2(i_s')\}$  and in Step 3d we insert the edge $\{f_2(i_r'), f_1(i_s')\}$
        \item If the edge set is $\{\{i_r, i_s\}, \{i_r, j_s\}, \{j_r, i_s\}, \{j_r, j_s\}\}$, we set $N_{\{i_r', i_s'\}}(i_r')=\{f_1(i_r'), f_2(i_r')\}, N_{\{i_r', i_s'\}}(i_s') = \{f_1(i_s'), f_2(i_s')\}$ and we do not insert an edge in Step 3d
    \end{itemize}

    The inter-link construction in \eqref{eq:add_inter_link_with_N} for the above specified choices of $N_{\{i_r', i_s'\}}(i_r')$ and $N_{\{i_r', i_s'\}}(i_s')$ together with the above specified choices of inserting an additional edge imply that the output edges between $(f_1(i_r'), f_2(i_r'))$ and $(f_1(i_s'), f_2(i_s'))$ in $\hat{E}^o$ are the same as the edges between $(i_r, j_r)$ and $(i_s, j_s)$.
\end{enumerate}
\end{proof}

To show that for a graph $\gG^o=(V^o, E^o)$ with $N$ node there is an unpooling layer that maps a graph $\gG$ with $N-n$ nodes to $\gG^o$, we need to show that there exists an eligible set $S=\{(i_1,j_1), ... (i_n, j_n)\}\subset V^o$. Indeed, the pooling process with respect to $S$ maps $\gG^o$ to a graph $\gG$ with $N-n$ nodes and by Lemma~\ref{lemma:pool_unpool} there exists an unpooling layer that maps $\gG$ to $\gG^o$. 

The next two lemmas conclude the theorem by implying that for a connected graph $\gG^o$ there exists an eligible set of maximal size, i.e., $n=\lfloor N/2 \rfloor$. The first lemma considers a graph with an even number of nodes and the second lemma considers a graph with an odd number of nodes.

\begin{lemma}
\label{lemma:even_graph}
For any connected graph $\gG^o=(V^o, E^o)$ with $2K$ nodes, there exists an eligible set $S$ containing $K$ pairs of nodes in $\gG^o$ such that the pooling process with respect to $S$ maps $\gG^o$ to a graph $\gG$ with $K$ nodes.
\end{lemma}
\begin{proof}
We prove this lemma by induction using $M=1,\ldots,K$. When $M=1$, the lemma is trivial.

Assume the lemma holds for $M=1, \ldots, K-1$. Given a connected graph $\gG^o=(V^o, E^o)$ with $2K$ nodes, we prove the result, while considering the following two different scenarios:

{\bf Case 1:} There exists a node in $V^o$ of degree $1$. We arbitrarily choose such a node and denote it by $j$. We consider its only neighbor, which we denote $k\in V^o$, and remove the pair of nodes $(j, k)$ and all the edges connected to them from the graph $\gG^o$. 
The remaining graph has $2K-2$ nodes.  
If it is also connected, then by the induction assumption
there exists an eligible set with $K-1$ pairs of nodes.
By adding the pair $(j, k)$ to that eligible set, we obtain an eligible set with $K$ pairs and conclude the proof.
    
If, on the other hand, the remaining graph is not connected, we partition it into maximally connected subgraphs $\gG_1, \gG_2, \cdots, \gG_m$. That is, each subgraph is  connected, but any two subgraphs are not connected to each other. 
Since $\gG_1, \ldots \gG_{m}$ are not connected to each other and to the degree-one node $j$, they all connect to the node $k$. 

We consider the following four steps that assist in finding an eligible set:
\begin{enumerate}
    \item 
    We identify the maximally connected subgraphs with even numbers of nodes. Clearly, each of these numbers is not greater than $2K-2$ and thus by induction all the nodes of each such subgraph can be used to form an eligible set. 
    The union of all these eligible sets forms a larger eligible set that uses all the nodes in these subgraphs. 
    If all maximally connected subgraphs have even numbers of nodes, then we terminate the procedure at this step. 
    
    \item 
    We re-index the maximally connected subgraphs with odd numbers of nodes as $\gG_1, ... \gG_{2s}$, for some $s \in \mathbb{N}$. We note that the total number, $2s$, is indeed even since the total number of the remaining nodes is even and the number of nodes in each subgraph is odd. For $1 \leq i \leq 2s$, let $g_i$ be a node in $\gG_i$ that connects to $k$ (we commented above on its existence). 
    We form the following $s$ pairs: $(g_1, g_2), (g_3, g_4), ...(g_{2s-1}, g_{2s})$. We note that they form an eligible set since they all connect to $k$. 
    This eligible set only uses the nodes $\{g_i\}_{i=1}^{2s}$.
We terminate the procedure at this step whenever all maximally connected subgraphs with an odd number of nodes only contain a single node (so that all nodes of these subgraphs are used by this eligible set).
    
    \item If $\gG_i \neq (\{g_i\}, \emptyset)$ and the number of nodes of $\gG_i$ is odd, then we remove $g_i$ from $\gG_i$ and also remove all the edges connected to $g_i$. The remaining graph contains an even number of nodes. If the remaining graph is connected, then we find an eligible set that uses all nodes in this remaining subgraph (its existence follows from the induction assumption). We terminate the procedure at this step if each of these remaining subgraphs is connected (the union of all such eligible sets then form a larger eligible set that uses all nodes in these subgraphs). 
    
    \item
    If a remaining subgraph from the above step (having $g_i$ and its connected edges removed from $\gG_i$) is not connected, we form its maximally connected subgraphs $\gH_{1}^{(i)}, \ldots \gH_{l_i}^{(i)}$. We note that $\gH_{1}^{(i)}, \ldots \gH_{l_i}^{(i)}$ are all connected to $g_i$ (since $\gG_i$ is connected and $\gH_1, \ldots \gH_{m_i}$ are not connected to each other). We also observe that $\cup_{r=1}^{l_i} \gH^{(i)}_r$ contains less than $2K-2$ nodes (indeed the number of nodes of $\gG_i$ is strictly less than the number of nodes in $\cup_{r=1}^{m}\gG_r$, which is $2K-2$; we remark that since $g_i$ was removed from $\gG_i$ the bound $2K-2$ is not tight).

\end{enumerate}
If this procedure terminates in its first three steps, then it finds an eligible set that uses all nodes in the maximally connected subgraphs and thus its size is $2K-2$.
Otherwise, we iteratively apply the same four-steps procedure on the maximally connected subgraphs of the last step. At each iteration the total number of nodes in the maximally connected subgraphs reduces (we clarified this at the end of the fourth step).
Since the graph is finite, the iteration either terminates or 
$\cup_{r=1}^{l_i} \gH^{(i)}_r$ in step 4 of the procedure is of size 
2, that is, there are two subgraphs with single nodes. When inputting these single nodes at the next iteration, the procedure will terminate at step 2.

The final eligible set is the union of all the eligible sets iteratively generated from steps (1)-(3) 
and the pair $(j, k)$. This eligible set uses all the nodes in $V^o$ and thus it contains $K$ pairs.

We remark that we introduced Case 1 in order to help the reader master the idea of the proof in a simpler case. We actually demonstrated how to consecutively handle the case where a single node is connected to maximally connected subgraphs whose total number of nodes is even. Starting with a node of degree 1 allowed us to proceed with this idea in one direction. Next, we pick up two different points and proceed with this idea in two different directions and, in fact, we could have started the proof with the latter setting right away. This setting has two subcases (2A and 2B). In Case 2A, we still have an even number of nodes in the maximally connected subgraphs, which are connected to a single node, so the ideas of Case 1 immediately apply. In Case 2B, the latter number of nodes is odd, but we can somehow reduce it to Case 2A.

  {\bf Case 2:} There does not exist any node in $V^o$ with degree $1$. In this case we randomly select two neighboring nodes $j, k\in V^o$. 
    We consider the remaining graph after removing these two nodes from $V^o$ and also remove all edges connected to them from $E^o$. If the remaining graph is connected or if all the maximally connected subgraphs of the remaining graph contain an even number of nodes, then the induction assumption concludes the proof.
    Otherwise, the remaining graph is partitioned into maximally connected subgraphs $\gG_1, ..., \gG_{n+m}$. We reindex these subgraphs so $\gG_1, ...\gG_m$ are all connected to node $j$ and not connected to node $k$ in $\gG^o$. The other maximally connected subgraphs, $\gG_{m+1}, ..\gG_{m+n}$, are connected to $k$ in $\gG^o$ (they may or may not be connected to $j$). We prove this case by further considering two different scenarios:

    {\bf Case 2A:} $\cup_{i=1}^m \gG_i$ contains an even number of nodes. In this case, it is clear that $\cup_{i=m+1}^{n+m} \gG_i$ also contains an even number of nodes. We could iteratively perform the above four-steps procedure introduced in Case 1 on $\{\gG_i\}_{i=m+1}^{n+m}$ (as they all connect to $k$) and on $\{\gG_i\}_{i=1}^m$ (as they all connect to $j$). Following the same argument at the end of the proof of case 1, we obtain two eligible sets that cover $\{\gG_i\}_{i=m+1}^{n+m}$ and $\{\gG_i\}_{i=1}^m$, respectively. The union of these two sets and the pair $(j, k)$ yields an eligible set that uses all nodes in $\gG^o$, which concludes the proof.
    
    {\bf Case 2B:} $\cup_{i=1}^m \gG_i$ contains an odd number of nodes. 
    We note that $\cup_{i=m+1}^{n+m}\gG_i$ contains an odd number of nodes. We further note that
    within $\{\gG_i\}_{i=1}^{m}$, there is an odd number of subgraphs that contain an odd number of nodes. After reindexing, these subgraphs are $\gG_1, \gG_2, ...\gG_{2r+1}$, where $2r+1 \leq m$.
    We pick one node from $\gG_{2r+1}$ that is connected to $j$ and denote it by $l$. By definition, $l$ is not connected to $\gG_i$, $i\in [2r]$ or $i=2r+2,\ldots m$, and also not connected to node $k$.
    Using this special node $l$, we redefine the ``remaining graph'' that was described in the beginning of case 2. That is, we remove from $\gG^o$ the pair $(j, l)$ (and the associated edges), whereas before the pair $(j, k)$ was removed from it. We can similarly identify maximally connected subgraphs of this new remaining graph. Now the collection of the maximally connected subgraphs that connect to $j$ but not connect to $l$ contains an even number of nodes in total.
    Indeed, this collection contains all nodes in the former subgraphs, $\gG_1, \ldots \gG_{2r}$, $\gG_{2r+2}, \ldots, \gG_{m+n}$ and also the node $k$ and there is an odd number of nodes in every $\gG_i$, $i=1,...2r$, an even number of nodes in every $\gG_i$, $i=2r+2, \ldots m$ and an odd number of nodes in $\cup_{i=m+1}^{n+m}\gG_i$ and the single node $k$, which yield an even number of total nodes.
    Note that we transformed the case to the one in case 2A (with the initially selected pair $(j, l)$) and the proof is thus concluded. 
\end{proof}

\begin{lemma}\label{lemma:odd_graph}
For a connected graph $\gG^o$ with $2K+1$ nodes, there exists an eligible set $S$ containing $K$ pairs of nodes in $\gG^o$ and the pooling process with respect to $S$ maps this graph to a graph with $K+1$ nodes.
\end{lemma}
\begin{proof}
We note that there exists a node in $\gG^o$ so that the remaining graph is still connected after removing this node from $V^o$ and removing all edges connected to this node from $E^o$. Indeed, it can be selected as a node with degree 1 in the spanning tree of $\gG^o$.
Considering the remaining graph with $2K$ nodes (after removing this node and the associated edges), Lemma~\ref{lemma:even_graph} implies the existence of an eligible set $S$ containing $K$ pairs of nodes so that the pooling process with respect to $S$ will map $\gG^o$ to a graph with $K+1$ nodes.

\end{proof}

{\bf Proof of Theorem~2. }
To conclude the theorem, we just need to show that for $\gG^o = (V^o, E^o)$ with $|V^o| = N$, and any number $K\in [\lceil N/2 \rceil, N-1]$, there exists a pooling process with an eligible set $S$ (containing $N-K$ pairs) that maps $\gG^o$ to a graph $\gG$ with $K$ nodes. It is sufficient to prove the statement for $K=\lceil N/2 \rceil$ with the eligible set $S^*$ that contains $N-\lceil N/2 \rceil$ pairs. 
Indeed, for any  $K > \lceil N/2\rceil$, we can take a subset $S_K$ of the eligible set $S^*$ (containing $N-\lceil N/2 \rceil$ pairs) such that $|S_K| = |S^*|- (K - \lceil N/2\rceil)$. The pooling process with respect to $S_K$ produces a graph with $K$ nodes.

Given a graph $\gG^o$ with $N$ nodes, Lemma~\ref{lemma:even_graph} and Lemma~\ref{lemma:odd_graph} imply that there exists an eligible set $S^*$ with $\lfloor N/2\rfloor$ pairs. Therefore, the pooling procedure with respect to $S^*$ maps $\gG^o$ to a graph $\gG$ with $\lceil N/2\rceil$ nodes. 
\qed

\subsection{Proof of Corollary~3}

For any connected graph $\gG^o$ with $N$ nodes,  iterative application of Theorem~2 implies the existence of a series of unpooling layers $U_k$, $k \geq 1$, and intermediate graphs $\gG_k$, $k \geq 1$, with $|\gG_k| = \lceil |\gG_{k-1}|/2 \rceil$, so that $U_k(\gG_k) = \gG_{k-1}$, where $\gG_0 := \gG^o$. We stop the process once $|\gG_k| \in [4, 6]$ (we will reach this range because for any $N'> 6$, $\lceil N'/2\rceil \geq 4$). Another application of Theorem~2 yields the existence of a 3-nodes graph, $\gG$, and a last unpooling layer, $U_{k+1}$, so that $U_{k+1} (\gG) = \gG_k$.

That is, there exist a 3-nodes graph $\gG$ and  a series of unpooling layers $U_1, ..U_{k+1}$ so that $U_{1}\circ U_2 \circ... \circ U_{k+1} (\gG) = \gG^o$. Note that the number of unpooling layers is $k+1 = \lceil\log_2 (N/3)\rceil$. $\qed$

\section{ANOTHER GENERATION TASK: OPTIMIZING SPECIFIC CHEMICAL PROPERTIES} 
\label{sec:opt_chem_prop}

In this task, we consider the following three chemical properties \citep{guimaraes2017organ}, which we evaluate only on the valid molecules among the 10,000 generated ones: druglikeliness, solubility and synthesizability. 
We calculate the first property using the Quantitative Estimate of Druglikeness (QED) package in RDKit (\url{https://www.rdkit.org/}) (licensed by BSD 3-Clause). These scores lie in $[0,1]$ and their values aim to express the likelihood of being a drug. We calculate solubility by the log octanol-water partition coefficient using the Crippen package in RDKit. We rescale this value to lie in $[0, 1]$, 
where 1 is the most soluble value.
In order to calculate synthesizability, we calculate the  synthetic accessibility \citep{ertl2009sascore} and rescale this value 
to lie in $[0, 1]$ where 1 is the easiest to synthesize. We use codes from \url{https://github.com/connorcoley/scscore} , licensed by MIT.

For the generative model, we follow with the same Wasserstein GAN architecture as described in \S\ref{sec:details} (with the final activation function to be a sigmoid function), but the discriminator minimizes the error between its output and the objective chemical property; it then outputs a reward score, which we need to maximize when training the generator.

{\bf Result of optimizing chemical properties.} Using QM9, we generated molecules that aimed to maximize druglikeliness, solubility and synthesizability.  Table~\ref{tab:qm9_prop} reports the six evaluation metrics (listed in its columns, whereas the properties we aimed to maximize are in its rows).  
In terms of generating molecules with the targeted chemical properties, UL GAN is competitive with the other approaches (this is noticed when looking at columns 4, 5, 6 of rows 1, 2, 3, respectively. 
When considering the other evaluation metrics, UL GAN generally outperforms Adj GAN, except for the objective of druglikeliness and the metric of synthesizability (first row and sixth column).
Note that the uniqueness of UL GAN and Adj GAN is very low because our generator does not aim to compete with the discriminator, but to generate molecules with a maximal property of interest. We did not report the good performance of UL GAN when considering this task with ZINC,
since the other methods we compared with were not tested on ZINC; furthermore, the superiority of UL GAN over Adj GAN for ZINC is already obvious from Table~4.

\begin{table*}[ht]
\caption{The six evaluation metrics (in rows) for generated samples that aim to minimize the three indicated chemical properties (in columns). We remark that QED is the acronym for quantitative estimate of druglikeliness. Scores for competing methods (above the indicated line) were copied from their original papers. NA means that the score is not available in the original papers.} 
\label{tab:qm9_prop}
\vskip 0.15in
\begin{center}
\begin{small}
\begin{tabular}{lccccccc}
\toprule
Objective& Method & Valid & Unique & Novel & QED &Solubility & Synthesizability \\
\midrule
\multirow{5}{*}{QED}& ORWGAN & 0.882 & 0.694 &  NA & 0.52 & 0.35 & 0.32\\
& Naive RL & 0.971 & 0.540 & NA & 0.57 & 0.50& 0.53\\
& MolGAN & 1.00 & 0.022 & NA & {\bf 0.62} & 0.59& 0.53\\
\cline{2-8}
& Adj GAN      & 0.991 & 0.005 & 0.865 & 0.443 & 0.288 & 0.658 \\
& UL GAN  & 0.9888 &  0.051 & 0.978 & 0.598 & 0.497 & 0.485 \\
\midrule
\multirow{5}{*}{Solubility}& ORWGAN & 0.965 & 0.459 & NA & 0.50 & 0.55& 0.63\\
& Naive RL & 0.927 & 1.00 & NA & 0.49 & 0.78& 0.70\\
& MolGAN & 0.998 & 0.002 & NA & 0.44 & {\bf 0.89}& 0.22\\
\cline{2-8}
& Adj GAN      & 0.940 & 0.003 & 0.958 & 0.378 & 0.367 & 0.007 \\
& UL GAN & 0.993 & 0.010 & 0.781 & 0.507 & 0.700 & 0.793 \\
\midrule
\multirow{5}{*}{Synthesizability}& ORWGAN & 0.965 & 0.459 & NA & 0.51 & 0.45& 0.83\\
& Naive RL & 0.977 & 0.136 & NA & 0.52 & 0.46& 0.83\\
& MolGAN & 1.00 & 0.021 & NA & 0.53 & 0.68& 0.95\\
\cline{2-8}
& Adj GAN      & 0.999 & 0.003 & 0.833 & 0.360 & 0.331 & 0.835 \\ 
& UL GAN & 1.00 & 0.006 & 0.433 & 0.468 & 0.569 & {\bf 0.953} \\
\bottomrule
\end{tabular}
\end{small}
\end{center}
\vskip -0.1in
\end{table*}

\section{SOME ADDITIONAL NUMERICAL RESULTS}\label{sec:additional_result}

We supplement the numerical results in \S\ref{sec:experiment}. 
Section~\ref{sec:stdres} reports standard deviations for the earlier experiments on QM9 and ZINC; \S\ref{subsec:1k_result} includes generation results with only 1,000 samples; and \S\ref{sec:samples} demonstrates examples of the generated molecules from UL GAN and UL VAE.

\subsection{Standard Deviations for Molecule Generation}
\label{sec:stdres}

Table~\ref{tab:std_qm9} supplements Table~3 in the main manuscript and reports the means and standard deviations of the evaluation metrics for our methods in QM9, including Adj GAN, UL GAN, and UL VAE. Table~\ref{tab:std_zinc} supplements Table~4 in the main manuscript and reports the means and standard deviations of the evaluation metrics for our methods in ZINC, including UL GAN and Adj GAN and UL VAE. These means and standard deviations are calculated from 100 runs.  

\begin{table}[ht]
\caption{Validity, uniqueness and novelty with standard deviation for Adj GAN, UL GAN, and UL VAE using QM9. } 
\label{tab:std_qm9}
\vskip 0.15in
\begin{center}
\begin{small}
\begin{tabular}{lcccr}
\toprule
Method & Valid & Unique & Novel \\
\midrule
UL VAE & 0.735 ($\pm$ 0.004) & 0.940 ($\pm$ 0.003) & 0.949 ($\pm$ 0.002)\\
Adj GAN      & 0.941 ($\pm$ 0.002) & 0.139 ($\pm$ 0.002) & 0.886 ($\pm$ 0.006) \\
UL GAN & 0.907 ($\pm$ 0.003) & 0.826 ($\pm$ 0.004) & 0.949 ($\pm$ 0.002)\\
\bottomrule
\end{tabular}
\end{small}
\end{center}
\vskip -0.1in
\end{table}

\begin{table}[ht]
\caption{
Validity, Uniqueness and Novelty with standard deviation for Adj GAN and UL GAN using ZINC.}
\label{tab:std_zinc}
\vskip 0.15in
\begin{center}
\begin{small}
\begin{tabular}{lcccr}
\toprule
Method & Valid & Unique & Novel \\
\midrule
Adj GAN   & 0.109 ($\pm$ 0.003) & 0.196 ($\pm$ 0.011) & { 1.00} ($\pm$ 0) \\
UL GAN & 0.871 ($\pm$ 0.004) & { 1.00} ($\pm$ 0) & { 1.00} ($\pm$ 0) \\
\bottomrule
\end{tabular}
\end{small}
\end{center}
\vskip -0.1in
\end{table}

\subsection{Evaluation with only 1,000 generated samples.}\label{subsec:1k_result}
In the experiments of \S\ref{sec:experiment}, we generated 10,000 samples and reported statistics based on these samples. 
To check whether the performance is preserved for a smaller sample, we report here results of UL GAN and UL VAE when generating only 1,000 samples. Table~\ref{tab:1kreport} reports such results for the Waxman, protein, QM9 and ZINC datasets. We note that all the reported metrics, but the uniqueness in QM9,  are similar to the ones in Table~\ref{tab:waxman}, Table~\ref{tab:protein}, Table~\ref{tab:qm9} and Table~\ref{tab:zinc}, where 10k samples were generated. The uniqueness in QM9 is significantly higher with 1,000 samples. \citet{polykovskiy2020molecular} also noticed a higher uniqueness rate with 1,000 samples than with 10,000 samples. 

\begin{table}
\caption{Model performance with 1,000 generated samples on various datasets.}\label{tab:1kreport}
\vspace{-.13in}
\begin{small}
\begin{center}
\begin{scriptsize}
\begin{tabular}{|l|cccc|cccc|ccc|}
\hline
& \multicolumn{8}{c|}{Waxman random graph} & \multicolumn{3}{c|}{QM9}  \\ \hline
Method & kl edge dense & kl clust& kl conn& kl node feat & wd edge dense & wd clust& wd conn& wd node feat & Valid & Unique & Novel \\ \hline
UL VAE & 0.115 & 0.443 & 0.279 &  0.451 & 0.009 & 0.100 & 0.106 & 0.126 & 0.737   & 0.991  & 0.932   \\
UL GAN & 0.007 & 0.030 & 0.033 & 0.152 & 0.002 & 0.022 & 0.012  & 0.039 & 0.905  & 0.970  & 0.927   \\
\hline
& \multicolumn{8}{c|}{Protein dataset} & \multicolumn{3}{c|}{ZINC}  \\ \hline
 & kl edge dense & kl clust& kl conn& kl node feat& wd edge dense & wd clust& wd conn& wd node feat& Valid & Unique & Novel \\ \hline
UL VAE & 0.565 & 1.235 & 0.822 & 0.260 &  0.034 & 0.071 & 0.319 & 4.400 & NA & NA & NA  \\
UL GAN  & 0.084 & 0.484 & 0.136 & 0.234 & 0.014 & 0.011 & 0.115  & 3.958 & 0.870 & 1.00  & 1.00  \\
\hline
\end{tabular}
\end{scriptsize}
\end{center}
\end{small}
\vskip -0.1in
\end{table}

\subsection{Synthetic Samples for Molecule Generation}
\label{sec:samples}

We demonstrate samples generated from the UL GAN for QM9 in Figure~\ref{fig:qm_complete_sample} and samples generated from the UL VAE in Figure~\ref{fig:ulvae_qm9}. Also, we present samples generated from the UL GAN for ZINC in Figure~\ref{fig:zinc_sample}.

We also illustrate some examples of the evolution of intermediate graphs and illustrate how a graph is generated from a generative GNN using unpooling layers, in Figure~\ref{fig:seq_gen} for QM9 and in Figure~\ref{fig:seq_gen_zinc} for ZINC.

    \begin{figure}[ht]
    \begin{center}
\includegraphics[scale=0.25]{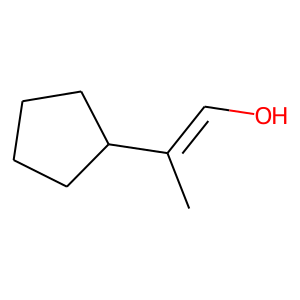}
\includegraphics[scale=0.25]{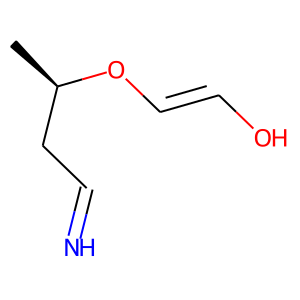}
\includegraphics[scale=0.25]{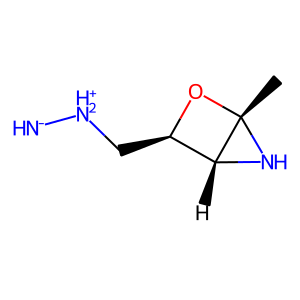}
\includegraphics[scale=0.25]{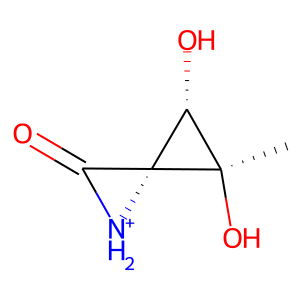}
\includegraphics[scale=0.25]{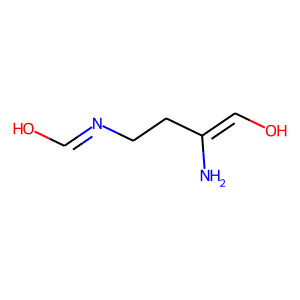}
\includegraphics[scale=0.25]{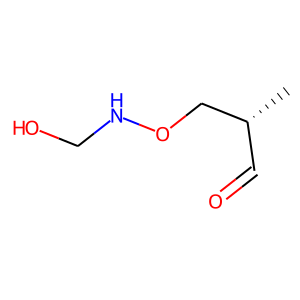}
\includegraphics[scale=0.25]{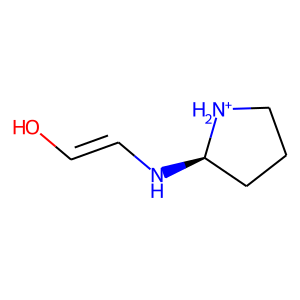}
\includegraphics[scale=0.25]{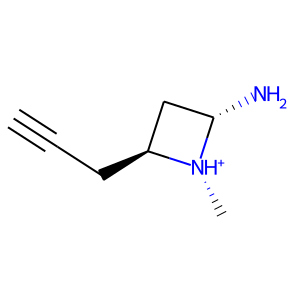}
\includegraphics[scale=0.25]{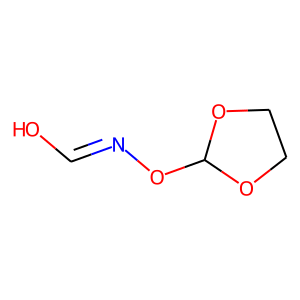}
\includegraphics[scale=0.25]{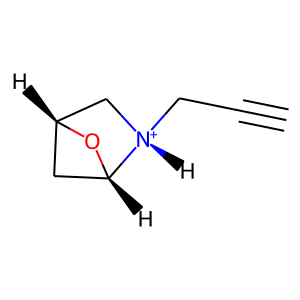}
\includegraphics[scale=0.25]{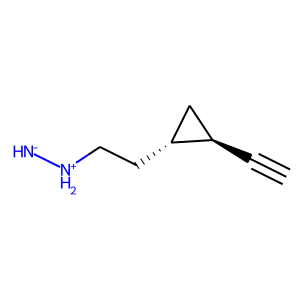}
\includegraphics[scale=0.25]{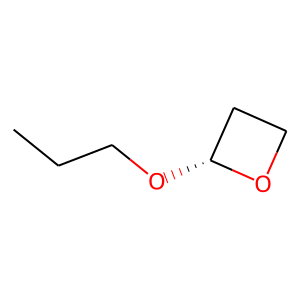}
\includegraphics[scale=0.25]{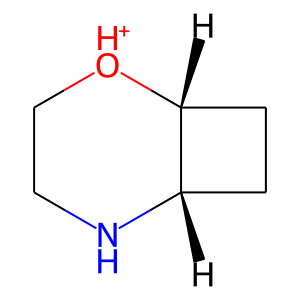}
\includegraphics[scale=0.25]{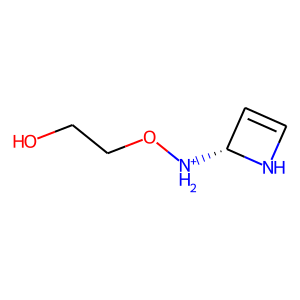}
\includegraphics[scale=0.25]{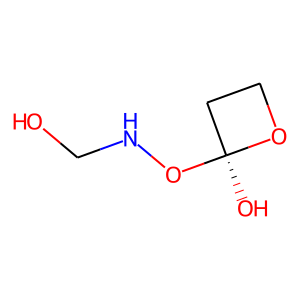}
\includegraphics[scale=0.25]{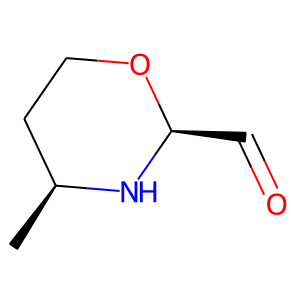}
\includegraphics[scale=0.25]{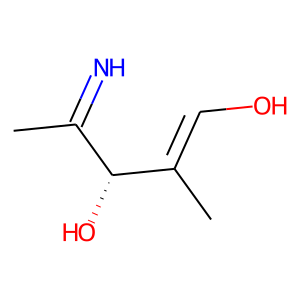}
\includegraphics[scale=0.25]{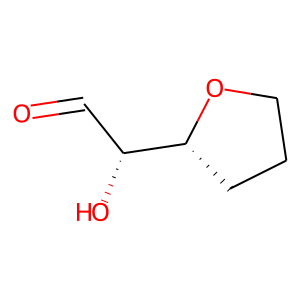}
\includegraphics[scale=0.25]{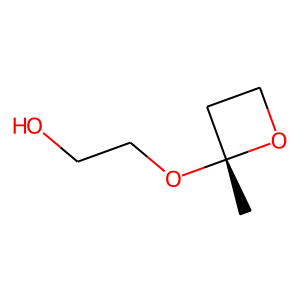}
\includegraphics[scale=0.25]{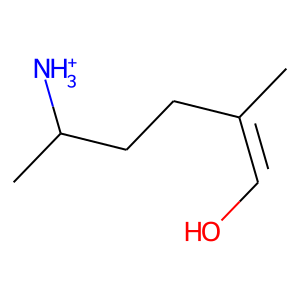}
\includegraphics[scale=0.25]{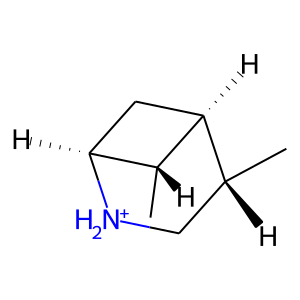}
\includegraphics[scale=0.25]{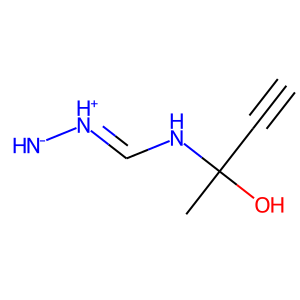}
\includegraphics[scale=0.25]{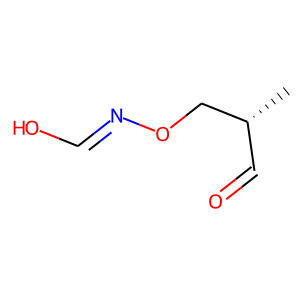}
\includegraphics[scale=0.25]{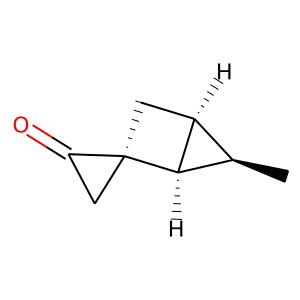}
   \caption{ Samples of molecules generated by UL GAN based on QM9 dataset.}\label{fig:qm_complete_sample}
    \end{center}
    \end{figure}   

\newpage

    \begin{figure}[ht]
    \begin{center}
\includegraphics[scale=0.25]{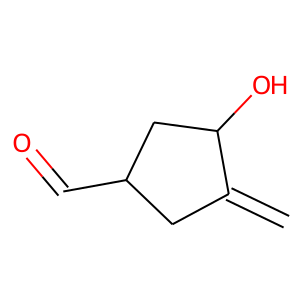}
\includegraphics[scale=0.25]{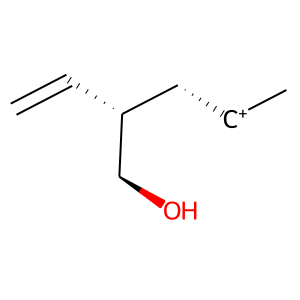}
\includegraphics[scale=0.25]{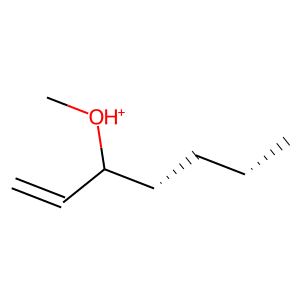}
\includegraphics[scale=0.25]{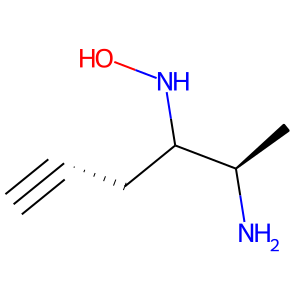}
\includegraphics[scale=0.25]{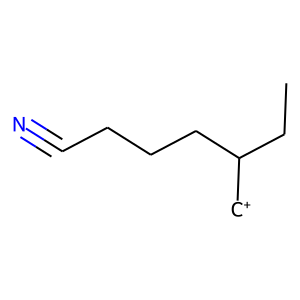}
\includegraphics[scale=0.25]{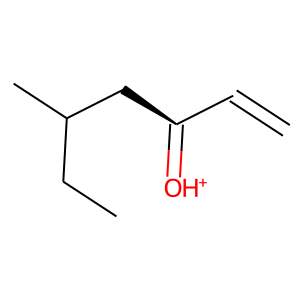}
\includegraphics[scale=0.25]{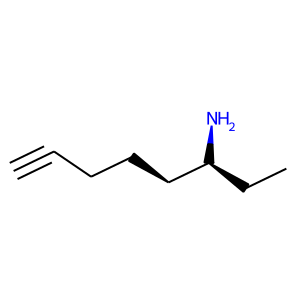}
\includegraphics[scale=0.25]{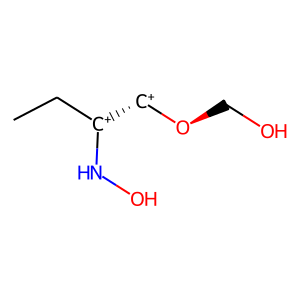}
\includegraphics[scale=0.25]{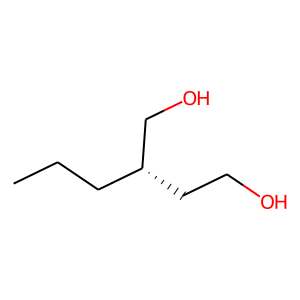}
\includegraphics[scale=0.25]{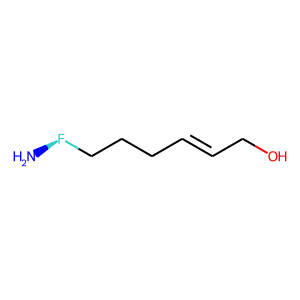}
\includegraphics[scale=0.25]{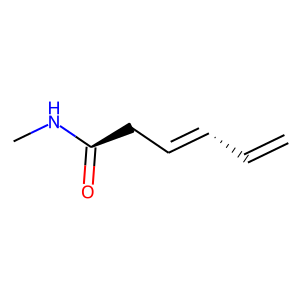}
\includegraphics[scale=0.25]{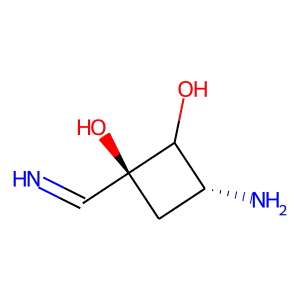}
\includegraphics[scale=0.25]{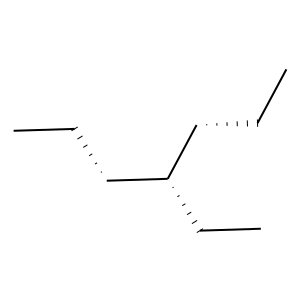}
\includegraphics[scale=0.25]{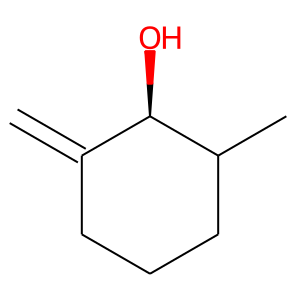}
\includegraphics[scale=0.25]{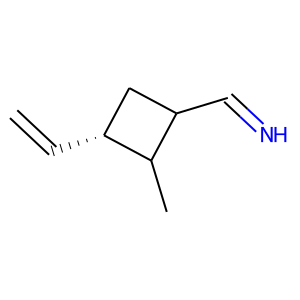}
\includegraphics[scale=0.25]{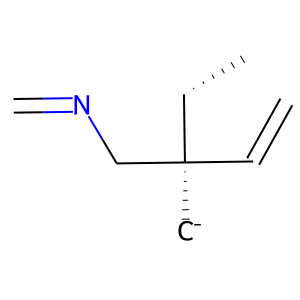}
\includegraphics[scale=0.25]{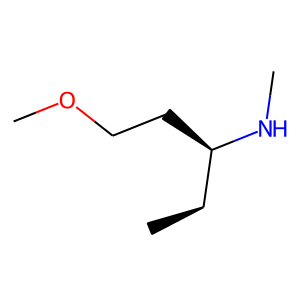}
\includegraphics[scale=0.25]{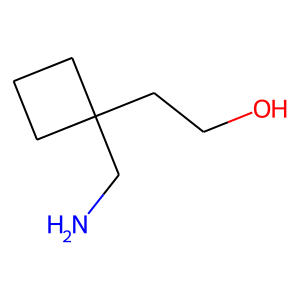}
\includegraphics[scale=0.25]{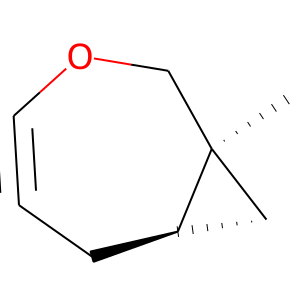}
\includegraphics[scale=0.25]{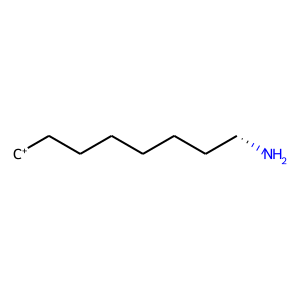}
\includegraphics[scale=0.25]{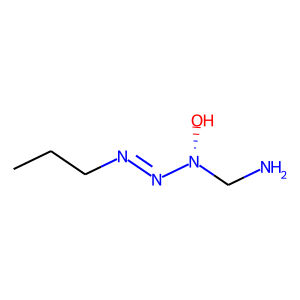}
\includegraphics[scale=0.25]{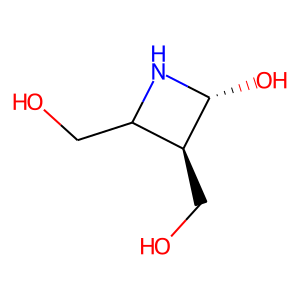}
\includegraphics[scale=0.25]{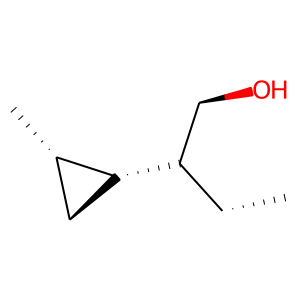}
\includegraphics[scale=0.25]{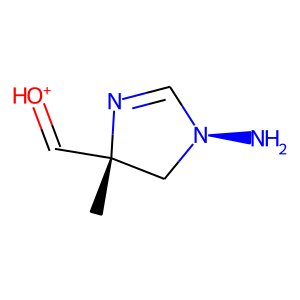}
   \caption{Samples of molecules generated from UL VAE based on QM9 dataset.}\label{fig:ulvae_qm9}
    \end{center}
    \end{figure}

\newpage

    \begin{figure}[ht]
    \begin{center}
\includegraphics[scale=0.25]{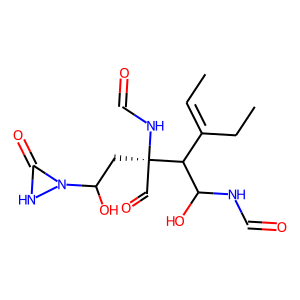}
\includegraphics[scale=0.25]{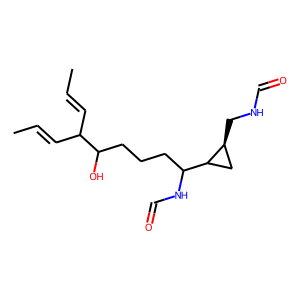}
\includegraphics[scale=0.25]{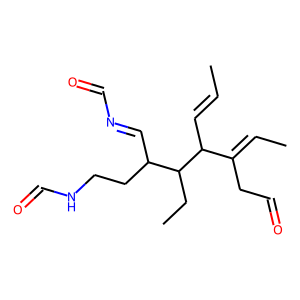}
\includegraphics[scale=0.25]{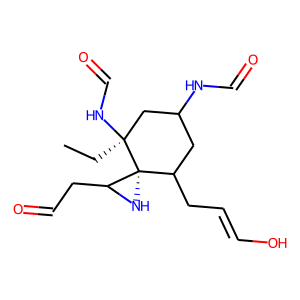}
\includegraphics[scale=0.25]{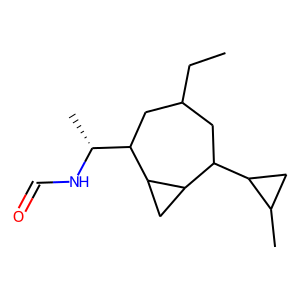}
\includegraphics[scale=0.25]{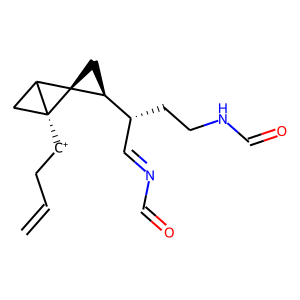}
\includegraphics[scale=0.25]{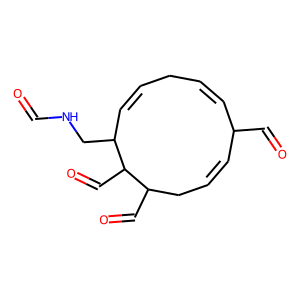}
\includegraphics[scale=0.25]{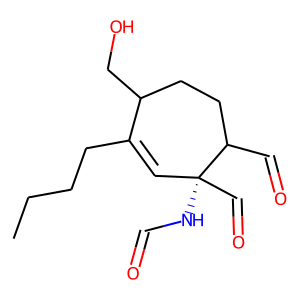}
\includegraphics[scale=0.25]{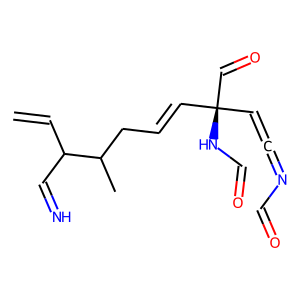}
\includegraphics[scale=0.25]{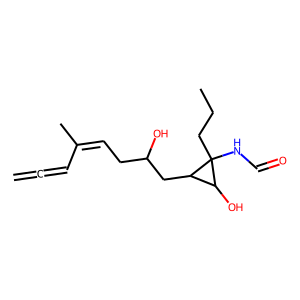}
\includegraphics[scale=0.25]{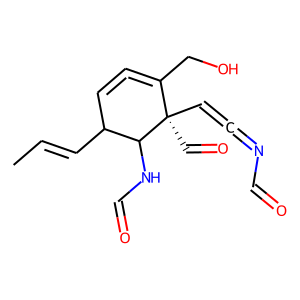}
\includegraphics[scale=0.25]{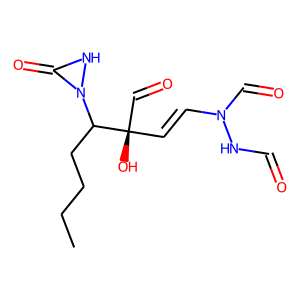}
\includegraphics[scale=0.25]{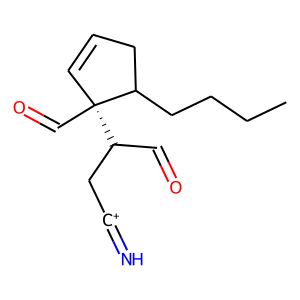}
\includegraphics[scale=0.25]{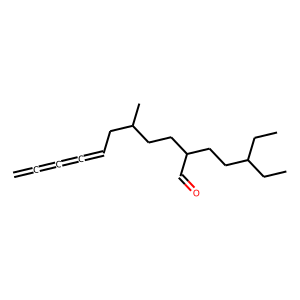}
\includegraphics[scale=0.25]{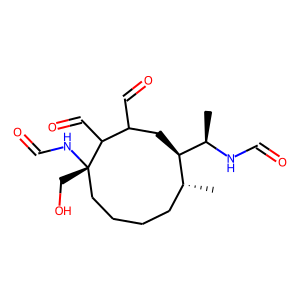}
\includegraphics[scale=0.25]{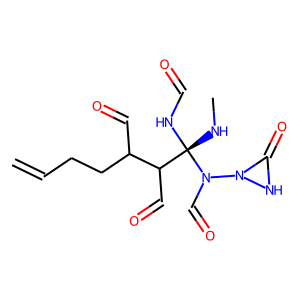}
\includegraphics[scale=0.25]{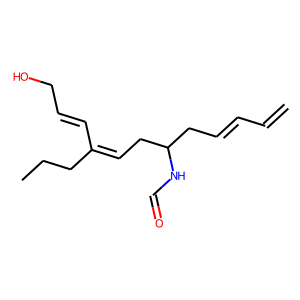}
\includegraphics[scale=0.25]{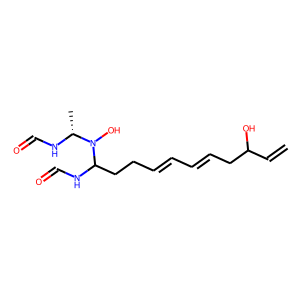}
\includegraphics[scale=0.25]{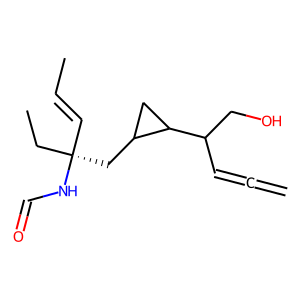}
\includegraphics[scale=0.25]{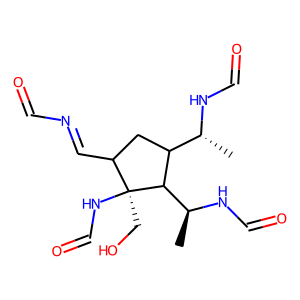}
\includegraphics[scale=0.25]{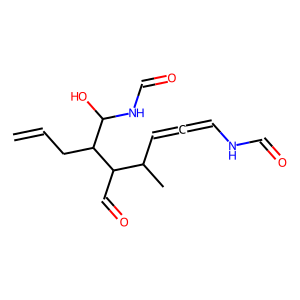}
\includegraphics[scale=0.25]{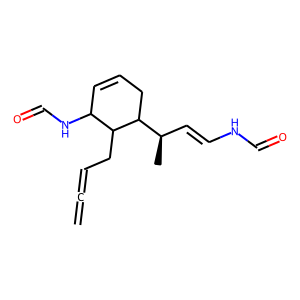}
\includegraphics[scale=0.25]{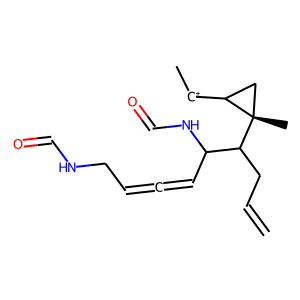}
\includegraphics[scale=0.25]{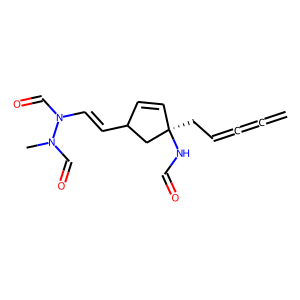}
   \caption{Samples of molecules generated from UL GAN based on ZINC dataset.}\label{fig:zinc_sample}
    \end{center}
    \end{figure}   
    
\newpage

    \begin{figure}[H]
    \begin{center}
\includegraphics[scale=0.25]{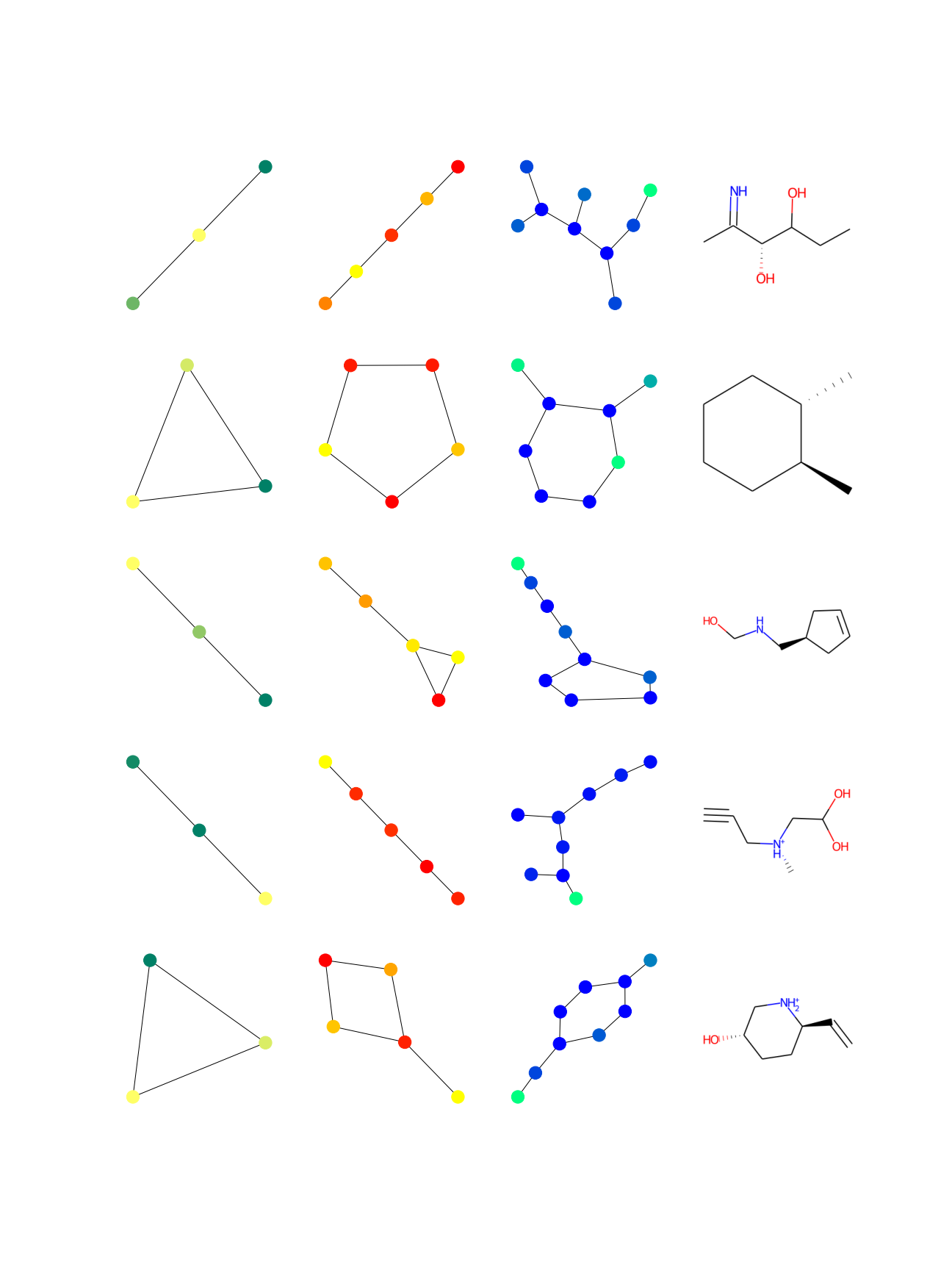}
   \caption{Five examples of generated graphs using UL GAN trained with QM9. Each row represents one example, showing intermediate graphs in the generation process. Left column: initial 3-nodes graph; Middle 2--3 columns: intermediate graphs after unpooling layers; Right column: the final generated molecule. The color represents one dimension of the node features.}\label{fig:seq_gen}
    \end{center}
    \end{figure}   

\newpage

    \begin{figure}[H]
    \begin{center}
  \hspace*{-2.0cm}
\includegraphics[scale=0.21]{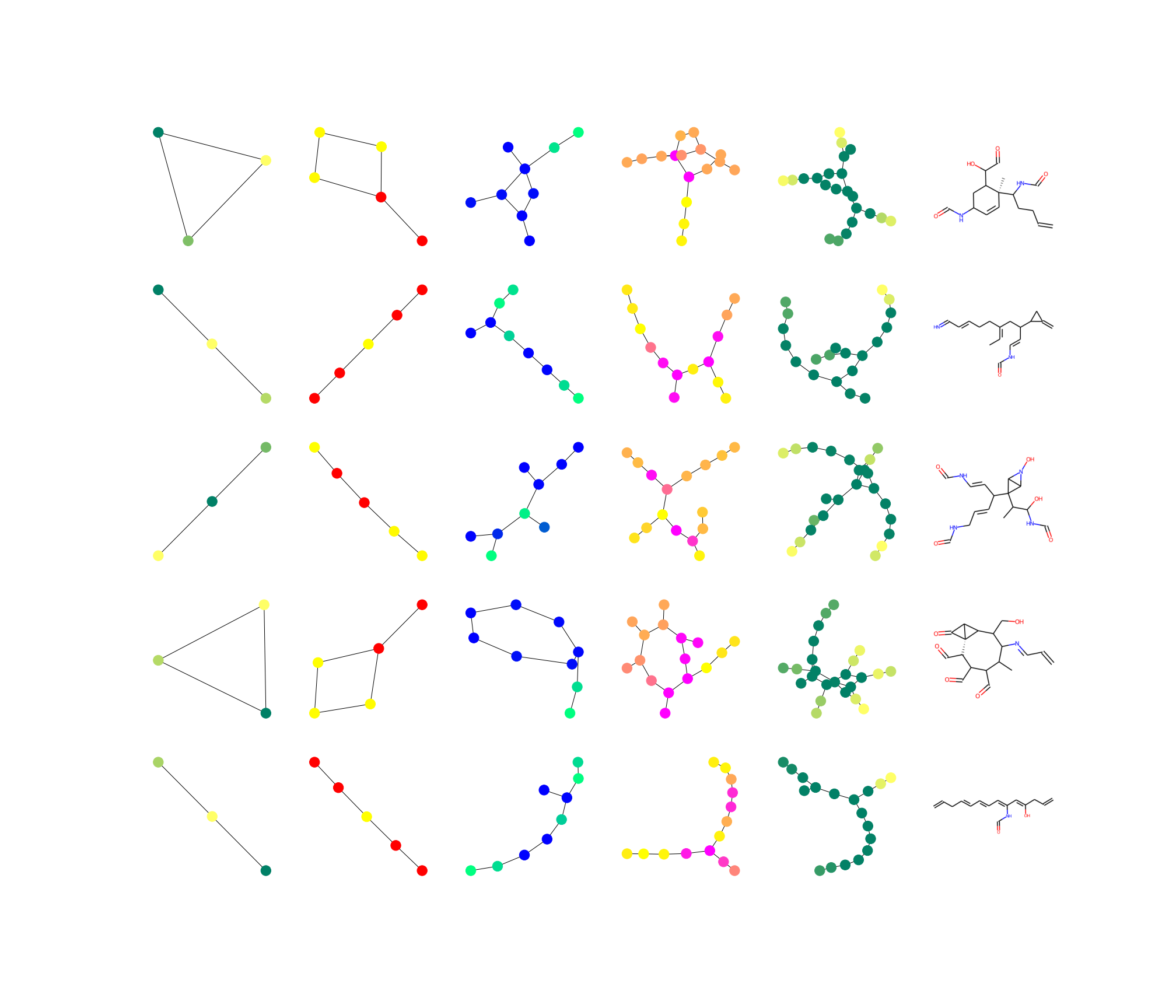}
   \caption{Five examples of generated graphs using UL GAN trained with ZINC dataset. Each row represents one example, showing intermediate graphs in the generation process. Left column: initial 3-nodes graph; Middle 2--5 columns: intermediate graphs after unpooling layers; Right column: the final generated molecule. The color represents one dimension of the node features.}\label{fig:seq_gen_zinc}
    \end{center}
    \end{figure}

\end{document}